\documentclass[twoside,11pt]{article}

\usepackage[preprint]{jmlr2e}

\usepackage[utf8]{inputenc} 
\usepackage[T1]{fontenc}    
\usepackage{hyperref}       
\usepackage{url}            
\usepackage{booktabs}       
\usepackage{amsfonts}       
\usepackage{nicefrac}       
\usepackage{microtype}      
\usepackage{lipsum}         
\usepackage{graphicx, graphbox}
\usepackage{natbib}
\usepackage{doi}
\usepackage{caption}
\usepackage{subcaption}
\usepackage{xspace}
\usepackage{soul}
\usepackage{kotex}          

\usepackage{amsmath}
\usepackage{amssymb}
\usepackage{mathtools}
\usepackage{dirtytalk}

\usepackage{cleveref}
\usepackage[textsize=tiny]{todonotes}
\usepackage[title]{appendix}
\usepackage{algorithm}
\usepackage{algorithmic}
\usepackage{float}
\usepackage{multirow}

\usepackage{pifont}

\newcommand{\B}{\mathcal{{B}}}
\newcommand{\R}{\mathbb{{R}}}
\newcommand{\sR}{\mathbb{{R}}}
\newcommand{\E}{\mathbb{{E}}}
\newcommand{\sN}{\mathbb{{N}}}
\newcommand{\I}{\mathcal{{I}}}
\newcommand{\GANO}{\textnormal{{GANO}}}

\newcommand{\Ft}{F_{\theta}}
\newcommand{\wskew}{w_{\text{skew}}}
\newcommand{\wvar}{w_{\text{var}}}
\newcommand{\wtotal}{w_{\text{total}}}
\newcommand{\skewfn}{\text{skew}}
\newcommand{\varfn}{\text{var}}

\newcommand{\vertrule}[1][1ex]{\,\rule{.6pt}{#1}\,}

\def\eqref#1{Equation~\ref{#1}}

\def\secref#1{Section~\ref{#1}}
\def\twosecrefs#1#2{Sections \ref{#1} and \ref{#2}}
\def\appref#1{Appendix~\ref{#1}}

\newenvironment{updaterequired}[1][red]{\par\color{#1}}{\par}
\newcommand{\updated}[1]{\textcolor{black}{#1}}

\usepackage{lastpage}

\ShortHeadings{Score-based Diffusion Models in Function Space}{Lim et al.}
\firstpageno{1}

\begin{document}

\title{Score-based Diffusion Models in Function Space}

\author{\name Jae Hyun Lim\thanks{Equal contribution.}~\thanks{Majority of the work was completed while the author was at NVIDIA.}
        \email limjaehy@mila.quebec \\
        \addr Universit{\'e} de Montr{\'e}al
        \AND
        \name Nikola B. Kovachki\footnotemark[1]
        \email nkovachki@nvidia.com \\
        \addr NVIDIA Corporation
        \AND
        \name Ricardo Baptista\footnotemark[1]
        \email rsb@caltech.edu \\
        \addr California Institute of Technology
        \AND
        \name Christopher Beckham
        \email christopher.beckham@mila.quebec \\
        \addr Polytechnique Montr{\'e}al
        \AND
        \name Kamyar Azizzadenesheli
        \email kamyara@nvidia.com \\
        \addr NVIDIA Corporation
        \AND
        \name Jean Kossaifi
        \email jkossaifi@nvidia.com \\
        \addr NVIDIA Corporation
        \AND
        \name Vikram Voleti
        \email vikram.voleti@umontreal.ca \\
        \addr Universit{\'e} de Montr{\'e}al
        \AND
        \name Jiaming Song
        \email jiamings@nvidia.com \\
        \addr NVIDIA Corporation
        \AND
        \name Karsten Kreis
        \email kkreis@nvidia.com \\
        \addr NVIDIA Corporation
        \AND
        \name Jan Kautz
        \email jkautz@nvidia.com \\
        \addr NVIDIA Corporation
        \AND
        \name Christopher Pal
        \email christopher.pal@mila.quebec \\
        \addr Polytechnique Montr{\'e}al \& Canada CIFAR AI Chair
        \AND
        \name Arash Vahdat
        \email avahdat@nvidia.com \\
        \addr NVIDIA Corporation
        \AND
        \name Anima Anandkumar
        \email anima@caltech.edu \\
        \addr NVIDIA Corporation \& California Institute of Technology
}

\editor{My editor}

\maketitle

\begin{abstract}
Diffusion models have recently emerged as a powerful framework for generative modeling. They consist of a forward process that perturbs input data with Gaussian white noise and a reverse process that learns a score function to generate samples by denoising. Despite their tremendous success, they are mostly formulated on finite-dimensional spaces, e.g., Euclidean, limiting their applications to many domains where the data has a functional form, such as in scientific computing and 3D geometric data analysis. This work introduces a mathematically rigorous framework called \textit{Denoising Diffusion Operators (DDOs)} for training diffusion models in function space. In DDOs, the forward process perturbs input functions gradually using a Gaussian process. The generative process is formulated by \updated{a function-valued annealed Langevin dynamic}. 
Our approach requires an appropriate notion of the score for the perturbed data distribution, which we obtain by generalizing denoising score matching to function spaces that can be infinite-dimensional. We show that the corresponding discretized algorithm generates accurate samples at a fixed cost independent of the data resolution. We theoretically and numerically verify the applicability of our approach on a set of function-valued problems, including generating solutions to the Navier-Stokes equation viewed as the push-forward distribution of forcings from a Gaussian Random Field (GRF), as well as volcano InSAR and MNIST-SDF.\footnotemark  
\end{abstract}



\section{Introduction}
\label{sec:intro}
Diffusion models (DMs)~\citep{song2020score,ho2020denoising,sohl-dickstein2015deep} have appeared as a highly successful generative approach for various domains, including images~\citep{saharia2022photorealistic}, 3D data~\citep{poole2022dreamfusion}, audio~\citep{kong2020diffwave}, video~\citep{voleti2022MCVD}, machine learning security~\citep{nie2022DiffPure}, natural language~\citep{li2022diffusion}, proteins~\citep{wu2022protein}, and molecules~\citep{xu2022geodiff}. 
These models consist of two processes: A forward diffusion process that corrupts input data by gradually adding white noise and a reverse generative process that proceeds by iterative denoising. 

Typically, DMs operate on a finite-dimensional space, e.g. \(\sR^n\), limiting their application in domains where the data is represented by infinite-dimensional objects, e.g. continuous functions. For example, in weather forecasting, data samples are functions of temperature, pressure, and wind, defined on the surface of the globe~\citep{pathak2022fourcastnet}. This also extends to seismology, geophysics, oceanography, aerodynamic vehicle design, and fluid dynamics, where we interact with functional data governed by partial differential equations (PDE)~\citep{yang2021seismic,wen2023real}. Additionally, in 3D vision and graphics applications, data is represented as functions in the form of radiance fields~\citep{mildenhall2021nerf} or signed distance functions (SDF)~\citep{park2019deepsdf}.

\begin{figure*}
\centering
\includegraphics[width=1.0\textwidth]{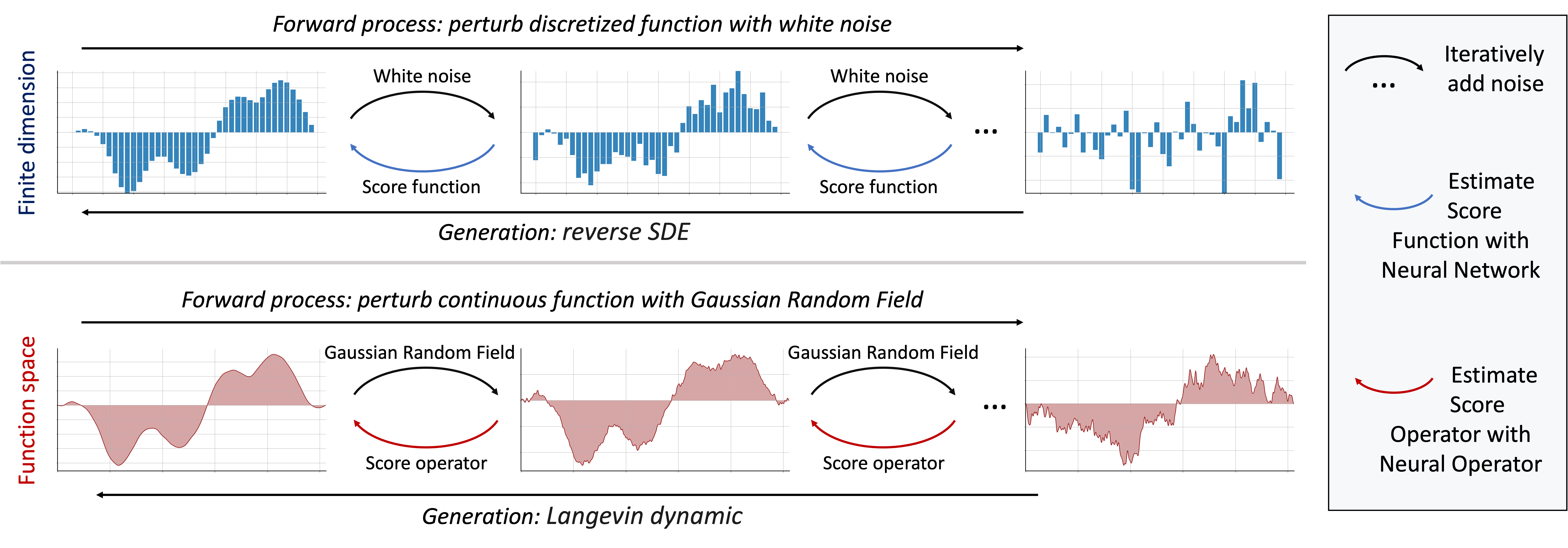}
\caption{\textbf{Overview of our approach}. While in the finite-dimensional case, inputs are discretized, we work directly in function space, on continuous inputs, here 1D functions on $\sR$. Noise is first added to the training samples during the forward process. A Neural Operator is used to estimate a score operator~(Sec.~\ref{subsec:cond_score_matching}) by minimizing the simplified loss in Eq.~\eqref{eq:conditional_score_matching_simple}. Samples are generated using Langevin dynamics~(Sec.~\ref{subsec:lagevin_dynamics}). 
Using structured noise enables efficient learning in function space while white noise does not as the model capacity required grows with the resolution.}
\label{subfig:structured_noise}
\end{figure*}

Recent attempts at applying DMs to functional data can be grouped into two categories: \textbf{(i)} the application of established algorithms on a discretization of functional data on \(\sR^n\) i.e. conditioning on point-wise values. While this approach can be made to work well at a fixed discretization, the models do not immediately transfer to 
variable discretizations of the data, and will not scale to higher resolutions~\citep{dutordoir2022neural, zhou20213d}, 
\textbf{(ii)} the mapping of input functions to a finite-dimensional latent space and modeling the latent embedding of the data with DMs \citep{dupont22afuncta,phillips2022spectral,hui2022neural,bautista2022gaudi,chou2022diffusionsdf}. Such approaches rely on efficient transformations of the data into compactly representable spaces, which limits their general applicability and are not guaranteed to be discretization-independent/convergent~\citep{kovachki2021neural}.  

The recently proposed infinite-dimensional diffusion model in~\cite{kerrigan2022diffusion} is closely related to our work. They consider a Gaussian noise corruption process in Hilbert space and derive a loss function to approximate the conditional mean of the reverse process. While the loss function is formulated using infinite-dimensional measures, the difference between the true and approximate means does not satisfy the strict range conditions that are required to have non-singular measures, and thus yields a loss that is almost surely infinite. Numerically, this effect can only be seen through progressive grid refinement which the work does not consider. For further discussion, see \appref{sec:noise_regularity}. 

Developing a diffusion-based generative framework for functions requires solving several technical challenges. First, instead of the commonly used Gaussian white noise, a new function-valued corruption process must be introduced to gradually map the data functions into random functions. Second, sample generation requires an appropriate notion of the score since infinite-dimensional distributions do not have standard probability density functions (pdf). 
 Finally, approximating the score requires both careful analysis and generalization of finite-dimensional techniques in order to obtain a well-defined optimization problem as well as approximation architectures that are consistent as mappings between function spaces.

\textbf{In our approach}, we introduce a rigorous framework termed denoising diffusion operators (DDOs) that addresses these challenges. DDOs use a Hilbert space-valued Gaussian process to perturb the input data. 
To define an appropriate notion of the score, we first consider densities with respect to a Gaussian measure (as opposed to the Lebesgue measure). 
The derivative of this density for certain perturbations of the Gaussian measure  
defines the score \emph{operator}. 
To approximate this score in practice, we generalize the denoising score matching objective of~\cite{pascal2011aconnection} to our setting, 
and show how samples can be generated using Langevin dynamics with a learned score operator. 

For learning the score, we utilize the neural operators~\citep{li2020neural, li2020fourier, kovachki2021neural}, which provide a consistent architecture in function space. We theoretically prove that approximating the score operator using neural operators is feasible. Figure~\ref{subfig:structured_noise} provides an overview of our approach. By working directly in the function space and discretizing only later for the purposes of computation, we obtain scalable and discretization-independent algorithms for generative models in function spaces.

Our primary contributions are summarized below:
\begin{enumerate} 
    \item We develop a mathematically rigorous framework for denoising score matching with function-valued data called DDO by formulating and extending all necessary theory to the abstract Hilbert space setting. 
    \item We propose a diffusion model for incrementally sampling from the data distribution by discretizing an infinite-dimensional Langevin equation with a hierarchy of noise corruption Gaussian processes, generalizing several popular finite-dimensional frameworks. 
    \item We empirically show DDO learns distributions of function-valued data on various datasets, including generating solutions to the Navier-Stokes equation viewed as the push-forward distribution of forcings from a Gaussian Random Field (GRF), as well as volcano Interferometric Synthetic Aperture Radar (InSAR)~\citep{rosen2012insar} and MNIST-SDF~\citep{sitzmann2019metasdf}.
    \item We empirically verify DDO's invariance to spatial discretization with fixed model capacity, and demonstrate accurate sample generation of a non-Gaussian distribution from the pushforward of random forcings from a GRF under the Navier-Stokes solution operator. 
\end{enumerate}

\section{Related Works}
\label{sec:relatedworks}

Our approach is broadly related to generative models formulated directly in function space instead of finite-dimensional Euclidean space~\citep{rahman2022generative}. Approaches for dealing with functional data include Gaussian processes~\citep{rasmussen2004gaussian}, and neural operators~\citep{li2020neural, li2020fourier, nelsen2021random}. These methods aim to define deep learning models in function spaces, generalizing traditional neural networks.

In the context of generative diffusion models, this complication enters the model complexity and the number of time steps that typically need to grow with the data dimension. To improve sample quality and reduce the cost of sample generation in high dimensions, yet finite, several methods propose to use diffusion models in transformed spaces. These include latent spaces~\citep{vahdat2021score}, hierarchically defined subspaces~\citep{jing2022subspace}, spectral decompositions~\citep{phillips2022spectral}, and extend to multi-scale wavelet transformations~\citep{guth2022wavelet}. Compared to score-based models operating in the original domain, the latter approach shows that the time complexity (i.e., the number of time steps required to achieve a fixed error) grows linearly with the image dimension. However, these models are not formulated in an infinite-dimensional space.

Neural Processes (NP) \citep{garnelo2018neural, kim2019attentive, bruinsma2021gaussian} aim to model distributions consistent with arbitrary discretizations, 
and \cite{dutordoir2022neural} have examined their extension to Neural Diffusion Processes (NDP). While the NP framework can process arbitrary sets of inputs, they inherit the limitations of using finite-dimensional latent variables; thus, consistency breaks in practice as the resolution grows. Moreover, the induced model distributions in NDP do not exist in function space due to independent noise in the noise process.

An earlier attempt to learn measures on function spaces deploys sequences of delta functions to fully memorize the data points~\citep{craswell1965density}. Such a method is based on pure memorization and ignores possible underlying structures of the data measure. Kernel density estimation was proposed as a heuristic approach in infinite dimensional spaces~\citep{dabo2004kernel}, though requires smoothness, extra regularity, and continuity with respect to an unspecified measure~\citep{dabo2004kernel}. Alternative methods treat a discretized function as a point cloud and aim to maximize the likelihood of the point values~\citep{garnelo2018neural}, similarly to NPs.

Leveraging neural operators,~\cite{rahman2022generative} propose the generative adversarial neural operator (\GANO) for learning function data distribution. 
%
As such, it enables learning of the distribution in function spaces through learning a mapping from infinite dimension spaces of Gaussian Random Fields (GRFs) to distribution in function spaces data. However, \GANO~ inherently suffers from the major drawbacks of adversarial training, such as limited stability, optimization, and flexibility, as pointed out in prior works~\citep{arjovsky2017towards, lin2018pacgan, song2019generative, berard2019closer}.

The use of GRFs in denoising diffusion models has been discussed but is yet to be explored in the domain of function spaces~\citep{voleti2022score}. The generative adversarial neural network framework~\citep{goodfellow2014generative} was recently used in conjunction with implicit neural network representations of data \citep{dupont2021generative, anokhin2021image, skorokhodov2021adversarial, chen2021learning}. 
These methods are not discretization invariant and fail as the discretization of the data changes~\citep{rahman2022generative}.
\cite{dupont22afuncta} embeds discretized data in function space using implicit neural network representations, but it still inherits the drawbacks of using finite-dimensional latent spaces to encode infinite-dimensional data.

\updated{
Recently, several continuous-time diffusion models in function space have also been proposed by \citet{pidstrigach2023infinite, baldassari2023conditional, hagemann2023multilevel}. These works define a forward and backward process by a pair of stochastic differential equations (SDEs) where the score operator is given as a conditional expectation, depending on the forward process. Our work offers an alternative viewpoint with the score defined as a logarithmic derivative of a perturbed measure and sampling done by a Langevin process and its annealed version. This allows us to make clearly interpretable assumptions on the data measure that are needed to guarantee convergence in the infinite-dimensional setting and furthermore allows us to study the interplay between the regularity of noise and the data.
}

\section{Background: Denoising Score Matching in Finite Dimensions}
\label{subsec:background}
Historically, \textit{score matching} refers to the notion of approximating the score (i.e., the logarithmic derivative) of some unknown or 
computationally intractable distribution for the purposes of sampling, testing, or density estimation.
Let \(p\colon \sR^d \to \sR\) denote the pdf of a $d$-dimensional distribution and let \(s_\theta : \sR^d \to \sR^d\) be a parametric mapping with 
parameters \(\theta \in \sR^m\). Ideally, score matching aims to solve
\begin{equation}
    \label{eq:finitedim_scorematching}
    \min_{\theta \in \sR^p} \E_{x \sim p(x)} \|s_\theta (x) - \nabla \log p (x) \|^2_2.
\end{equation}

In many applications, we are only given samples from \(p\), but do not know its analytic form. Therefore, solving~\eqref{eq:finitedim_scorematching}
is intractable. Using integration by parts on the objective, \cite{hyvarinen05a} showed that the minimizer of~\eqref{eq:finitedim_scorematching} can be found by
optimizing

\begin{equation}
    \label{eq:finitedim_derivativescorematching}
    \min_{\theta \in \sR^p} \E_{x \sim p(x)} \left[\text{Tr} \big ( \nabla s_\theta (x) \big ) + \|s_\theta (x) \|^2_2 \right].
\end{equation}
Remarkably, the objective in~\eqref{eq:finitedim_derivativescorematching} 
can be minimized using a Monte-Carlo approximation to the expectation. 
It was later noted in \cite{pascal2011aconnection} that, up to a perturbation 
of the data distribution, the optimization problem is equivalent to \textit{denoising score matching} where the objective depends on the analytically tractable 
score of the conditional perturbed distribution and no derivatives of the approximating function. In particular, for a Gaussian perturbation of variance $\sigma^2$,
\eqref{eq:finitedim_scorematching} is equivalent to optimizing
\begin{equation}
    \label{eq:finitedim_denoisingscorematching}
    \min_{\theta \in \sR^p} \E_{\eta \sim N(0, \sigma^2 I)} \E_{x \sim p(x)} \left\|\frac{\eta}{\sigma^2} + s_\theta(x + \eta)\right\|_2^2,
\end{equation}
where \(s_\theta\) is now an approximation to the score of the perturbed distribution.
Since \eqref{eq:finitedim_denoisingscorematching} does not require knowledge of \(p\) or computation of any derivatives, denoising score matching 
is attractive for problems in high dimensions where computing derivatives is costly. Furthermore, it is argued in~\cite{song2019generative}, that for many practical applications, 
for example, photorealistic image generation, \(p\) is supported on a lower dimensional manifold and thus approximating the score on the
ambient space can be unstable. Thus perturbing the data distribution gives \emph{both} a more computationally tractable optimization problem and acts as a regularizer by
spreading the support of \(p\) to the entire space. 

We build on this framework by generalizing the notion of score and 
denoising score matching to infinite dimensions. By working directly in the infinite-dimensional setting, we derive a methodology that is consistent and generalizable across  different discretizations of the data.

\section{Denoising Diffusion Operators (DDO)}
\label{sec:methodology}
\label{subsec:cond_score_matching}

We introduce DDO to perform denoising score matching in function space. We work on an infinite-dimensional, real, separable Hilbert space \((H, \langle \cdot, \cdot \rangle, \| \cdot \| )\)
with the Borel \(\sigma\)-algebra of measurable sets denoted \(\B(H)\)\footnote{While a more general formulation on Banach or even 
locally convex spaces is possible, explicit computations for Gaussian measures on Hilbert spaces are more readily available 
and thus we consider this setting.}. Since there is no Lebesgue measure in infinite dimensions, there is no standard notion of 
a probability density; we therefore adopt the more general, measure-theoretic notation to introduce our setting. We denote
by \(\mu\) a probability measure on \(\B(H)\) which we will call our \textit{data measure}. In particular, we assume to have
a dataset of samples \(\{u_j\}_{j=1}^N\) where \(u_j \sim \mu\) are i.i.d. random variables. These samples are considered to be
infinite-dimensional objects, i.e. functions or infinite sequences, before any finite-dimensional discretization is done for the purposes of 
computation. 

For a corruption process, we consider additive Gaussian perturbations to the data in the form of function-valued GRF perturbations. This choice is motivated by the availability of analytical results related to 
Gaussian measures, the ease and efficiency of sampling Gaussians in infinite dimensions by means of the Karhunen-Lo\'{e}ve expansion~\citep{lord2014anintoroduction} (see \appref{sec:kl_expansion}), and the 
plethora of empirically successful results for denoising score matching with Gaussians in finite dimensions~\citep{song2020denoising, ho2020denoising}. We employ the centered 
Gaussian measure on \(H\) denoted by \(\mu_0 = N(0,C)\) with a covariance operator \(C:H \to H\) to be self-adjoint, non-negative, and 
trace-class (nuclear). Indeed, these conditions on \(C\) are necessary and sufficient for \(\mu_0\) to be Gaussian on \(H\)~\citep{da1992stochastic}. We note that since
trace-class implies compact, the identity covariance operator is ruled-out  as \(H\) is infinite-dimensional. In particular, white noise does not live in
\(H\) but must rather be defined on a larger space \citep{da1992stochastic}. We show empirically that by working with noise defined on \(H\) our method remains
discretizationally invariant with respect to the data. On the other hand, working with white noise breaks this property precisely because white noise samples 
are not regular compared to the elements of \(H\).

\subsection{\updated{Denoising score matching on function spaces}}
\label{subsec:denoisings_score_matching}

We consider the perturbation to the data samples
\begin{equation}
    \label{eq:rough_perturbation}
    v = u + \eta, \qquad u \sim \mu, \; \eta \sim \mu_0,
\end{equation}
with \(u \perp \eta\) and denote by \(\nu\) the probability measure induced by
the random variable \(v\) i.e. the convolution \(\nu = \mu * \mu_0\); see \appref{sub:conv_measure} for more details. We show in Lemma~\ref{lemma:wasserstein_approx}, that when the noise \(\eta\) is small in an approximate sense,
\(\mu\) and \(\nu\) are close as measures in the Wasserstein metric. It is therefore reasonable to approximate \(\nu\) instead of \(\mu\) as is done in denoising score matching
since in the limit of vanishing noise, the two become identical.

We define the score of \(\nu\) via an appropriate notion of density, which is 
defined with respect to a reference measure. 
In infinite dimensions, much work has been 
focused on studying densities defined with respect to Gaussian measures (as opposed to the Lebesgue measure in finite dimensions) 
as doing so has natural applications in statistics, inverse problems, and quantum field theory \citep{ghosal2017fundamentals, stuart2010inverse, kupiainen2016quantum}. We also 
take this approach as it leads to a well-defined notion of the score that is analytically tractable and comes with an associated 
Langevin equation which can be solved to produce samples from \(\nu\). We choose the reference  to be perturbing measure \(\mu_0\), 
which is natural in this setting since the conditional \(v|u\) is Gaussian with the same covariance as \(\mu_0\).
A density 
is then be obtained by the Radon–Nikodym Theorem under the assumption that $\nu$ is absolutely continuous with respect to $\mu_0$, i.e., \(\nu \ll \mu_0\)~\citep{halmos1976measure}. 

To satisfy the absolute continuity condition with respect to Gaussian $\mu_0$, 
it is reasonable to expect that the data measure \(\mu\) must satisfy certain assumptions. 
The assumption we make is that \(\mu(H_{\mu_0}) = 1\), i.e., $\mu$ is fully supported on the Cameron-Martin space of \(\mu_0\) that is denoted by \(H_{\mu_0} \coloneqq C^{1/2}(H)\). Cameron-Martin spaces 
play a crucial role in the theory of Gaussian measures as they are an invariant of the measure that gives it 
meaning outside the ambient space \(H\)~\citep{bogachev2015gaussian}. We remark that this assumption can make precise the \say{manifold hypothesis} in~\cite{song2019generative} that is
used to justify the perturbation since \(H_{\mu_0}\) is a proper subspace of \(H\) and, in fact, \(\mu_0 (H_{\mu_0}) = 0\); see Section 6 of~\cite{stuart2010inverse} for more details. In particular,
data samples lie on a measure-zero set of the perturbing measure. The addition of noise thereby spreads out the samples to the whole space. We note that when this
assumption is not satisfied, we can still apply our 
framework using a different form of the perturbation in~\eqref{eq:rough_perturbation}; see \secref{subsec:smoothing_operators}. \updated{Some example measure that satisfies \(\mu(H_{\mu_0}) = 1\) are listed in \appref{sec:further_examples}.}

Under the condition above on the data perturbations, we can now state the following theorem.
\begin{theorem}[Measure Equivalence]
\label{thm:equivalence_informal}
The perturbed measure \(\nu\) and the centered Gaussian \(\mu_0\) are equivalent
in the sense of measures, which we denote by \(\nu \sim \mu_0\).
\end{theorem}
A more general statement and proof of this result are given in \appref{sub:conv_measure}. The importance of Theorem~\ref{thm:equivalence_informal} is that 
it allows us to obtain a density. Indeed, it verifies the assumption of 
the Radon–Nikodym Theorem, which we apply to obtain a strictly positive density of \(\nu\) with respect to \(\mu_0\).
In particular, there exists a Borel measurable mapping \(\Phi\colon H \to \sR\) such that
\begin{equation}
    \label{eq:v_density}
    \frac{d \nu}{d \mu_0}(w) = \text{exp} \big (  \Phi(w) \big ), \qquad \mu_0\text{-a.s. } w \in H.
\end{equation}

We will assume that \(\Phi\) is Fr\'{e}chet differentiable along the Cameron-Martin space \(H_{\mu_0}\) which is itself a Hilbert space continuously embedded in \(H\). This is a reasonable assumption since the vectors of differentiability of any Gaussian are precisely those in its Cameron-Martin space  and \(\nu\) is equivalent, in the sense of measures, to the Gaussian \(\mu_0\) \citep{bogachev2015gaussian}. While in finite dimensions differentiability is always ensured since Gaussians have infinitely smooth density and convolutions preserve this regularity, in infinite dimensions, this need not always be the case. We therefore make it an assumption, however, the following example shows that it is true of any Gaussian data measure.
\begin{example}
    Suppose $\mu = N(0,Q)$ for some self-adjoint, non-negative, and trace-class operator $Q: H \to H$. It follows by non-negativity that
    \[\langle h, C h \rangle \leq \langle h, (C+Q) h \rangle \qquad \forall \: h \in H.\]
    Therefore by Lemma 6.15 in \citep{stuart2010inverse}, $C^{1/2}(H) \subseteq (C + Q)^{1/2}(H)$. From definition $\nu = N(0,C + Q)$ and, by Proposition 5.1.6. in \citep{bogachev2015gaussian}, $\nu$ is differentiable along its Cameron-Martin space $(C+Q)^{1/2}(H)$. Therefore $\nu$ is differentiable along $C^{1/2}(H) = H_{\mu_0}$. 
\end{example}
We define the score precisely as the Fr\'{e}chet derivative of \(\Phi\) in the direction of \(H_{\mu_0}\) and denote it \(D_{H_{\mu_0}} \Phi\colon H \to H_{\mu_0}^*\) where \(H_{\mu_0}^*\) is the topological (continuous) dual of \(H_{\mu_0}\). In other words, the score of \(\nu\) with respect to \(\mu_0\) is the Fr\'{e}chet derivative of the logarithm of the density of \(\nu\) with respect to \(\mu_0\),
\begin{equation}
    \label{eq:score_def}
    D_{H_{\mu_0}} \Phi = D_{H_{\mu_0}} \log \frac{d \nu}{d \mu_0}.
\end{equation}
We refer the reader to Chapter 5 in \cite{bogachev2015gaussian} for a general discussion of differentiability in infinite dimensions.

Having appropriately defined the score of \(\nu\), we can introduce a score matching objective. Let \(G_\theta\colon H \to H_{\mu_0}^*\) be a parametric mapping with parameters \(\theta \in \mathbb{R}^p\). We consider the learning problem
\begin{equation}
    \label{eq:score_matching}
    \min_{\theta \in \mathbb{R}^p} \E_{v \sim \nu} \|D_{H_{\mu_0}}\Phi(v) - G_\theta (v)\|_{H^*_{\mu_0}}^2.
\end{equation}
Since \(D_{H_{\mu_0}}\Phi\) is unknown to us, solving~\eqref{eq:score_matching} is computationally intractable. 

To obtain a tractable problem, we generalize the conditioning theorem in~\cite{pascal2011aconnection}. Let us first notice that the measure induced by the conditional \(v|u\) is the Gaussian \(N(u, C) \coloneqq \gamma^u\) for \(\mu\)-almost any \(u \in H\). Since \(\mu(H_{\mu_0}) = 1\), the Feldman–H\'{a}jek Theorem implies that \(\gamma^u \sim \mu_0\) \citep{da1992stochastic}. In particular, we may compute explicitly that, for \(\mu_0\text{-almost any } w \in H\) and \(\mu\)-almost any \(u \in H_{\mu_0}\),
\begin{align}
    \label{eq:v_given_u_density}
    \frac{d \gamma^u}{d \mu_0}(w) &= \text{exp} \left ( \sum_{j=1}^\infty \lambda_j^{-1} \langle w, \varphi_j \rangle \langle u, \varphi_j \rangle - \frac{1}{2} \|C^{-1/2}u\|^2 \right ) \nonumber \\
    &\coloneqq \text{exp} \big( \Psi(w; u) \big ),
\end{align}
where \(C \varphi_j = \lambda_j \varphi_j\) for \(j \in \sN\) is an eigendecomposition of \(C\) and \(C^{-1/2}\) denotes the inverse of \(C^{1/2}\) on \(H_{\mu_0}\), see Theorem 2.23 in \cite{da1992stochastic}.
The score of each conditional \(\gamma^u\) is given as the Fr\'{e}chet derivative (in the first argument) of the potential \(\Psi : H \times H_{\mu_0} \to \sR\) in the direction of \(H_{\mu_0}\). 
We can now state the following (informal) theorem relating~\eqref{eq:score_matching} to the solution of a tractable problem.
\begin{theorem}[Denoising Score Matching]
    \label{thm:conditional_score_informal}
    Under some integrability assumptions on \(D_{H_{\mu_0}} \Phi\) and \(G_\theta\), the minimizers of \eqref{eq:score_matching} 
    are the same as the minimizers of
    \begin{equation}
        \label{eq:conditional_score_matching}
        \min_{\theta} \E_{u \sim \mu} \E_{w \sim \gamma^u} \|D_{H_{\mu_0}} \Psi (w;u) - G_\theta(w) \|_{H_{\mu_0}^*}^2.
    \end{equation}
\end{theorem}
The more general statement (for a broader class of perturbations than~\eqref{eq:rough_perturbation}) and proof are given in \appref{sub:conditioning}. 
 \eqref{eq:conditional_score_matching} gives us an infinite-dimensional analog of \eqref{eq:finitedim_denoisingscorematching}, where we can compute \(D_{H_{\mu_0}} \Psi (w;u)\) from \eqref{eq:v_given_u_density}. That is,
\begin{equation}
    \label{eq:gaussian_score}
    D_{H_{\mu_0}} \Psi(w;u) = \sum_{j=1}^\infty \lambda_j^{-1} \langle u, \varphi_j \rangle \varphi_j,
\end{equation}
where we interpret \eqref{eq:gaussian_score} as
\begin{equation}
    \label{eq:gaussian_score_action}
    D_{H_{\mu_0}} \Psi(w;u) z = \sum_{j=1}^\infty \lambda_j^{-1} \langle z, \varphi_j \rangle \langle u, \varphi_j \rangle,
\end{equation}
for any  \(z \in H_{\mu_0}\). Indeed, Lemma~\ref{lemma:gaussian_derivative} shows that \(D_{H_{\mu_0}} \Psi(w;u) \in H_{\mu_0}^*\) as defined by \eqref{eq:gaussian_score}.

Recall that our objective is to approximate \(D_{H_{\mu_0}} \Phi\) by solving~\eqref{eq:score_matching}, which we have shown is equivalent to \eqref{eq:conditional_score_matching}. Given such an approximation, we can then solve 
a Langevin equation with the learned score in order to obtain samples from \(\nu\). As we will show in the next section, this Langevin equation requires only knowledge of the 
\(D_{H_{\mu_0}} \Phi\) projected onto \(H_{\mu_0}\). We can thus simplify the optimization problem in~\eqref{eq:conditional_score_matching} by considering the Reisz map \(R : H_{\mu_0}^* \to H_{\mu_0}\),
which is the canonical isometric isomorphism between the Hilbert spaces \(H_{\mu_0}^*\) and \(H_{\mu_0}\). Using the isometric property, we find
\[
    \|D_{H_{\mu_0}} \Psi (v;u) - G_\theta(v) \|_{H_{\mu_0}^*}^2 = \|C^{-1/2} \big ( u - RG_\theta(v) \big )\|^2
\]
by noting that \(R\) acts as \(C\) to elements of \(H_{\mu_0}^*\) that are not in \(H_{\mu_0}\) and using \eqref{eq:gaussian_score}.
In particular, we have shown that minimizing \eqref{eq:conditional_score_matching} is equivalent to minimizing 
\begin{equation}
    \label{eq:conditional_score_matching_simple}
    \min_{\theta \in \mathbb{R}^p} \E_{u \sim \mu} \E_{\eta \sim \mu_0} \|C^{-1/2} \big ( u - RG_\theta(u + \eta) \big )\|^2,
\end{equation}
which is a de-noising problem pre-conditioned by \(C^{-1/2}\).
Note that \eqref{eq:conditional_score_matching_simple} is almost surely finite since \(u - RG_\theta (w) \in H_{\mu_0}\)
for any \(w \in H\) by our assumption that \(\mu(H_{\mu_0}) = 1\).

\begin{updaterequired}[black]
To that end, supposing that \(C\) is positive, then \(\text{ker}(C^{-1/2}) = \{0\}\). Therefore, 
optimizing \eqref{eq:conditional_score_matching_simple} is equivalent to optimizing
\begin{equation}
    \label{eq:score_matching_no_C}
    \min_{\theta \in \mathbb{R}^p} \E_{u \sim \mu} \E_{\eta \sim \mu_0} \|u - RG_\theta(u + \eta)\|^2.
\end{equation}
\end{updaterequired}

\subsection{Smoothing Operators}
\label{subsec:smoothing_operators}
When the assumption \(\mu(H_{\mu_0}) = 1\) is not satisfy, we may consider a different form of the perturbation in \eqref{eq:rough_perturbation} to remove this regularity assumption. To that end, let \(A\colon H \to H\) be a linear operator with the property that \(A(H) \subseteq H_{\mu_0}\). Consider the data perturbation
\begin{equation}
    \label{eq:smooth_perturbation}
    v = Au + \eta, \qquad u \sim \mu, \; \eta \sim \mu_0.
\end{equation}
We re-define the measures \(\nu\) and \(\gamma^u\) appropriately according to \eqref{eq:smooth_perturbation}. Corollary~\ref{cor:equivalent_convolution} and the Feldman–H\'{a}jek Theorem imply that \(\nu \sim \mu_0\) and \(\gamma^u \sim \mu_0\) for \(\mu\)-almost every \(u \in H\). Therefore the results of the previous section hold with the mapping \(u \mapsto Au\) implemented in all formulae. \updated{Crucially, the learning problem in~\eqref{eq:score_matching_no_C} becomes}
\begin{equation}
    \label{eq:conditional_smooth_matching_simple_A}
    \min_{\theta \in \mathbb{R}^p} \E_{u \sim \mu} \E_{\eta \sim \mu_0} \|Au - RG_\theta(Au + \eta)\|^2.
\end{equation}
Here \(A\) acts as a smoothing operator, bringing the data into a regular enough space for the required absolute continuity to hold.
This makes mathematically precise diffusion models which use heat-dissipation or blurring as a forward operator \citep{rissanen2022generative, hoogeboom2022blurring}.
\updated{We expand on this idea in \secref{subsec:multiple_noise_scales} (See also \appref{sec:num_smoothing}).}

\subsection{Approximation Theory}
\label{subsec:approximation_theory}

We have shown that the pre-conditioned score operator necessary for sampling is a non-linear mapping of the Hilbert space
\(H\) into itself. We therefore need architectures which can approximate such mapping. We employ 
the neural operator framework of~\cite{kovachki2021neural}. The following approximation result then follows by Theorems 11 and 13 in \cite{kovachki2021neural} and the proof
methods therein. 
\begin{theorem}[Score Approximation]
    \label{thm:approximation}
    Let \(D \subset \sR^d\) be a bounded open set with Lipschitz boundary and consider \(H = L^2 (D;\sR)\). 
    Suppose \(\I \subset \sR^n\) is compact and let \(R D_{H_{\mu_t}} \Phi (\cdot, t)\colon H \to H\) be the pre-conditioned score
    of the perturbation in~\eqref{eq:time_pertrubation} for each \(t \in \I\).
    Suppose \(\nu_t\) has a finite second-moment for each \(t \in \I\) and the map \(t \mapsto R D_{H_{\mu_t}} \Phi (\cdot, t)\) is uniformly continuous. Then,
    for any \(\epsilon > 0\), there exists a number \(p = p(\epsilon) \in \sN\) and a parameter vector \(\theta = \theta(\epsilon) \in \sR^p\)
    such that a neural operator \(G_\theta \colon H \times \I \to H\) satisfies
    \[\sup_{t \in \I} \E_{u \sim \nu_t} \|R D_{H_{\mu_t}} \Phi(u,t) - G_\theta (u,t) \|^2 < \epsilon.\]
\end{theorem}

\begin{remark} In Theorem~\ref{thm:approximation}, we crucially work in a setting where the map \(t \mapsto R D_{H_{\mu_t}} \Phi (\cdot, t)\) is uniformly continuous and the score is well-defined for every \(t \in \I\) i.e., the perturbing noise has a non-zero covariance uniformly across \(\I\).
This is important in avoiding the well-known singularity in the conditional score in the limit of vanishing noise. See \cite{kim2021soft} for numerical methods for accurately approximating the score at small times for score-based models in finite dimensions.
\end{remark}

Theorem~\ref{thm:approximation} suggests that approximating score operators in infinite dimensions is feasible using neural operators.
We demonstrate this numerically in the next section.

\subsection{Langevin Dynamics}
\label{subsec:lagevin_dynamics}
To sample from $\nu$, we consider the infinite-dimensional, pre-conditioned, Langevin equation,
\begin{equation}
    \label{eq:langevin}
    \frac{du}{dt} = -u + R D_{H_{\mu_0}} \Phi(u) + \sqrt{2} \frac{dW}{dt}, \quad u(0) = u_0
\end{equation}
for some \(u_0 \in H\) where \(R D_{H_{\mu_0}} \Phi : H \to H_{\mu_0}\) and W is a \(C\)-Wiener process \citep{da1992stochastic}.
It is shown in~\cite{dashti2017bayesian} that, under appropriate boundedness assumptions on \(R D_{H_{\mu_0}} \Phi\),
equation~\eqref{eq:langevin} has a unique strong solution with continuous paths and an invariant measure \(\nu\). In particular, samples from \(\nu\) can be 
obtained as the long-time solutions of \eqref{eq:langevin}. We will approximate~\eqref{eq:langevin} by using
the learned score \(RG_\theta\) and discretizing in time using the Euler–Maruyama scheme with step-size \(h > 0\). This gives us the update
\begin{equation}
    \label{eq:langevin_discrete}
    u_{n+1} = u_n  + h(RG_\theta (u_n) - u_n) + \sqrt{2h} \xi_n,
\end{equation}
for any \(n \in \mathbb{N}\) where \(\xi_n \sim N(0,C)\) are i.i.d.\thinspace random variables. Equation~\eqref{eq:langevin_discrete} also suggests that
instead of looking for the map \(RG_\theta\), we can re-parameterize and instead directly find the mapping \(v \mapsto R G_\theta (v) - v\).
\updated{
Defining \(F_\theta\colon H \to H\) by \(F_\theta(v) = R G_\theta (v) - v\), optimizing \eqref{eq:conditional_smooth_matching_simple_A} is equivalent to 
\begin{equation}
    \label{eq:learn_noise}
    \min_{\theta \in \mathbb{R}^p} \E_{u \sim \mu} \E_{\eta \sim \mu_0} \|\eta + F_\theta(Au + \eta)\|^2,
\end{equation}}
which simplifies the sampling update in~\eqref{eq:langevin_discrete} to 
\begin{equation}
    \label{eq:sampling}
    u_{n+1} = u_n  + hF_\theta(u_n) + \sqrt{2h} \xi_n.
\end{equation}
Note that this re-parameterization is only valid when \(C\) is positive, otherwise \eqref{eq:conditional_score_matching_simple} and \eqref{eq:score_matching_no_C}
 are not equivalent and \(\text{Im}(F_\theta) = H\) while \(\text{dom} (C^{-1/2}) = H_{\mu_0}\). In particular, for general \(C\), we may optimize
\eqref{eq:conditional_score_matching_simple} and sample with \eqref{eq:langevin_discrete}, while for \(C\) positive, we can alternatively optimize \eqref{eq:learn_noise}
and sample with \eqref{eq:sampling}. The advantage of \eqref{eq:learn_noise} is that we can parameterize \(F_\theta\) as an arbitrary \(H \to H\) mapping 
without any restrictions on its range space. Furthermore current empirical evidence suggests that learning the noise from the signal instead of the signal from the noise 
yields better sample quality \citep{song2020improved, ho2020denoising}. Since \(C\) is a choice in our method that can be tuned, we always pick it positive and thus utilize this re-parametrization in our experiments.

We remark that we have only picked the Euler–Maruyama discretization of \eqref{eq:langevin} here for the sake of clarity in exposition. 
Many other choices such as~\cite{cotter2013mcmc} are possible; see \appref{sec:disc_langevin} for details.


\subsection{\updated{Multiple Noise Scales and Annealed Langevin Dynamics}}
\label{subsec:multiple_noise_scales}

As argued in \cite{song2019generative}, the mixing times of Langevin dynamics such as \eqref{eq:langevin} may be slow.
It therefore of practical interest to consider multiple noise processes over different scales and thereby an annealing process for discretizing \eqref{eq:langevin}. To that end, let \(\I\) be
some (possibly uncountable) index set and consider the data perturbations
\begin{equation}
    \label{eq:time_pertrubation}
    v_t = A_t u + \eta_t, \qquad u \sim \mu, \; \eta_t \sim \mu_t
\end{equation}
for a family of linear operators \(\{A_t\colon H \to H\}_{t \in \I}\) and Gaussian measures \(\{\mu_t\}_{t \in \I}\). Let $\nu_t$ be the measure for $v_t$.

\updated{
Let us first consider the case \(\mu(H_{\mu_0}) = 1\). Moreover, we assume \(\mu_t = N(0,C_t)\) where \(C_t = g(t) C\) with \(C\) as before, \(A_t = f(t) I\) where \(I\) is the identity operator, and \(f,g : \I \to \sR\) are mappings bounded from above and below away from zero. 
Lemma~\ref{lemma:multiple_scales} shows that 
\(A_t(H) \subseteq H_{\mu_t}\) for all $t \in \I$ 
and therefore our previous theory holds. The choices \(\I = [T]\) for some \(T \in \sN\)
and \(f(t) = 1\), \(g(t) = \sigma_t^2\) for some sequence \(0 < \sigma_T \leq \dots \leq \sigma_1\)
recovers the NCSN framework of \cite{song2019generative}. Similarly, let \(0 < \beta_1 \leq \dots \leq \beta_T < 1\) be some sequence
and define \(\alpha_t = \prod_{s=1}^t (1-\beta_s)\). Then, setting
\(f(t) = \sqrt{\alpha_t}\) and \(g(t) = 1 - \alpha_t\) recovers the DDPM framework of \cite{ho2020denoising}; see \appref{sec:ddpm} for more details on this connection. In particular, we
generalize two widely used diffusion models in infinite-dimensions, up to the method selected for generating samples.
}

\updated{Let us now consider a case where we do not make assumptions on the data measure \(\mu\).} 
For our previous theory to hold, we need that \(A_t(H) \subseteq C^{1/2}_t (H)\) for all \(t \in \I\).
This can be accomplished with various choices of \(A_t\), see \appref{sec:timedeptnoise}. For the current discussion, 
we will take \(H = \dot{L}^2(\mathbb{T}^d;\sR)\) with \(\mu_t\) as before where \(C\)
has the form \eqref{eq:example_C1}. Let \(\I = [T_0, T]\) for some \(0 < T_0 < T < \infty\)
and choose \(A_t = f(t) \text{e}^{t \Delta}\) with the same boundedness assumptions on 
\(f,g\). In particular, the family \(\{A\}_{t \in \I}\) is a re-scaled subset of the semi-group 
associated to the solution operator of the heat equation (assuming \(f\) is continuous so that \(t \mapsto A_t\) is continuous). 
Classical results on the heat equation show that for any \(u \in H\), \(A_t u \in \dot{H}^s (\mathbb{T}^d;\sR)\)
for any \(s > 0\) \cite{evans2010partial}. In particular, by choosing \(\alpha_1 > d/2\) in \eqref{eq:example_C1} so
that \(C\) is trace-class, we find that \(A_t (H) \subset C^{1/2}_t (H) = \dot{H}^{\alpha_1}(\mathbb{T}^d;\sR)\).
We have thus exhibited an infinite-dimensional generalization to the \say{inverse heat-dissipation} framework of~\cite{rissanen2022generative}. 

\begin{algorithm}[!htb]
   \caption{Annealed Langevin Dynamics}
   \label{alg:sampling_alg}
\begin{algorithmic}
   \STATE {\bfseries Input:} \(F_\theta\), \(u_0 \in H\), \(\{\sigma_t\}_{t=1}^T\), \(M \in \sN\), \(\epsilon > 0\)
   \FOR{$t=1$ {\bfseries to} $T$}
   \STATE $h_t = \epsilon \sigma_t^2/\sigma_T^2$.
   \FOR{$n=0$ {\bfseries to} $M-1$}
   \STATE $\eta^{(t)}_n \sim N(0,C)$
   \STATE $u_{n+1} = u_n + h_t F_\theta (u_n, t) + \sqrt{2 h_t} \eta^{(t)}_n$
   \ENDFOR
   \STATE $u_0 = u_M$
   \ENDFOR
\end{algorithmic}
\end{algorithm}

\begin{updaterequired}[black]
To sample $\{ \nu_t \}_{t \in \I}$ and thus eventually sample \(\nu_T\), 
we can apply to the Euler-Maruyama scheme for each \(t \in \I\) to obtain the iteration
\begin{equation}
    \label{eq:em_scheme}
    u_{n+1} = u_n + h_t F(u_n, t) + \sqrt{2 h_t} \eta^{(t)}_n
\end{equation}
for any \(n \in \sN\), where \(\eta^{(t)}_n \sim \mu_t\) form an i.i.d. sequence, with \(h_t > 0\) and \(t_0 \in \I\) are fixed. Here, \(F : H \times \I \to H\) is a model defined by \(F(u,t) = -u + R D_{H_{\mu_t}} \Phi(u,t) \) as discussed in \eqref{eq:sampling}. 
For any \(v \in H\), the iteration in~\eqref{eq:em_scheme} starting with \(u_0 = v\) at \(t = t_0\) transforms 
\(v\) to an approximate sample of \(\nu_{t_0}\). We denote this sample by \(v_{t_0}\). Now fix \(t_1 \in \I\). We again run the 
iteration~\eqref{eq:em_scheme} with \(t = t_1\) and \(u_0  = v_{t_0}\). This will transform \(v_{t_0}\) into an approximate sample 
from \(\nu_{t_1}\) which we denote \(v_{t_1}\). If \(\I = \{1,\dots,T\}\) for some \(T \in \sN\) then repeating this process
yields \(v_T\), which is approximately distributed according to \(\nu_T\). 
Moreover, according to Lemma~\ref{lemma:wasserstein_approx}, \(v_T\) is approximately distributed according to our original data measure \(\mu\). 
We outline this annealing process in Algorithm~\ref{alg:sampling_alg}. 

\end{updaterequired}

\begin{updaterequired}[black]
\subsection{Conditional Sampling}
\label{subsec:conditional}

Finally, we demonstrate the straightforward extension of our approach to conditional simulation. We apply DDO to sample the conditional distribution $\mu(\cdot|y)$ for a parameter $u$ supported on $H$ given a relevant observation $y \in \R^m$ for inferring $u$. Specifically, we aim to solve a Bayesian inverse problem---modeling a posterior distribution---where the observation $y$ is typically assumed to arise from the additive noise model $y = \mathcal{F}(u) + \eta$ where $\mathcal{F}\colon H \rightarrow \R^m$ is a forward operator and $\eta \in \R^m$ is a noise random variable that is independent of $u$. The noise model induces a likelihood function $\mu(y|u)$ that together with a prior measure $\mu(u)$ for the parameter yields the posterior measure from Bayes' rule as: $\mu(u|y) \propto \mu(y|u)\mu(u)$.

To extend the DDO framework to sample conditionally, we first consider the  perturbations to the data samples in~\eqref{eq:rough_perturbation} with the data sample drawn from the posterior distribution. That is,
$$v = u + \eta,  \quad u \sim \mu(\cdot|y), \quad \eta \sim \mu_0.$$ 
The resulting random variable $v$ has probability measure $\nu(\cdot|y) = \mu(\cdot|y) * \mu_0(\cdot)$ that depends on the observation $y$. Under the same assumptions as in Section~\ref{subsec:denoisings_score_matching}, one can define the logarithmic derivative of $\nu(\cdot|y)$ denoted by $D_{H_{\mu_0}} \Phi(\cdot;y)= D_{H_{\mu_0}} \log \frac{d\nu(\cdot;y)}{d \mu_0}$. The following result shows that we can approximate the logarithmic derivatives of these measures (depending by $y$) using a parametric mapping \(G_\theta\colon H \times \R^m \to H_{\mu_0}^*\) by solving a denoising score matching problem. The proof follows identically from the one for Theorem~\ref{thm:conditional_score_informal} and hence is omitted. 
\begin{theorem}
    Under integrability assumptions on \(D_{H_{\mu_0}} \Phi\) and \(G_\theta\), the minimizers of the score matching problem for conditional sampling
    \begin{equation}
        \min_{\theta \in \mathbb{R}^p} \E_{y \sim \mu(y)} \E_{v \sim \nu(\cdot|y)} \|D_{H_{\mu_0}}\Phi(v) - G_\theta (v,y)\|_{H^*_{\mu_0}}^2.
    \end{equation}
    are identical to the minimizers of the problem
    \begin{equation} \label{eq:denoising_score_matching_posterior}
        \min_{\theta \in \mathbb{R}^p} \E_{y \sim \mu(y)} \E_{u \sim \mu(\cdot|y)} \E_{w \sim \gamma^u} \|D_{H_{\mu_0}} \Psi (w;u) - G_\theta(w,y) \|_{H_{\mu_0}^*}^2,
    \end{equation}    
    where $D_{H_{\mu_0}} \Psi$ has the form in~\eqref{eq:gaussian_score}.
\end{theorem}
After identifying the parametric approximation to the logarithmic derivative of the conditional measure by solving~\eqref{eq:denoising_score_matching_posterior}, one can use the resulting parametric map to construct a Langevin sampling algorithm whose stationary distribution is (approximately) $\nu(\cdot|y)$. Moreover, as discussed in Sections~\ref{subsec:smoothing_operators} and~\ref{subsec:multiple_noise_scales}, we can introduce smoothing operators in the learning problem and consider multiple noise scales to sample our target posterior measure via a sequence of less noisy target distributions using annealed Langevin dynamics.

\end{updaterequired}


\section{Numerical Experiments}
\label{sec:numerics}

In all examples, we use the Fourier neural operator (FNO)~\citep{li2020fourier}, U-shaped neural operator(UNO)~\citep{rahman2023u} as they are well-defined architecture for maps
between Hilbert spaces \cite{li2020fourier,kovachki2021onuniversal}. The goal of our numerics is to showcase the simple 
message that by employing trace-class noise and a consistent architecture for function space data, we obtain dimension (i.e., resolution)-independent results, observed by varying the discretization of the data. All experiments are done by solving
\eqref{eq:langevin} in a way similar to~\cite{song2019generative}, generalized to function spaces; see \appref{sec:disc_langevin}. 

\begin{figure*}[!htbp]
    \vspace{-0.5em}
    \begin{minipage}{0.05\textwidth}
        \vspace{0.5em}
        \centering
        \includegraphics[width=\textwidth]{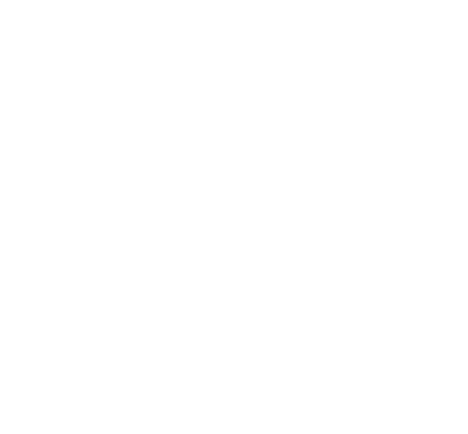}
        \vspace{-2em}
        \caption*{\textcolor{white}{(a)}}
    \end{minipage}\hfill 
    \begin{minipage}{0.27\textwidth}
        \centering
        \includegraphics[width=\textwidth]{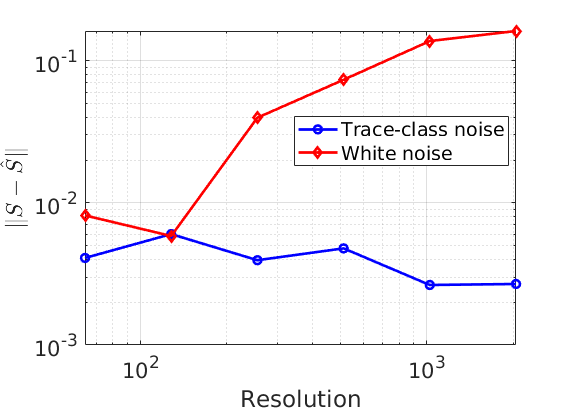}
        \vspace{-2em}
        \caption*{(a)}
    \end{minipage}\hfill 
    \begin{minipage}{0.3\textwidth}
        \centering
        \includegraphics[width=\textwidth]{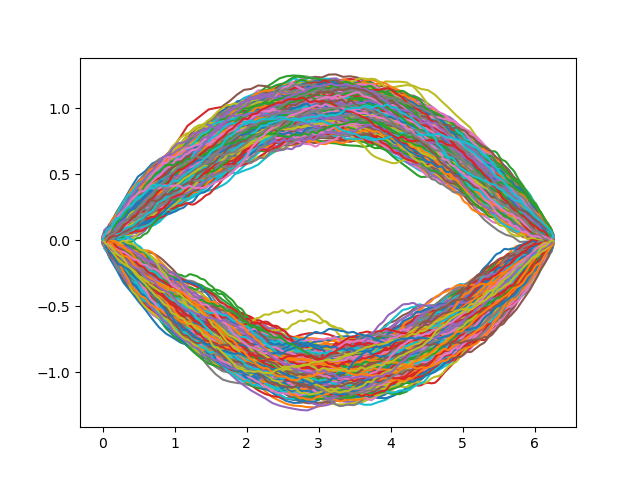}
        \vspace{-2.5em}
        \caption*{(b)}
    \end{minipage}\hfill
    \begin{minipage}{0.3\textwidth}
        \centering
        \includegraphics[width=\textwidth]{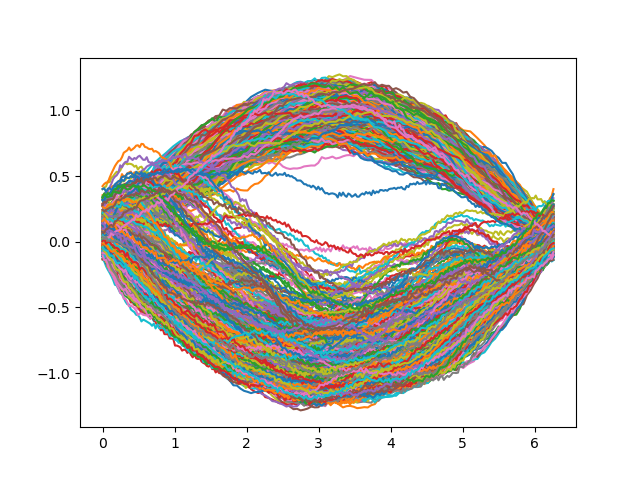}
        \vspace{-2.5em}
        \caption*{(c)}
    \end{minipage}\hfill
    \begin{minipage}{0.05\textwidth}
        \centering
        \includegraphics[width=\textwidth]{figs/background.png}
        \vspace{-2em}
        \caption*{\textcolor{white}{(a)}}
    \end{minipage}
    \vspace{-0.5em}
     \caption{\textbf{Gaussian mixture} (\secref{subsec:num_gaussian_mixture}): \textbf{(a)} Uniform-norm error in the average spectra of samples when using trace-class noise vs. white noise. \textbf{(b)} Generated samples
     at a resolution of 256 with trace-class noise and \textbf{(c)} with white noise.}
     \label{fig:gaussian_mixture_spectrum}
\end{figure*}

\begin{figure*}[!htbp]
    \begin{minipage}{0.22\textwidth}
        \centering
        \includegraphics[width=\textwidth]{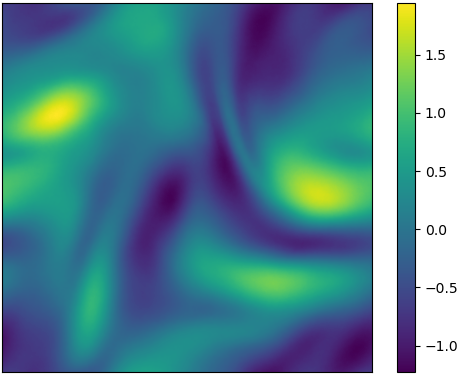}
        \vspace{-1.5em}
        \caption*{(a)}
    \end{minipage}\hfill 
    \begin{minipage}{0.25\textwidth}
        \centering
        \includegraphics[width=\textwidth]{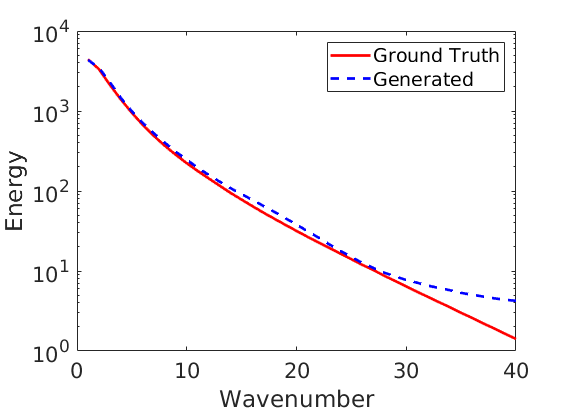}
        \vspace{-2em}
        \caption*{(b)}
    \end{minipage}\hfill
    \begin{minipage}{0.25\textwidth}
        \centering
        \includegraphics[width=\textwidth]{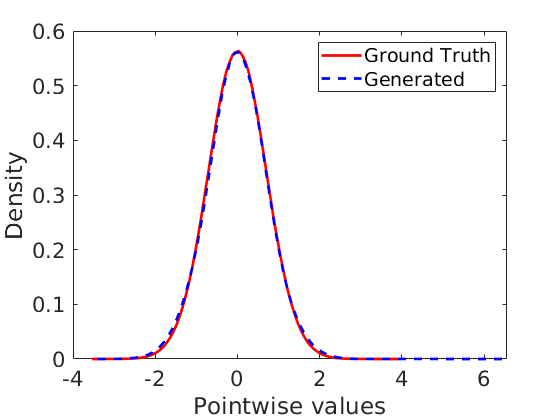}
        \vspace{-2em}
        \caption*{(c)}
    \end{minipage}
    \begin{minipage}{0.25\textwidth}
        \centering
        \includegraphics[width=\textwidth]{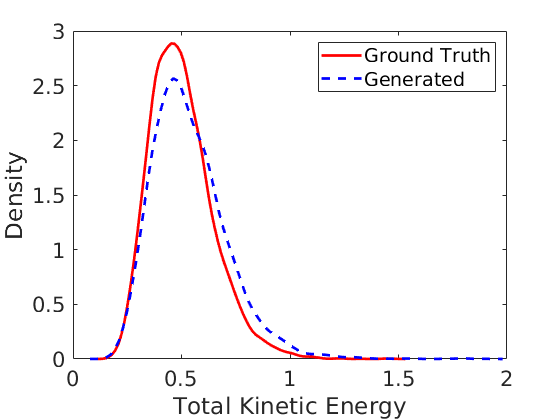}
        \vspace{-2em}
        \caption*{(d)}
    \end{minipage}
    \vspace{-0.5em}
     \caption{\textbf{Navier-Stokes} (\secref{subsec:num_navier_stokes}): \textbf{(a)} Generated sample at the resolution \(1024 \times 1024\) with a model trained at \(128 \times 128\), \textbf{(b)} Spectrum,
     \textbf{(c)} Pointwise value density, \textbf{(d)} Total kinetic energy density of samples from the model vs. the data.}
     \label{fig:num_navier_stokes}
\end{figure*}

\subsection{Gaussian Mixture}
\label{subsec:num_gaussian_mixture}
We consider a Gaussian mixture model by sampling a Gaussian random field (GRF) on the domain \((0,2\pi)\) and assigning it one of two
mean functions with a fixed probability. Details on the precise construction can be found in \appref{subsec:exp_gaussianmixutre}.
We fix a FNO model architecture and train DDO on various discretizations of the data using either trace-class noise 
or white noise. In Figure~\ref{fig:gaussian_mixture_spectrum}(a) we compare the uniform (or sup) norm error in the spectrum 
of the true and generated data for the two types of noise. We see that while white noise achieves small errors at low resolutions,
its error grows as we refine the resolution. On the other hand, trace-class noise  achieves a consistent error across many resolutions. 
Indeed even at a resolution of 256, the trace-class noise model captures the right 
distribution in Figure~\ref{fig:gaussian_mixture_spectrum}(b), unlike the white noise model in Figure~\ref{fig:gaussian_mixture_spectrum}(c); see \appref{subsec:exp_gaussianmixutre}
for further visualizations and \appref{sec:num_smoothing} for an example that uses the smoothing operators described in \secref{subsec:smoothing_operators}. This is because as
we refine the resolution, the model trained with white noise has to capture progressively higher frequency functions 
and thus it fails to do so with a fixed capacity model. Trace-class noise, on the other hand, has a convergent Fourier spectrum that the model can capture
independently of the discretization. The white noise issue can be fixed by designing larger architectures and more sampling steps, but this yields a model where both the number of parameters and sampling steps need to increase with dimension. Therefore algorithms designed with white noise cannot be expected to scale to arbitrarily large resolutions. 

\subsection{Navier-Stokes}
\label{subsec:num_navier_stokes}
Next, we consider the vorticity form of the Navier-Stokes equation on the 2D-torus with a Reynolds number of 500. We develop a solver for this problem and solve it up to a fixed time for a fixed initial condition with different forcing functions generated by a GRF.
The data distribution is therefore a pushforward of a Guassian under a non-linear map and is therefore
non-Gaussian. Details are given in \appref{subsec:exp_navierstokes}. We train a DDO with FNO based model with data 
at a fixed \(128 \times 128\) resolution and \(L^2 (\mathbb{T}^2; \sR)\) valued noise. We observe that the trained DDO accurately generate the function valued data learned from the underlying distribution. In Figure~\ref{fig:num_navier_stokes}(b-d), we compare statistics
relevant for turbulence analysis from the data and the samples from the model, verifying that we are able to
capture the true distribution~\cite{li2021learningdissipative}. In Figure~\ref{fig:num_navier_stokes}(a) we show a
sample from the model generated at a \(1024 \times 1024\) resolution without any re-training; more samples are visualized
in \appref{subsec:exp_navierstokes}. In particular,
our model generalizes to high resolutions at no extra cost, performing super-resolution natively. Such a method
has powerful applications for learning the invariant measures of dissipative dynamical systems which can used
for turbulance analysis and climate science~\cite{temam1988infinite}.

\subsection{Volcano Dataset}
\label{subsec:volcano}

For the following experiments we use the volcano dataset originally proposed in GANO \cite{rahman2022generative}, and UNO as the base architecture. The volcano InSAR dataset consists of 4096 data points of spatial resolution $128 \times 128$, derived from raw interferograms produced from satellites covering the Long Valley Caldera near Mammoth Lakes, California, United States. Since the dataset consists of relatively few examples, we employ a light amount of data augmentation during training in the form of random horizontal and vertical flips.
We present key elements of our loss function and architecture below and provide further details about these experiments in \appref{app:volcano}, e.g. learning and hyperparameter details.

Instead of manually comparing histograms of these evaluation metrics to their respective statistics computed on the training set, here we quantitatively measure how close their histograms are by measuring the 1D Wasserstein distance between them. That is, we define:
\begin{align} \label{eq:volcano_metrics}
\wvar & = W_1( \varfn(\theta(\mathbf{u})), \varfn(\theta(\tilde{\mathbf{u}})) ) \\
\wskew & = W_1( \skewfn(\theta(\mathbf{u})), \skewfn(\theta(\tilde{\mathbf{u}})) ) \\
\wtotal & = \wvar + \wskew
\end{align}
where $\mathbf{u} = \{ \mathbf{u}_{i} \}_{i=1}^{N}$ denotes the training set and conversely $\tilde{\mathbf{u}} = \{ \tilde{\mathbf{u}}_j \}_{j=1}^{M}$ generated samples from the diffusion model. We set $M = 256$ for fast metric tracking since it is computationally expensive to generate many samples. During training, we periodically evaluate both evaluation metrics and keep track of checkpoints corresponding to the smallest values seen so far, then when we perform a final evaluation of the model we use the checkpoint corresponding to the smallest $\wtotal$ seen so far.

For these experiments we sample from a 2D GRF based on the RBF kernel (\secref{app:volcano}), and therefore an important hyperparameter to tune is $\gamma$, the smoothness of the noise.

When the best model has been selected, we re-compute the Wasserstein metrics using $M = 1024$ instead. While $M = 256$ was used during training, we find it did not show significant changes in the computed statistics.

\begin{figure}[h!]
    \centering
    \begin{subfigure}[b]{\textwidth}
        \centering
        \includegraphics[width=0.9\textwidth,trim=0 280 0 0,clip]{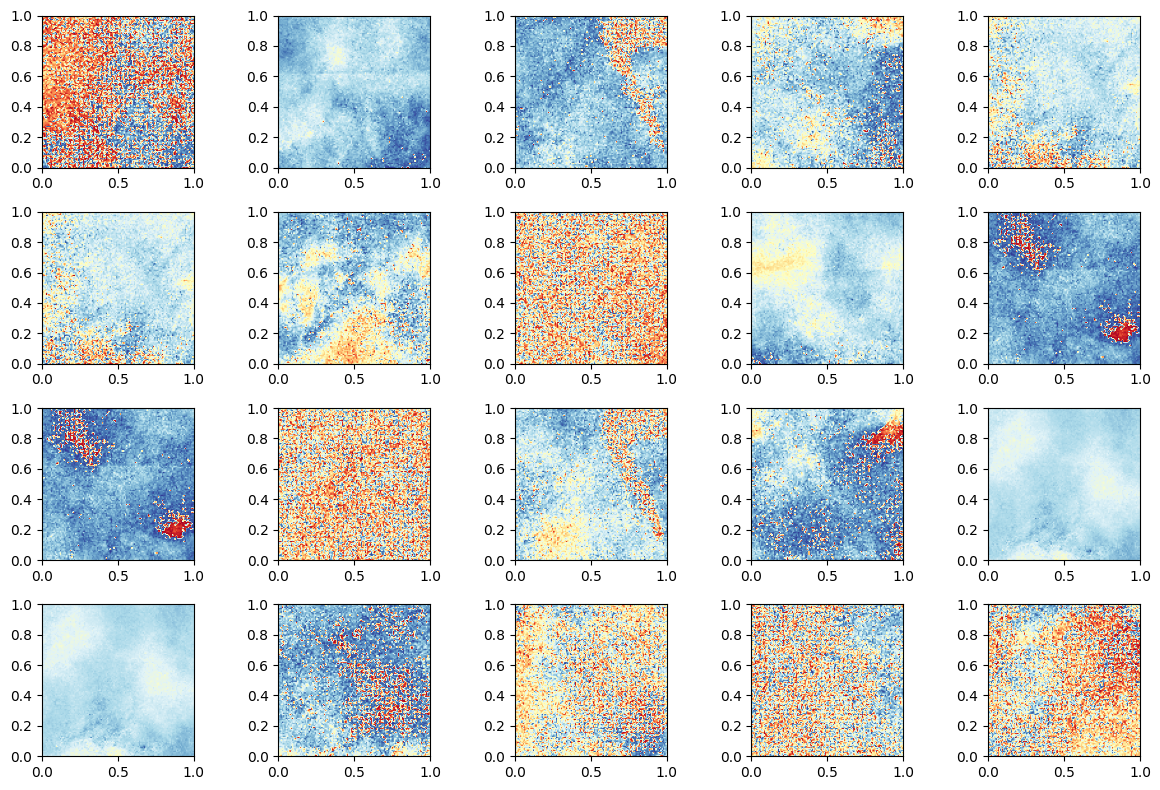}
        \caption{Random samples from the best performing diffusion model. 
        }
        \label{fig:sbgm_samples}
    \end{subfigure} \\
    \vspace{0.1cm}
    \begin{subfigure}[b]{0.45\textwidth}
        \centering
        \includegraphics[width=\textwidth]{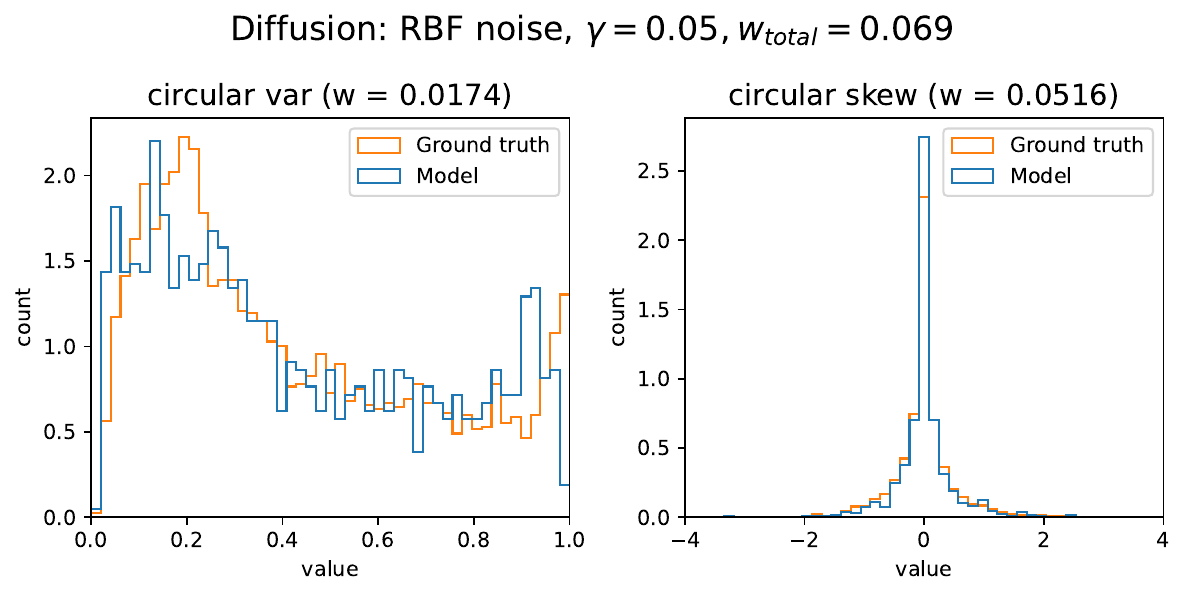}
        \caption{Best performing diffusion model}
        \label{fig:sbgm_hist}
    \end{subfigure} \begin{subfigure}[b]{0.45\textwidth}
        \centering
        \includegraphics[width=\textwidth]{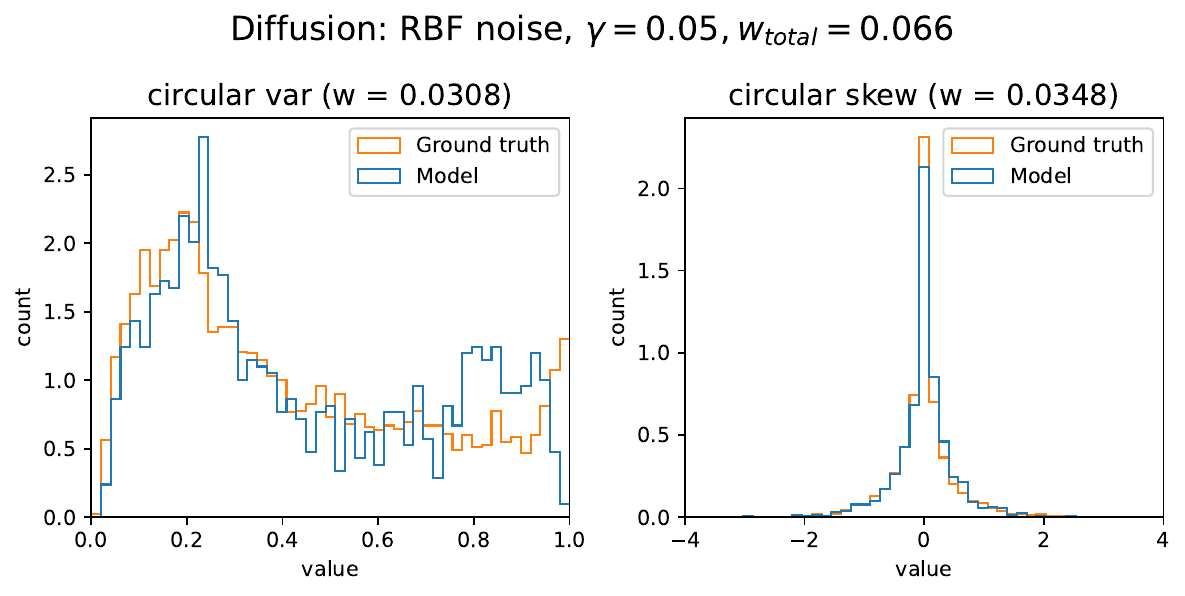}
        \caption{Best performing GANO model}
        \label{fig:gano_hist}
    \end{subfigure}   
    \caption{\textbf{Results of Volcano dataset experiments} (\secref{subsec:volcano}): Samples from the best performing FNO diffusion model with an RBF scale of $\gamma = 0.05$. For both histograms, $M = 1024$ generated samples were used to compute skew and variance.}
    \label{fig:sbgm}
\end{figure}



\paragraph{Results}
In Figure \ref{fig:sbgm} we demonstrate samples and histograms produced by our best performing diffusion model, with smoothness parameter $\lambda = 0.05$. These are shown in Figures \ref{fig:sbgm_samples} and \ref{fig:sbgm_hist} respectively, and a reference GANO model is also shown in \ref{fig:gano_hist}. We can see that our model is able to accurately learn the ground truth function, as indicated by the histograms shown in Figure \ref{fig:sbgm_hist}. Due to the noisiness of this dataset we found that the best results were achieved with an RBF scale parameter of $\gamma = 0.05$ (see \secref{app:volcano}), which corresponds to very rough levels of noise.

At generation time, in order to generate at twice the resolution we construct a meshgrid that is twice as granular as that used in training. For example, if $s^2$ is the original resolution then we compute a meshgrid $x \in [0,1]^{(s^2)^2 \times 2}$ and use that to sample RBF noise at twice the resolution. Concretely, we train DDO on a downsampled version of Volcano at $60 \times 60$ resolution and double the resolution at generation time to $120 \times 120$. In order to quantitatively evaluate this task we still compute skew and variance metrics as previously described (Equation \ref{eq:volcano_metrics}), but now these are computed between the super-resolution samples and the original $120 \times 120$ resolution dataset. These results are shown in Figure \ref{fig:diffusion_superres}, and we demonstrate results comparing independent Gaussian noise to different values of RBF smoothness $\gamma$ used during training. As expected, independent noise performs abysmally (Figure \ref{fig:d_sr_hist_wn}), and we also found that $\gamma = 0.05$ did not perform well (Figure \ref{fig:d_sr_hist_005}). However, smoother levels of RBF noise performed well, and the best results were achieved with $\gamma = 0.2$ (Figure \ref{fig:d_sr_hist_02}). This demonstrates the ability of our model to query the sampled function.


\subsection{MNIST-SDF Dataset}
\label{subsec:mnist-sdf}


Next, to demonstrate the efficacy of the proposed method on function generation in conjunction with images, \updated{we conduct experiments on MNIST-SDF \citep{sitzmann2019metasdf} and compare the proposed method to GANO \citep{rahman2022generative} and MultilevelDiff \citep{hagemann2023multilevel}.
GANO is an adversarial training-based function space generative model. 
As compared to other concurrent works on diffusion models designed for function spaces~\citep{kerrigan2022diffusion, pidstrigach2023infinite, baldassari2023conditional}, we selected MultilevelDiff since it was already tested on similar two-dimensional datasets, while other models have been limited to one-dimensional datasets. 
In this setting, \citep{bond2024infty} is another a function space diffusion model which aims to model 2D image datasets. However, this model relies on frameworks tailored for high-fidelity image modeling, which includes performing diffusion in latent spaces, thereby making direct comparisons unfair. An indirect comparison with this model is discussed in \appref{app:alias-free}.
}


\updated{
MNIST-SDF is a collection of 2D signed distance functions (SDFs), each of which is extracted by applying a distance transform to every image in the MNIST dataset. As compared to resizing finite-dimensional datasets for various resolutions, this conversion makes the dataset defined on function space, and allows us to consistently compute evaluation statistics like FID \citep{heusel2017gans} and precision-recall \citep{kynkaanniemi2019improved} metrics across different resolutions. Examples of the 2D SDFs are shown in Figure \ref{fig:mnist-sdf-data}(a).
}

\updated{
We aim to train models on 32$\times$32-resolutions and evaluate the FID and precision-recall at different resolutions. Specifically for image datasets, we choose to upsample 32$\times$32 resolution to a 64$\times$64 resolution, since this ensures we can select the number of Fourier modes to represent the data or noise to be higher than the discretization. We will discuss the necessity of the upsampling in the following section. The other experimental details, including the architecture and training procedure, are described in the \appref{app:mnist-sdf}
}

As we find that the classifiers pre-trained for the evaluation metrics are more suitable for the original MNIST-like binary digits than 2D SDFs, we evaluate the metrics after generating binary masks by thresholding the sample SDFs to larger than $0$. The masked 2D SDFs images are illustrated in Figure \ref{fig:mnist-sdf-data}.(b). We follow the styleGAN3's evaluation protocol  \citep{karras2021alias} for the FID and precision-recall.

\begin{figure}[H]

    \centering
    \begin{minipage}[b]{0.28\textwidth}
        \centering
        \includegraphics[width=\textwidth]{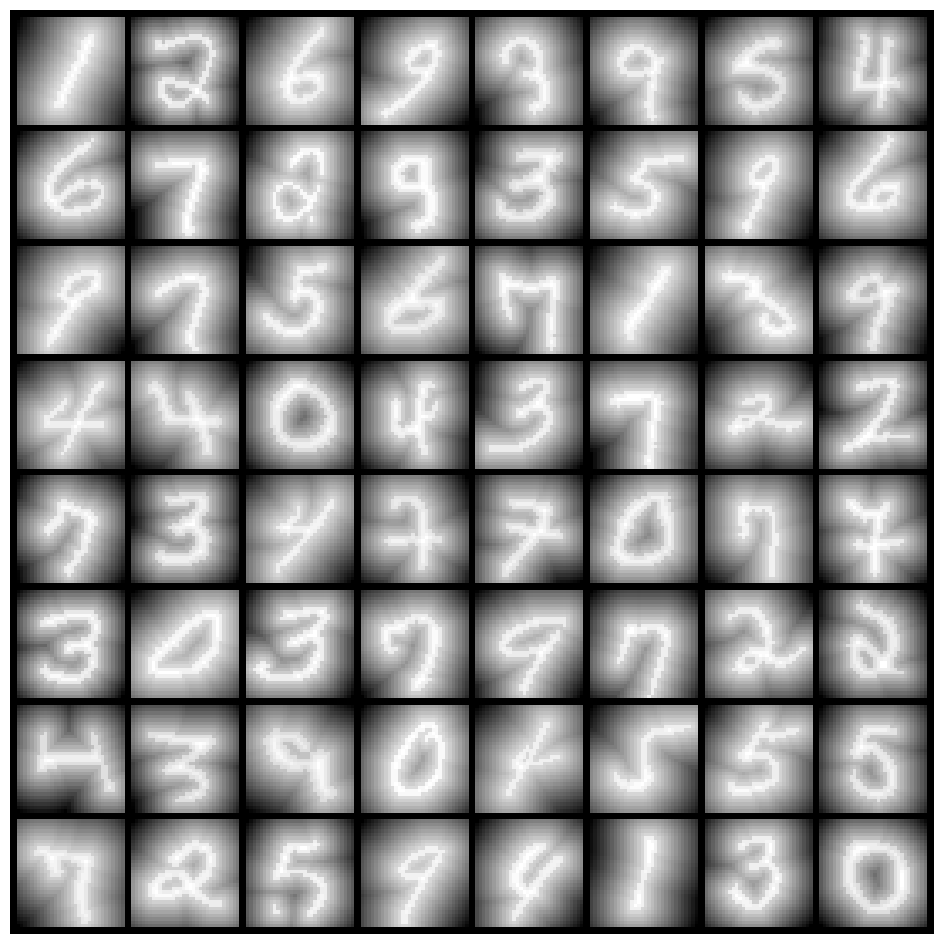}
        \vspace{-6mm}
        \caption*{(a) MNIST-SDF}
    \end{minipage}
    \begin{minipage}[b]{0.02\textwidth}
        \centering
        \includegraphics[width=\textwidth]{figs/background.png}
    \end{minipage}
    \begin{minipage}[b]{0.28\textwidth}
        \centering
        \includegraphics[width=\textwidth]{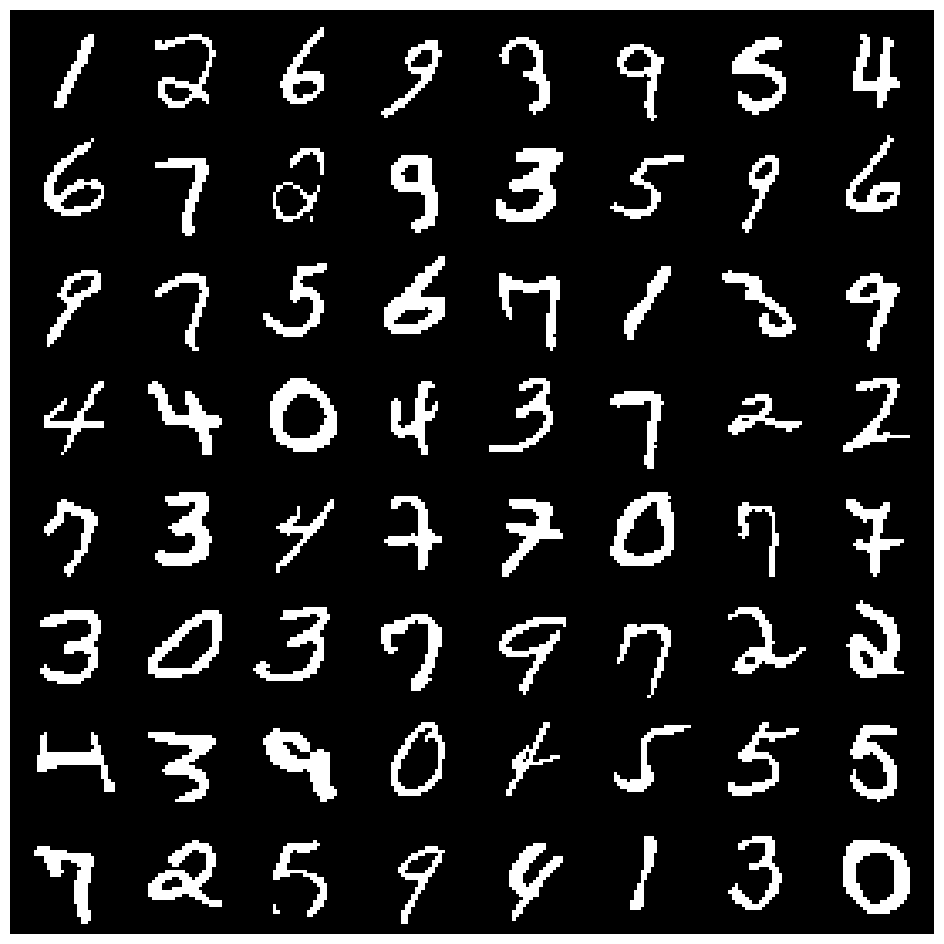}
        \vspace{-6mm}
        \caption*{(b) Masked}
    \end{minipage}
    \vspace{-1mm}

    \caption{\textbf{MNIST-SDF Dataset.} (a) Samples of 2D SDFs in the datasets. A 2D SDF is generated by applying a distance transform to each image in the MNIST dataset. (b) Binary masks extracted by thresholding the sample SDFs where the value is larger than $0$. Every generated samples from the models will be masked before running FID and precision-recall metrics.}
    \label{fig:mnist-sdf-data}
\end{figure}

\paragraph{Upsampling the Finite-dimensional Observations}

Our theory suggests that any sample function in the target data distribution should be smoother than the samples of the noise distribution to satisfy \(\mu(H_{\mu_0}) = 1\). This assumption implies that when we represent the data and noise in Fourier space, the Fourier bases required to describe all noise samples include the basis set of the data. This requirement leaves us some implementation constraints when dealing with the finite observations of the data or noises, especially when the number of bases representing a noise sample is larger than the discretization size. This is often the case since useful distributions like Gaussian measures are often obtained in $d$-dimensional observations while its basis set size is much larger than $d$. In this section, we will discuss how to address the constraints efficiently. 

For most experimental scenarios, we assume that we can only access $d$-dimensional observations of the data. This means that we can treat the size of the data's basis set to be $d$. While the true size could be larger than $d$, we won't be able to model anything other than observations on the $d$ basis. However, we can learn such discrete observations on the Fourier bases by discrete Fourier transforms and generate arbitrary discretization from them.

On the contrary to data, for some useful distributions in the infinite-dimensional space, the size of basis sets is often larger than the observation size $d$; for example, Gaussian measures. For Gaussian, there exists a positive integer $M < \infty$ such that the $M$-size of the basis set can represent all its samples. When we select $M$ to be higher than $d$ to model the $d$-dimensional data, the components on $M$-$d$ number of bases won't be observed at the given discretization. Thus, any model may fail to generate proper images at any resolution larger than $d$, as the unseen noise components will be introduced.

To address this, for a given $d$-dimensional observations, we propose upsampling to $d'$ such that $d'$ is large enough to $M$. For the 32$\times$32-resolution observation of the MNIST-SDF, we choose $d'$ to 64$\times$64. For the upsampling of the finite observations, this paper follows the filtered upsampling implementation discussed in \citet{karras2021alias}. \updated{While the models are trained on 64$\times$64-resolution, we use only 32 modes in the spectral convolutions (at the lowest level) of the architectures used in DDO, MultilevelDiff, and GANO, which enables them to generate images at the 32$\times$32-resolution as well.}

Note that one can truncate the modes of Gaussians up to $d$ so that the unseen noise components won't be introduced. However, we find that such truncations often generate artifacts in super-resolution tasks; periodic waves are drawn, which are supposed to be straight lines.

\begin{table}[htb!]
\centering
\begin{tabular}{cccccccccc}
\toprule
         & \multicolumn{3}{c}{$64\times64$} & \multicolumn{3}{c}{$128\times128$} & \multicolumn{3}{c}{$256\times256$} \\
\midrule
         & FID   & Prec$^*$      & Rec$^\dagger$  & FID     & Prec      & Rec  & FID      & Prec    & Rec   \\
\midrule
GANO & 3.41  & \textbf{0.75}  & 0.63    & 13.05   & 0.68  & 0.50    & 23.89    & 0.60         & 0.32    \\
\updated{MultilevelDiff} & \updated{35.09} 
& \updated{0.03}  
& \updated{0.06} 
& \updated{201.08} 
& \updated{0.00} %
& \updated{0.00} %
& \updated{365.90} 
& \updated{0.00} %
& \updated{0.00} %
\\
DDO (Ours) & \textbf{2.74} & 0.73 & \textbf{0.68} & \textbf{7.96} & \textbf{0.71} & \textbf{0.60} & \textbf{17.76} & \textbf{0.65} & \textbf{0.39} \\
\bottomrule
\multicolumn{1}{l}{\scriptsize $^*$Precision. $^\dagger$Recall.}
\end{tabular}
\caption{Results of MNIST-SDF experiments.}
\captionsetup{justification=centering}
\label{table:mnist-sdf-results}
\end{table}

\begin{figure}[htb!]
    \centering
    \begin{minipage}[b]{0.02\textwidth}
        \centering
        \includegraphics[width=\textwidth]{figs/background.png}
    \end{minipage}
    \begin{minipage}[b]{0.05\textwidth}
        \centering
        \includegraphics[width=\textwidth]{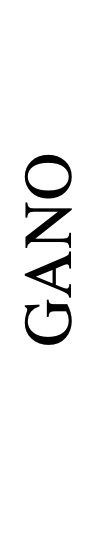}
    \end{minipage}
    \begin{minipage}[b]{0.29\textwidth}
        \centering
        \includegraphics[width=\textwidth]{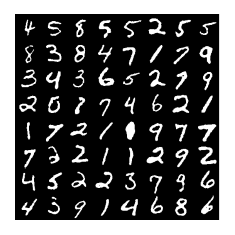}
    \end{minipage}
    \hfill
    \begin{minipage}[b]{0.29\textwidth}
        \centering
        \includegraphics[width=\textwidth]{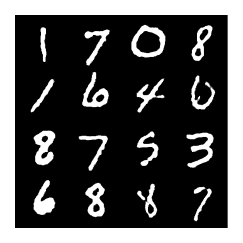}
    \end{minipage}
    \hfill
    \begin{minipage}[b]{0.29\textwidth}
        \centering
        \includegraphics[width=\textwidth]{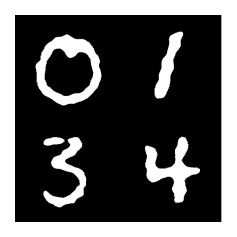}
    \end{minipage}
    \hfill
    \begin{minipage}[b]{0.02\textwidth}
        \centering
        \includegraphics[width=\textwidth]{figs/background.png}
    \end{minipage}

    \begin{minipage}[b]{0.02\textwidth}
        \centering
        \includegraphics[width=\textwidth]{figs/background.png}
    \end{minipage}
    \begin{minipage}[b]{0.05\textwidth}
        \centering
        \includegraphics[width=\textwidth]{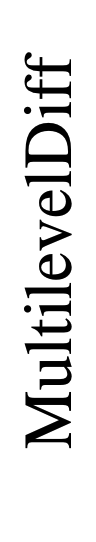}
    \end{minipage}
    \begin{minipage}[b]{0.29\textwidth}
        \centering
        \includegraphics[width=\textwidth]{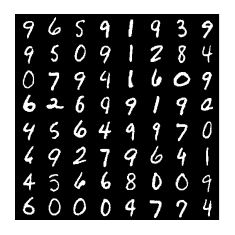}
    \end{minipage}
    \hfill
    \begin{minipage}[b]{0.29\textwidth}
        \centering
        \includegraphics[width=\textwidth]{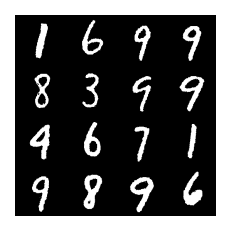}
    \end{minipage}
    \hfill
    \begin{minipage}[b]{0.29\textwidth}
        \centering
        \includegraphics[width=\textwidth]{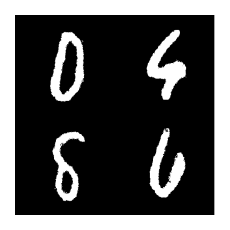}
    \end{minipage}
    \hfill
    \begin{minipage}[b]{0.02\textwidth}
        \centering
        \includegraphics[width=\textwidth]{figs/background.png}
    \end{minipage}

    \begin{minipage}[b]{0.02\textwidth}
        \centering
        \includegraphics[width=\textwidth]{figs/background.png}

        \vspace{-4mm}

        \caption*{\textcolor{white}{()}}
    \end{minipage}
    \begin{minipage}[b]{0.05\textwidth}
        \centering
        \includegraphics[width=\textwidth]{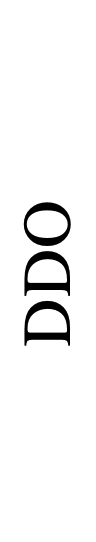}

        \vspace{-4mm}

        \caption*{\textcolor{white}{()}}
    \end{minipage}
    \begin{minipage}[b]{0.29\textwidth}
        \centering
        \includegraphics[width=\textwidth]{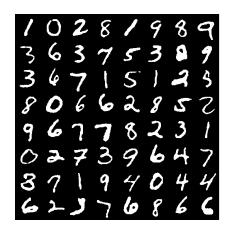}

        \vspace{-4mm}

        \caption*{(a) \,$64\times64$}
    \end{minipage}
    \hfill
    \begin{minipage}[b]{0.29\textwidth}
        \centering
        \includegraphics[width=\textwidth]{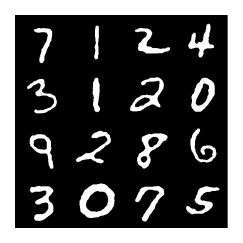}

        \vspace{-4mm}

        \caption*{(b) \,$128\times128$}
    \end{minipage}
    \hfill
    \begin{minipage}[b]{0.29\textwidth}
        \centering
        \includegraphics[width=\textwidth]{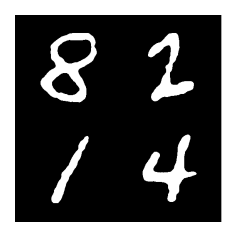}

        \vspace{-4mm}

        \caption*{(c) \,$256\times256$}
    \end{minipage}
    \hfill
    \begin{minipage}[b]{0.02\textwidth}
        \centering
        \includegraphics[width=\textwidth]{figs/background.png}

        \vspace{-4mm}

        \caption*{\textcolor{white}{()}}
    \end{minipage}

    \caption{\textbf{Generated Samples} (\secref{subsec:mnist-sdf}): Generated samples (masked) of the learned GANO, MultilevelDiff, and DDO models at various resolutions; (a) 64$\times$64, (b) 128$\times$128, and (c) 256$\times$256-resolutions. All images are plotted according to their relative resolutions. All models are trained on 64$\times$64-resolution images, which are upsampled from 32$\times$32-resolution observations of 2D SDFs.}
    \label{fig:mnist-sdf-samples}
\end{figure}

\paragraph{Results}
\begin{updaterequired}[black]
Table \ref{table:mnist-sdf-results} shows the FID and precision-recall metrics evaluated from learned models in the MNIST-SDF experiments. DDO outperforms both GANO and MultilevelDiff baselines, except for the precision at the training resolution (64$\times$64). Moreover, DDO demonstrates higher recall at all resolutions, which is coherent with the general property of diffusion-based models, whose objectives are to minimize the KL divergence between the data distribution and the model. Note that such connections to the DDO's objective are briefly discussed in Lemma \ref{lemma:wasserstein_approx}. Interestingly, MultilevelDiff exhibits a notable performance lag in comparison to the other two. The discussion related to this performance gap will be presented while examining the generated samples.

Figure \ref{fig:mnist-sdf-samples} illustrates the generated samples at various resolutions produced by our DDO model and the baselines\footnote{Figure \ref{fig:mnist-sdf-samples-8x8} illustrates additional samples at various resolutions produced by all models.}. Visually, all models appear to achieve high-quality generation of all digits across all resolutions. However, in the case of MultilevelDiff, the variation in digit shapes and styles is noticeably lower. This tendency becomes more pronounced at higher resolutions (see Figure \ref{fig:mnist-sdf-samples-8x8}), which only leads to generating a few digits. This characteristic explains the poor FID scores and precision-recall metrics observed for MultilevelDiff above. In contrast, both DDO and GANO not only achieve high-quality generation across all resolutions but also maintain the variation in styles and digits observed at the training resolution.  

On the other hand, the generated samples from both DDO and GANO show curved boundaries, which the MNIST-SDF dataset doesn't have. This artifact originates from the spectral convolution, as it cut off higher frequency components than its parameters' highest mode. Again, this results in the loss of some high-frequency components which would be necessary to represent arbitrary curved lines with no artifacts. This observation also emphasizes the importance of upsampling during training instead of truncating the noise. We leave addressing such artifacts for future works. 
\end{updaterequired}

\begin{updaterequired}[black]
\subsection{Darcy Flow Bayesian Inverse Problem}
\label{subsec:darcyflow}

We apply our method to the geophysical inverse problem of recovering a subsurface permeability field from pointwise observations of the pressure at the surface, which is also known as the Darcy flow inverse problem. We refer to Section~\ref{subsec:conditional} for the abstract formulation of such problems and how our framework can be applied to solve them. In this setting, the forward model is defined as the solution of the following elliptic partial differential equation (PDE) on the domain $\mathcal{D} = (0,1)^2$ given by
\begin{subequations}
\begin{align}
-\nabla (a(s) \nabla p(s)) &= 1, \quad s \in \mathcal{D} \\
p(s) &= 0, \quad s \in \partial\mathcal{D},
\end{align}
\end{subequations}
where $a(s) \in \R_{+}$ represents the strictly positive permeability and $p(s)$ is the pressure. We consider the parameter $u = \log(a)$ defining the log-permeability, which we recover from 64 observations on an $8\times 8$ grid, i.e., $y(s_i) = p(s_i)$ for locations $s_i$ in the interior of domain $\mathcal{D}$ on a regular grid. To invert for the solution, we consider a log-normal prior distribution for the log-permeability $u = \log(a) \sim \mathcal{N}(0,(-\Delta + \tau I)^{-2}),$ where $\Delta$ is the Laplacian operator and we set $\tau = 9$. We assume the observations are corrupted with  Gaussian observational noise  $\eta \sim \mathcal{N}(0,\Sigma_m)$ where $\Sigma_m$ is a diagonal matrix with entries $\mathbb{E}[p(s_i)^2] / 5$, implying a signal-to-noise ratio of 5.

To generate data at a resolution of $64 \times 64$ for training, we solve the PDE on a regular grid of resolution $1024 \times 1024$ for each realization of the permeability field $a^i = \exp(u^i)$ for $u^i \sim \mu(\cdot)$ to obtain the observations $y^i = \mathcal{F}(u^i) + \eta^i$. This process yields pairs of observations $(u^i,y^i)$ drawn from the joint measure for the parameter and observation. Finally, we downsample the high-resolution observation to a resolution of $64 \times 64$ for the training data. 

To quantify the performance of the learned posteriors, we evaluate the relative errors of the sample mean and variance relative to those of the posterior obtained through a Markov chain Monte Carlo (MCMC) simulation based on a pre-conditioned Crank Nicolson method that is consistent in function space~\citep{cotter2013mcmc}; see the first row of Figure \ref{fig:darcyflow-samples-64x64} for the MCMC results. The relative errors are defined by
\begin{align*}
     \mathcal{E}_{\textrm{mean}} \coloneqq \frac{ \Vert m_{\scriptscriptstyle \mathrm{MCMC}} - m_{\mathrm{model}} \Vert^2_{L^2}} { \Vert m_{\scriptscriptstyle \mathrm{MCMC}} \Vert^2_{L^2} } \quad\textrm{and}\quad
     \mathcal{E}_{\textrm{var}} \coloneqq \frac{ \Vert \sigma^2_{\scriptscriptstyle \mathrm{MCMC}} - \sigma^2_{\mathrm{model}} \Vert^2_{L^2} }{ \Vert \sigma^2_{\scriptscriptstyle \mathrm{MCMC}} \Vert^2_{L^2} },
\end{align*}
where $(m_{\scriptscriptstyle \mathrm{MCMC}}, \sigma^2_{\scriptscriptstyle \mathrm{MCMC}})$ and $(m_{\mathrm{model}}, \sigma^2_{\mathrm{model}})$ are the sample mean and variance pairs for MCMC and the model, respectively. To compute the sample means and variances, we use 10,000 samples generated by the model and 10,000 from MCMC. Similar to the MNIST-SDF experiments, we compare the proposed method to GANO and MultilevelDiff.

\begin{table}[htb!]
\captionsetup{skip=2pt}
\centering
\begin{tabular}{ccccccc}
\toprule
         & \multicolumn{2}{c}{$64\times64$} & \multicolumn{2}{c}{$128\times128$} & \multicolumn{2}{c}{$256\times256$} \\
\midrule
         & $\mathcal{E}_{\textrm{mean}}^*$ & $\mathcal{E}_{\textrm{var}}^\dagger$ 
         & $\mathcal{E}_{\textrm{mean}}$ & $\mathcal{E}_{\textrm{var}}$
         & $\mathcal{E}_{\textrm{mean}}$ & $\mathcal{E}_{\textrm{var}}$   \\
\midrule
GANO 
& 0.28
& 0.75
& 0.32
& 0.76
& 0.34
& 0.77 \\
MultilevelDiff 
& 0.26
& 0.52
& 0.23
& 0.61
& 0.23
& 0.72
\\
DDO (Ours) 
& 0.26
& 0.69
& 0.27
& 0.78
& 0.28
& 0.79 \\
\bottomrule
\multicolumn{5}{l}{\scriptsize $^*$Scaled error of mean. $^\dagger$Scaled error of variance.}
\end{tabular}
\caption{Results of Darcy flow Bayesian inverse problem.}
\captionsetup{justification=centering}
\label{table:darcyflow-results}
\end{table}

\paragraph{Results}



Table \ref{table:darcyflow-results} presents the relative errors of the sample mean and variance relative to those of the posterior computed through the Markov chain Monte Carlo (MCMC) simulation. In general, DDO outperforms the GANO baseline, showcasing better alignment with the true posterior. In particular, our model maintains consistently lower relative mean errors at all resolutions in contrast to GANO, highlighting DDO's consistency for sampling at higher resolutions. However, unlike the results from the MNIST-SDF experiments (Section \ref{subsec:mnist-sdf}), the MultilevelDiff baseline achieves the best performance among all methods. Notably, this method also demonstrates better mean and variance errors compared to the other two methods. We will revisit the performance analysis of MultilevelDiff in contrast to MNIST-SDF experiments later in this section.

Figure \ref{fig:darcyflow-mean-var} illustrates the sample mean and variance at various resolutions generated by our DDO model and the baselines, including MCMC. These results help to identify the trends observed in the relative errors. Notably, as shown in Figure \ref{fig:darcyflow-mean-var} (b), GANO's sample variance is significantly lower than MCMC, potentially indicating that mitigating mode collapse in GANO remains challenging. Although the other two models perform better than GANO, neither our DDO nor the MultilevelDiff models the sample variations accurately in comparison to the MCMC's statistics. In addition, DDO and MultilevelDiff also exhibit a tendency for decreased variation as the resolution increases. To further analyze these statistical differences, we examine the generated samples in more detail.

Figure \ref{fig:darcyflow-samples-64x64} depicts the generated samples of the trained models at the training resolution, while the generated samples at higher resolutions are shown in Figures \ref{fig:darcyflow-samples-128x128} and \ref{fig:darcyflow-samples-256x256} (in \appref{app:darcyflow}). For MCMC samples, the reverse C-shaped valley in the center retains low values, with strong variation at the edges across different samples. In contrast, GANO fails to maintain sufficiently low values in the reverse C-shaped region, exhibiting only minor local variations. DDO demonstrates a pattern similar to the MCMC samples but does not manage to achieve sufficiently low values in the center of the field, which appears to contribute to its error. 

As we discussed above, MultilevelDiff demonstrates improved performance on this dataset unlike with MNIST-SDF. This appears to result from its noise design and the inherent characteristics of the Darcy flow solutions. While GANO and DDO rely on Gaussian-based trace-class noise, the distribution of the Darcy flow solutions exhibits heavier tails, which decay more slowly than a Gaussian distribution. In contrast, MultilevelDiff utilizes a combination of a spectral convolution-based kernel and a fixed kernel, where the spectral convolution-based noise facilitates the generation of high-frequency noise more effectively.

Interestingly, the MNIST-SDF experiment results also support the hypothesis that MultilevelDiff's noise design has an advantage in modeling high-frequency components. Despite its notable underperformance (see Figure \ref{fig:mnist-sdf-samples}), MultilevelDiff avoids producing wavy aliasing artifacts, even when it is underfit on the MNIST-SDF dataset. This observation underscores the importance of trace-class noise for diffusion-based models on function space. This suggests that further refinements to noise design could be a promising direction for future research. 
\end{updaterequired}

\begin{figure}[htb!]
\captionsetup{skip=4pt}
\centering
\begin{subfigure}[b]{0.43\textwidth}
    \captionsetup{skip=3pt}
    \centering
    \begin{minipage}[b]{0.07\textwidth}
        \centering
        \includegraphics[width=\textwidth, trim={0 2.9cm 0 3.1cm}, clip=true, bmargin=1mm]{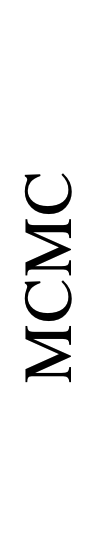}
    \end{minipage}
    \begin{minipage}[b]{0.29\textwidth}
        \centering
        \includegraphics[width=\textwidth, bmargin=1mm]{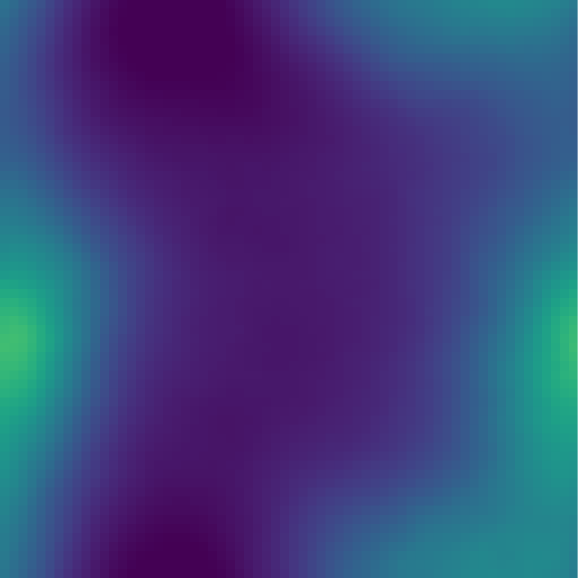}
    \end{minipage}
    \hfill
    \begin{minipage}[b]{0.29\textwidth}
        \centering
        \includegraphics[width=\textwidth, bmargin=1mm]{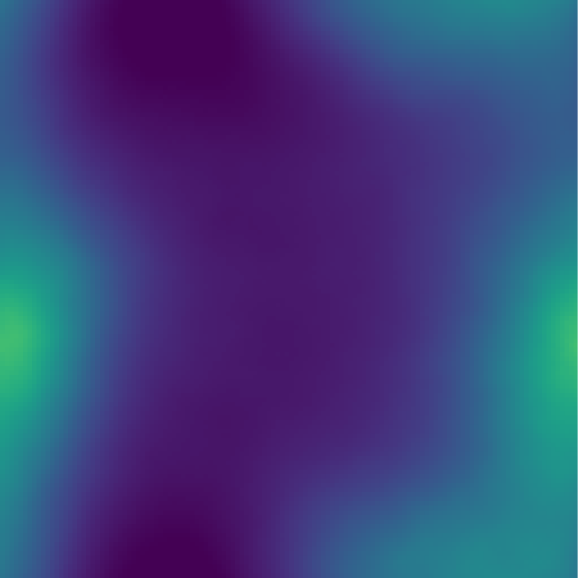}
    \end{minipage}
    \hfill
    \begin{minipage}[b]{0.29\textwidth}
        \centering
        \includegraphics[width=\textwidth, bmargin=1mm]{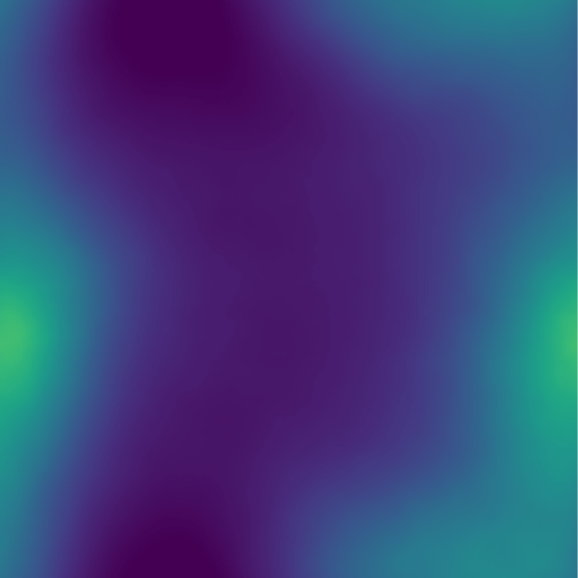}
    \end{minipage}

    \begin{minipage}[b]{0.07\textwidth}
        \centering
        \includegraphics[width=\textwidth, trim={0 3.5cm 0 2.5cm}, clip=true, bmargin=1mm]{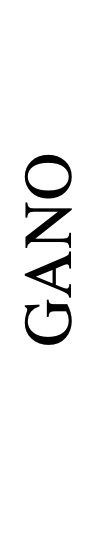}
    \end{minipage}
    \begin{minipage}[b]{0.29\textwidth}
        \centering
        \includegraphics[width=\textwidth, bmargin=1mm]{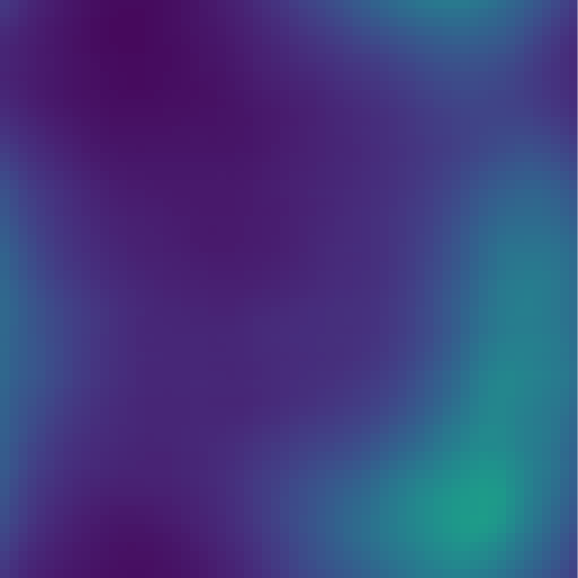}
    \end{minipage}
    \hfill
    \begin{minipage}[b]{0.29\textwidth}
        \centering
        \includegraphics[width=\textwidth, bmargin=1mm]{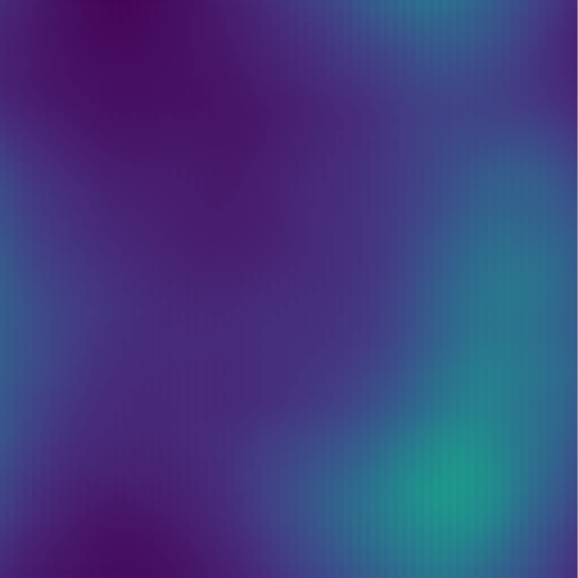}
    \end{minipage}
    \hfill
    \begin{minipage}[b]{0.29\textwidth}
        \centering
        \includegraphics[width=\textwidth, bmargin=1mm]{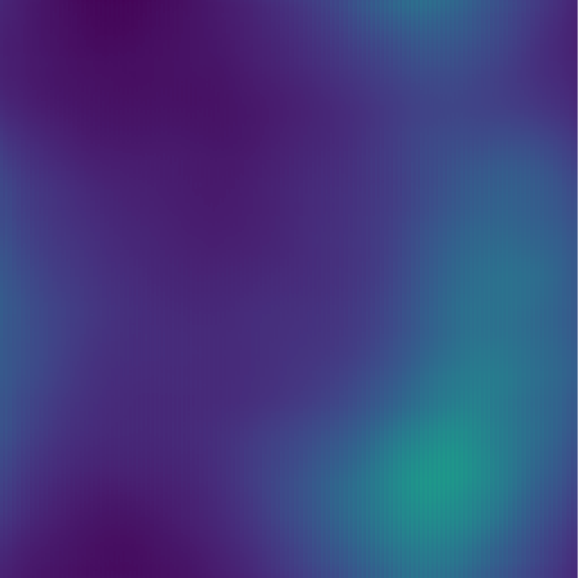}
    \end{minipage}

    \begin{minipage}[b]{0.07\textwidth}
        \centering
        \includegraphics[width=0.9\textwidth, trim={0 2.5cm 0 1.9cm}, clip=true, bmargin=0mm]{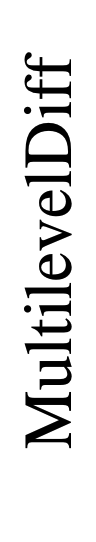}
    \end{minipage}
    \begin{minipage}[b]{0.29\textwidth}
        \centering
        \includegraphics[width=\textwidth, bmargin=1mm]{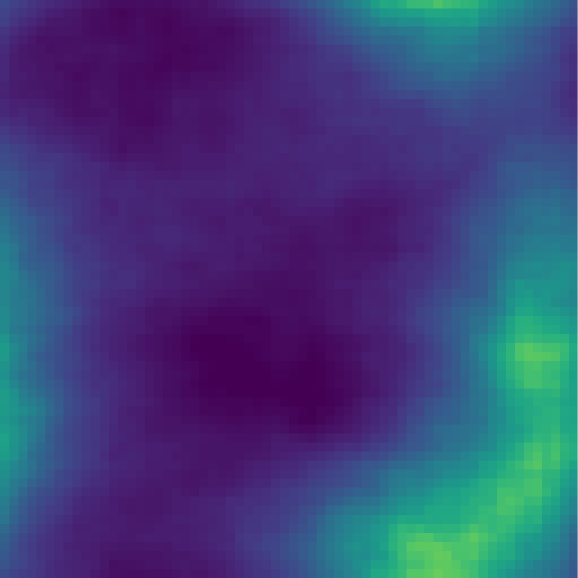}
    \end{minipage}
    \hfill
    \begin{minipage}[b]{0.29\textwidth}
        \centering
        \includegraphics[width=\textwidth, bmargin=1mm]{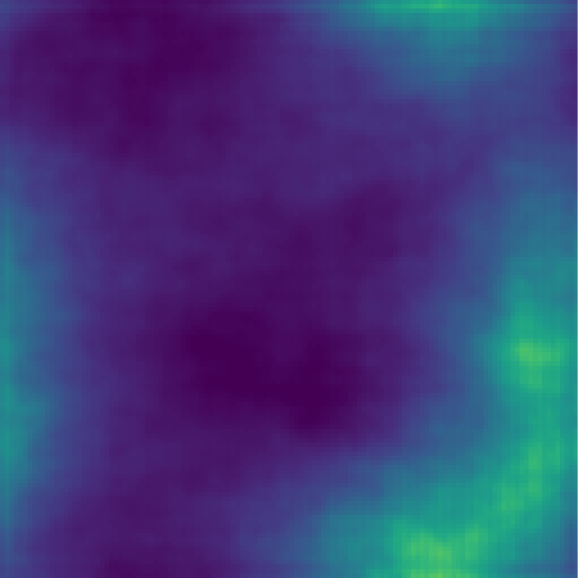}
    \end{minipage}
    \hfill
    \begin{minipage}[b]{0.29\textwidth}
        \centering
        \includegraphics[width=\textwidth, bmargin=1mm]{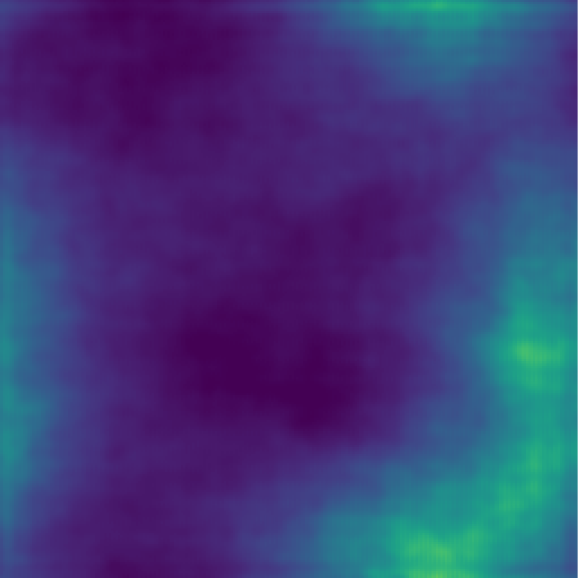}
    \end{minipage}

    \begin{minipage}[b]{0.07\textwidth}
        \captionsetup{skip=2pt}
        \centering
        \includegraphics[width=\textwidth, trim={0 2.5cm 0 3.5cm}, clip=true, bmargin=0.5mm]{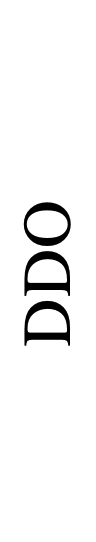}
        \caption*{\textcolor{white}{()}}
    \end{minipage}
    \begin{minipage}[b]{0.29\textwidth}
        \captionsetup{skip=2pt}
        \centering
        \includegraphics[width=\textwidth, bmargin=0.5mm]{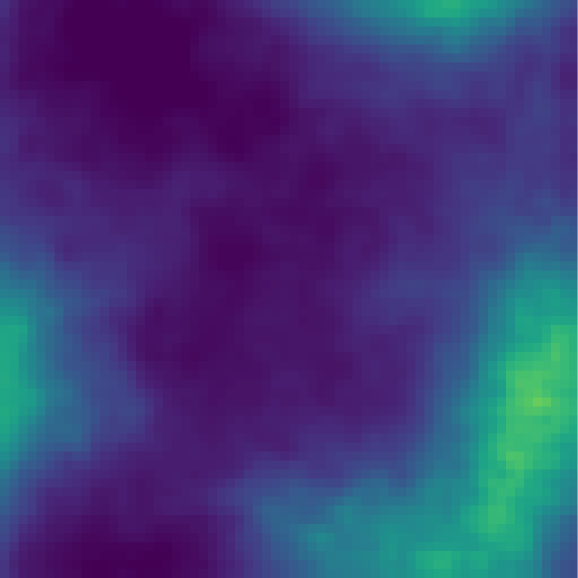}
        \caption*{$64\times64$}
    \end{minipage}
    \hfill
    \begin{minipage}[b]{0.29\textwidth}
        \captionsetup{skip=2pt}
        \centering
        \includegraphics[width=\textwidth, bmargin=0.5mm]{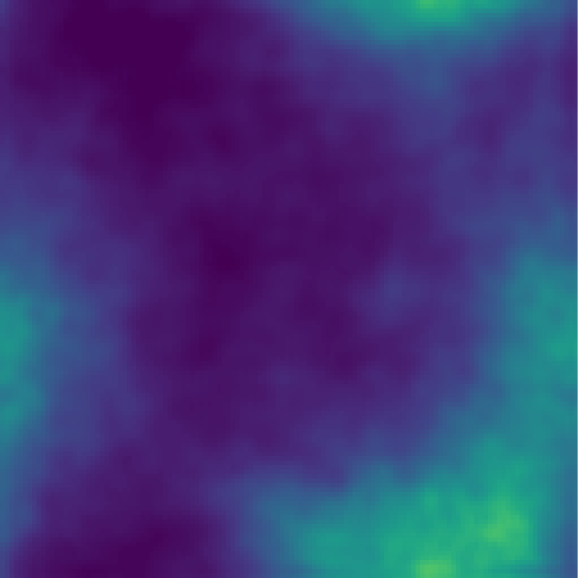}
        \caption*{$128\times128$}
    \end{minipage}
    \hfill
    \begin{minipage}[b]{0.29\textwidth}
        \captionsetup{skip=2pt}
        \centering
        \includegraphics[width=\textwidth, bmargin=0.5mm]{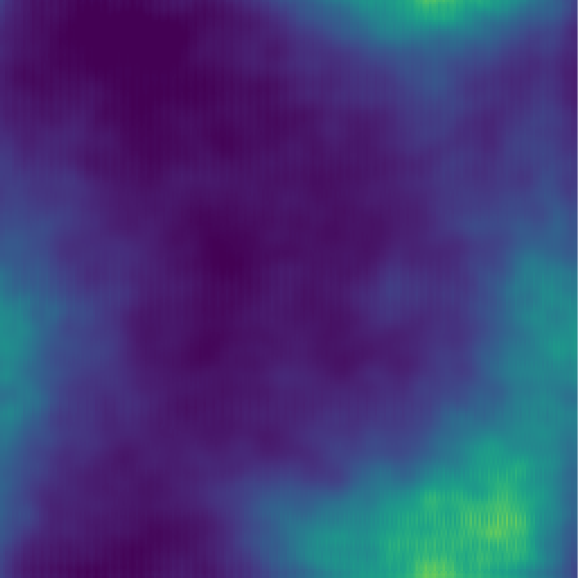}
        \caption*{$256\times256$}
    \end{minipage}

    \caption*{(a) Mean}
\end{subfigure}
\hfill
\begin{subfigure}[b]{0.05\textwidth}
    \captionsetup{skip=3pt}
    \begin{minipage}[b]{\textwidth}
    \captionsetup{skip=2pt}
    \includegraphics[width=\textwidth]{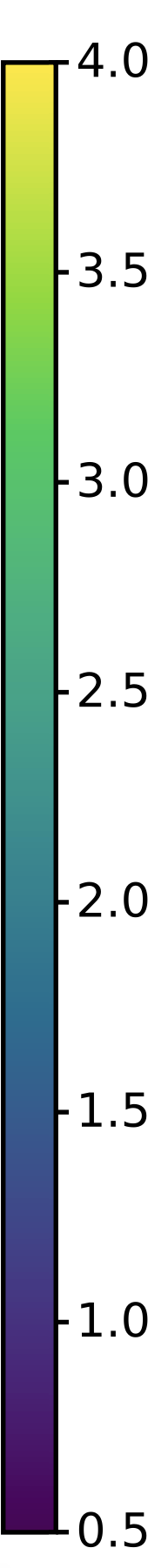}
    \caption*{\textcolor{white}{$\times$}}
    \end{minipage}
    \caption*{\textcolor{white}{a}}
\end{subfigure}
\hfill
\begin{subfigure}[b]{0.43\textwidth}
    \captionsetup{skip=3pt}
    \centering
    \begin{minipage}[b]{0.07\textwidth}
        \centering
        \includegraphics[width=\textwidth, bmargin=1mm]{figs/background.png}
    \end{minipage}
    \begin{minipage}[b]{0.29\textwidth}
        \centering
        \includegraphics[width=\textwidth, bmargin=1mm]{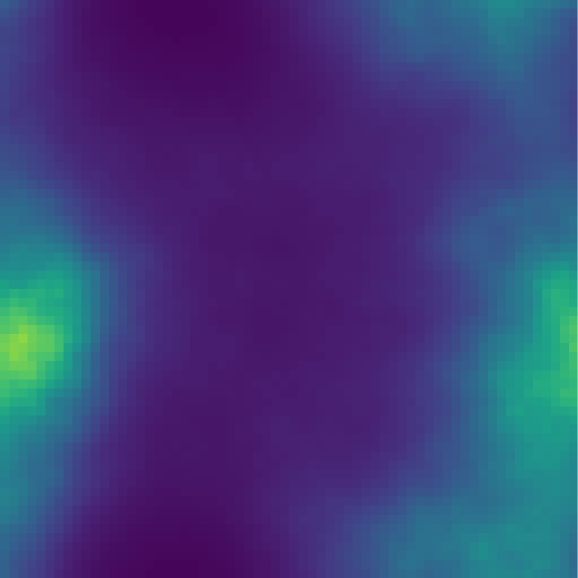}
    \end{minipage}
    \hfill
    \begin{minipage}[b]{0.29\textwidth}
        \centering
        \includegraphics[width=\textwidth, bmargin=1mm]{figs/darcyflow/darcy_mcmc_mean_128x128.pdf}
    \end{minipage}
    \hfill
    \begin{minipage}[b]{0.29\textwidth}
        \centering
        \includegraphics[width=\textwidth, bmargin=1mm]{figs/darcyflow/darcy_mcmc_mean_256x256.pdf}
    \end{minipage}

    \begin{minipage}[b]{0.07\textwidth}
        \centering
        \includegraphics[width=\textwidth, bmargin=1mm]{figs/background.png}
    \end{minipage}
    \begin{minipage}[b]{0.29\textwidth}
        \centering
        \includegraphics[width=\textwidth, bmargin=1mm]{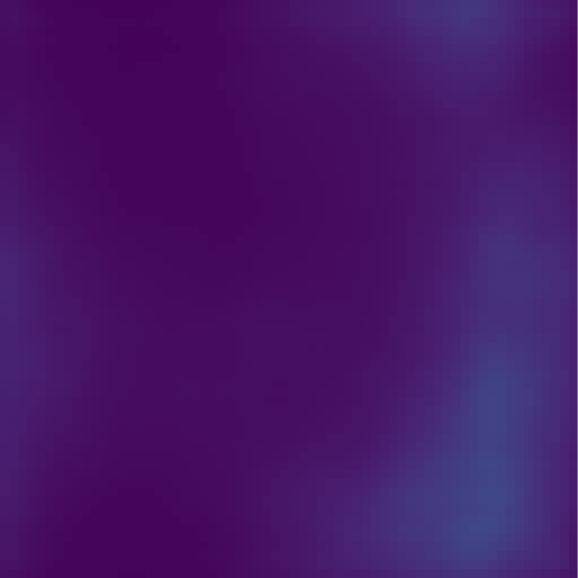}
    \end{minipage}
    \hfill
    \begin{minipage}[b]{0.29\textwidth}
        \centering
        \includegraphics[width=\textwidth, bmargin=1mm]{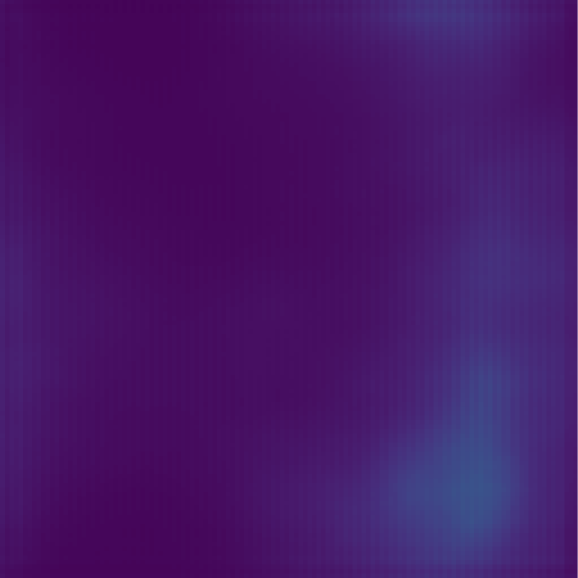}
    \end{minipage}
    \hfill
    \begin{minipage}[b]{0.29\textwidth}
        \centering
        \includegraphics[width=\textwidth, bmargin=1mm]{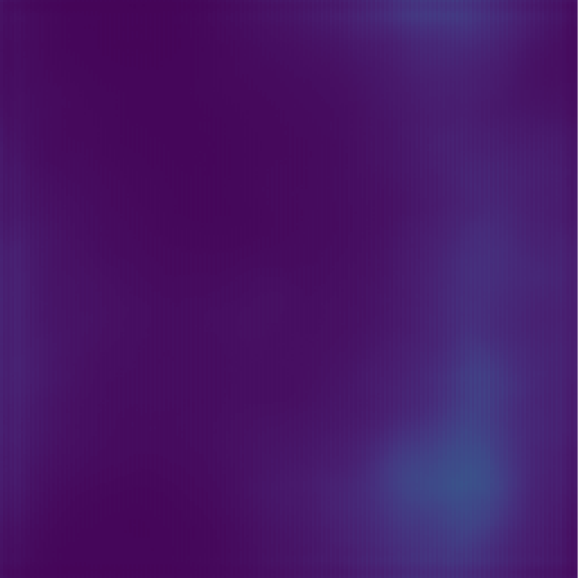}
    \end{minipage}

    \begin{minipage}[b]{0.07\textwidth}
        \centering
        \includegraphics[width=\textwidth, bmargin=1mm]{figs/background.png}
    \end{minipage}
    \begin{minipage}[b]{0.29\textwidth}
        \centering
        \includegraphics[width=\textwidth, bmargin=1mm]{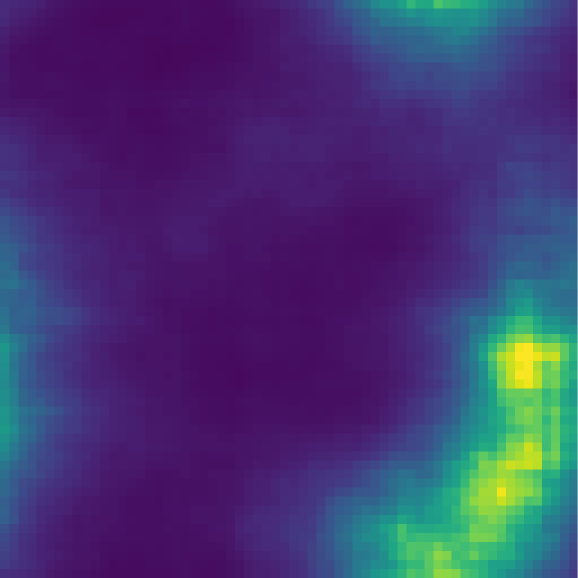}
    \end{minipage}
    \hfill
    \begin{minipage}[b]{0.29\textwidth}
        \centering
        \includegraphics[width=\textwidth, bmargin=1mm]{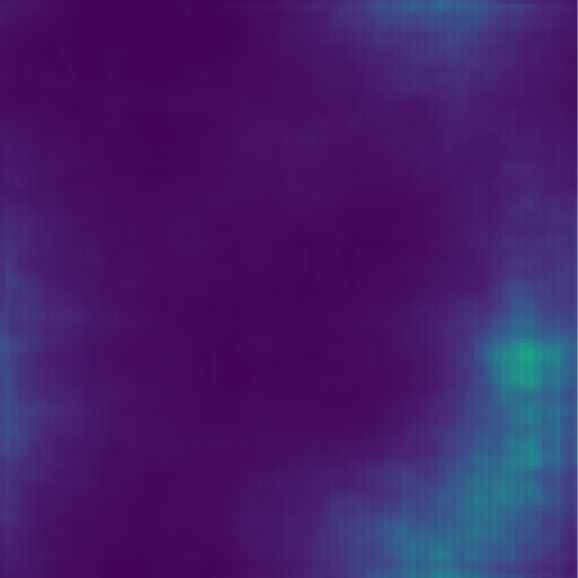}
    \end{minipage}
    \hfill
    \begin{minipage}[b]{0.29\textwidth}
        \centering
        \includegraphics[width=\textwidth, bmargin=1mm]{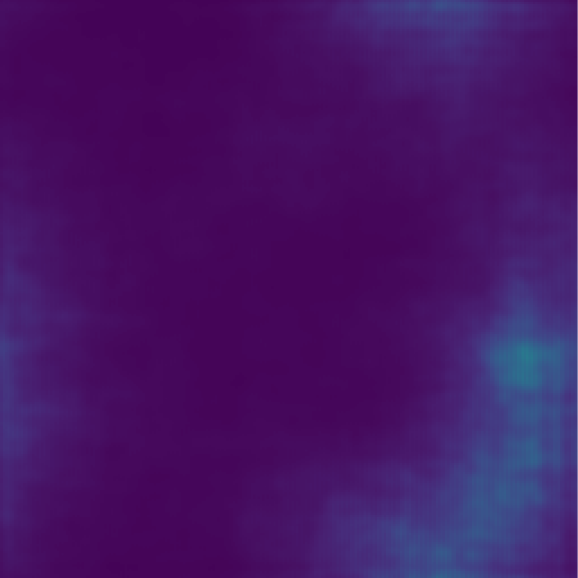}
    \end{minipage}

    \begin{minipage}[b]{0.07\textwidth}
        \captionsetup{skip=2pt}
        \centering
        \includegraphics[width=\textwidth, bmargin=0.5mm]{figs/background.png}
        \caption*{\textcolor{white}{()}}
    \end{minipage}
    \begin{minipage}[b]{0.29\textwidth}
        \captionsetup{skip=2pt}
        \centering
        \includegraphics[width=\textwidth, bmargin=0.5mm]{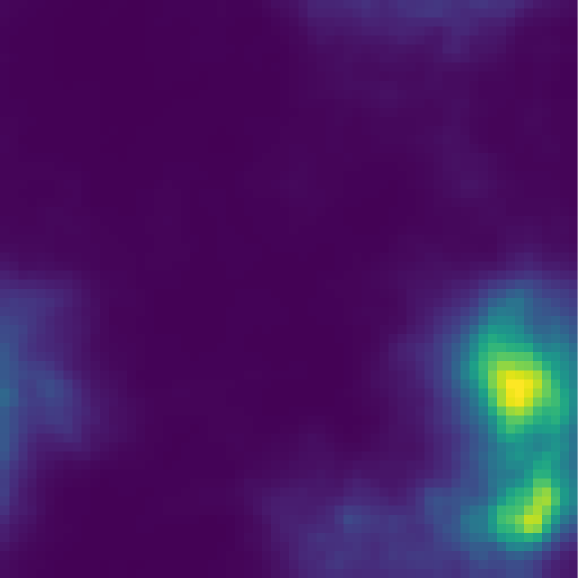}
        \caption*{$64\times64$}
    \end{minipage}
    \hfill
    \begin{minipage}[b]{0.29\textwidth}
        \captionsetup{skip=2pt}
        \centering
        \includegraphics[width=\textwidth, bmargin=0.5mm]{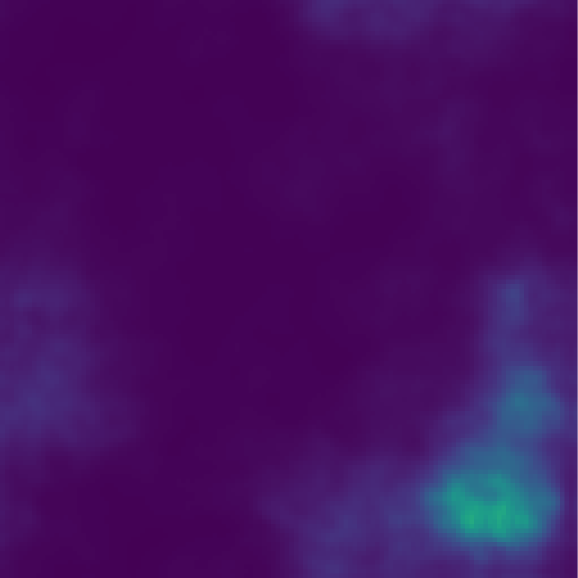}
        \caption*{$128\times128$}
    \end{minipage}
    \hfill
    \begin{minipage}[b]{0.29\textwidth}
        \captionsetup{skip=2pt}
        \centering
        \includegraphics[width=\textwidth, bmargin=0.5mm]{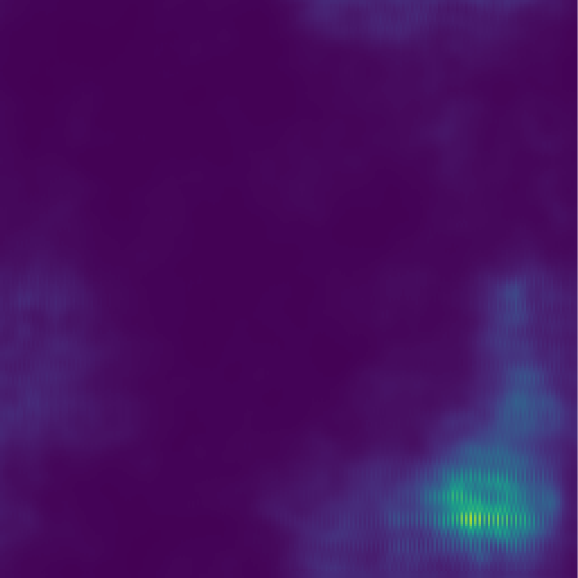}
        \caption*{$256\times256$}
    \end{minipage}

    \caption*{(b) Variance}
\end{subfigure}
\hfill
\begin{subfigure}[b]{0.05\textwidth}
    \captionsetup{skip=3pt}
    \begin{minipage}[b]{\textwidth}
    \captionsetup{skip=2pt}
    \includegraphics[width=\textwidth]{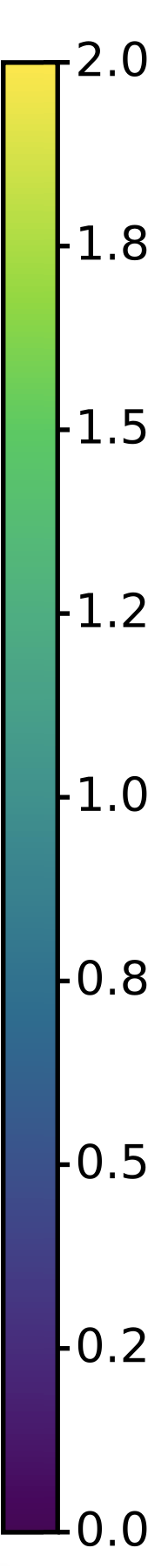}
    \caption*{\textcolor{white}{$\times$}}
    \end{minipage}
    \caption*{\textcolor{white}{a}}
\end{subfigure}
\caption{\textbf{Posterior Sample Statistics} (\secref{subsec:darcyflow}): The sample (a) mean and (b) variance of posterior samples of the MCMC as well as the learned GANO, MultilevelDiff, and DDO models at various resolutions. 10,000 samples are used. 
}
\label{fig:darcyflow-mean-var}
\end{figure}

\begin{figure}[htb!]
\captionsetup{skip=8pt}
\centering
\begin{subfigure}[b]{0.92\textwidth}
    \centering
    \begin{minipage}[b]{0.029\textwidth}
        \centering
        \includegraphics[width=\textwidth, trim={0 2.5cm 0 3cm}, clip=true]{figs/darcyflow/mcmc.png}
    \end{minipage}
    \begin{minipage}[b]{0.95\textwidth}
        \centering
        \includegraphics[width=\textwidth]{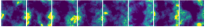}
    \end{minipage}

    \begin{minipage}[b]{0.029\textwidth}
        \centering
        \includegraphics[width=\textwidth, trim={0 3.2cm 0 2.3cm}, clip=true]{figs/darcyflow/gano.png}
    \end{minipage}
    \begin{minipage}[b]{0.95\textwidth}
        \centering
        \includegraphics[width=\textwidth]{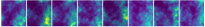}
    \end{minipage}

    \begin{minipage}[b]{0.029\textwidth}
        \centering
        \includegraphics[width=0.9\textwidth, trim={0 2.5cm 0 1.6cm}, clip=true, bmargin=0.0mm]{figs/darcyflow/multileveldiff.png}
    \end{minipage}
    \begin{minipage}[b]{0.95\textwidth}
        \centering
        \includegraphics[width=\textwidth]{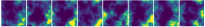}
    \end{minipage}

    \begin{minipage}[b]{0.029\textwidth}
        \centering
        \includegraphics[width=\textwidth, trim={0 3.0cm 0 3.0cm}, clip=true]{figs/darcyflow/ddo.png}
    \end{minipage}
    \begin{minipage}[b]{0.95\textwidth}
        \centering
        \includegraphics[width=\textwidth]{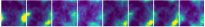}
    \end{minipage}
\end{subfigure}
\begin{subfigure}[b]{0.063\textwidth}
    \includegraphics[width=\textwidth]{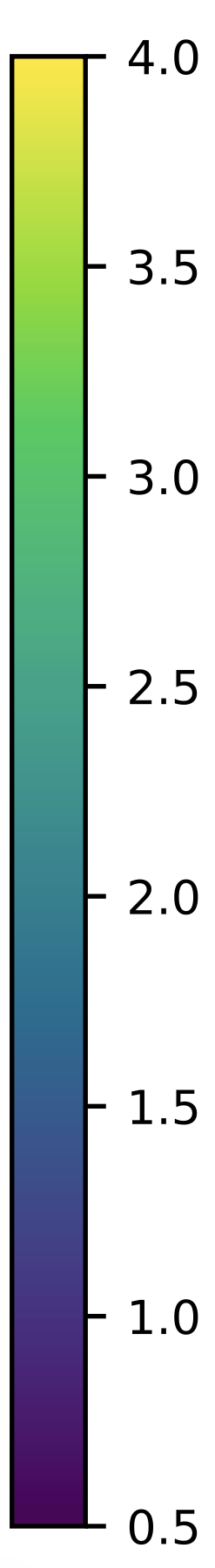}
\end{subfigure}
\caption{\textbf{Posterior Samples at 64$\times$64 resolution} (\secref{subsec:darcyflow}): The samples of the MCMC as well as the learned GANO, MultilevelDiff, and DDO models at the training resolution (64$\times$64).
}
\label{fig:darcyflow-samples-64x64}
\end{figure}


\begin{updaterequired}[black]
\subsection{Additional Results}
Finally, we analyze the trade-offs inherent in different neural operator approaches, including spectral and continuous convolution methods, and demonstrate their practical implications through comparative experiments. This is because, unlike finite-dimensional models, function space models require parametric designs that not only enhance expressivity but also satisfy discretization invariance. Consequently, in function-valued deep generative models, the design of the neural operator is a critical factor in determining overall model performance. In this context, in \appref{app:alias-free}, we provide a comprehensive analysis of various neural operator designs introduced in function-valued generative models, examining their respective strengths and weaknesses, such as aliasing issues in spectral convolution and overfitting tendencies in continuous convolution methods. Readers interested in understanding the impact of these design choices on model performance are encouraged to explore this section. 
\end{updaterequired}

\section{Discussion and Conclusions}

We propose DDOs, the first theoretical framework and numerical demonstration of resolution invariant diffusion generative models on function space. Our approach generalizes denoising score matching for trace-class noise corruptions that live in the Hilbert space of the data, and considers a discrete-time diffusion model for sampling using infinite-dimensional Langevin dynamics. Future work will connect this framework with noise scales that depend continuously on time (as in  \appref{sec:timedeptnoise}) to the forward and backward SDEs in~\cite{song2020score}. Defining the backward SDE will require satisfying conditions that guarantee time reversibility of infinite-dimensional diffusions; see~\cite{follmer1986time} for examples of these conditions. Adapting the covariance of the reference noise process based on the data distribution may also be helpful for generative modeling with other functional datasets, and to extend this framework to solve inverse problems with infinite-dimensional parameters~\cite{stuart2010inverse}. Lastly, rigorous error analysis (e.g. using an approximate score) 
will be important to understand the class of data distributions that can be accurately characterized with infinite-dimensional diffusion models.

\newpage

\vskip 0.2in
\bibliography{references}

\newpage
\appendix

\section{Notation}
\label{sec:notation}

We denote by \(\sR\) the real numbers and by \(\sR^n\) their \(n\)-fold Cartesian product and write \(\|\cdot\|_2\) for Euclidean norm. We write \(\sN\) for the set of natural numbers.
We denote by \(H\) a real, separable, Hilbert space and by \(\langle \cdot, \cdot \rangle\), \(\|\cdot\|\) its inner-product and norm respectively.
We write \(\mathcal{B}(H)\) for the Borel sets of \(H\) generated by the open sets induced from the norm topology. For probability measures 
\(\mu, \nu\) on \( \big ( H, \mathcal{B}(H) \big )\), we say \(\mu\) is absolutely continuous with respect to \(\nu\) and denote it \(\mu \ll \nu\)
if, for any \(B \in \mathcal{B}(H)\), \(\nu (B) = 0\) implies \(\mu(B) = 0\). If \(\mu \ll \nu\) and \(\nu \ll \mu\) hold then we \(\mu\) and \(\nu\)
are equivalent and denote it \(\mu \sim \nu\). If neither \(\mu \ll \nu\) or \(\nu \ll \mu\) hold then we say \(\mu\) and \(\nu\) are mutually 
singular and denote it \(\mu \perp \nu\). We say \(u\) is a random variable distributed according to \(\mu\) and denote it \(u \sim \mu\)
if the law of \(u\) is \(\mu\). Given two random variables \(u, v\), we write \(u \perp v\) if they are independent. For any mapping,
\(f : H \to \sR\), we denote by \(\E_{u \sim \mu} [f(u)]\) the expected value of \(f\) under \(\mu\).

For any bounded operator \(C : H \to H\), we say \(C\) is self-adjoint if \(\langle Cu,  v \rangle = \langle u, Cv \rangle\) for all \(u,v \in H\).
We say an operator is positive if \(\langle C u, u \rangle > 0 \) for all \(u \in H \setminus \{0\}\) (equivalently non-negative if \(\langle C u, u \rangle \geq 0 \)). We say \(C\) is trace-class, or nuclear,
if for any orthonormal basis \(\{\phi_j\}_{j=1}^\infty\) of \(H\), we have \(\text{Tr}(C) = \sum_{j=1}^\infty \langle C \phi_j, \phi_j \rangle < \infty\).
For any self-adjoint, non-negative, trace-class operator, we denote by \(C^{1/2}\) the unique operator such that \(C = C^{1/2} C^{1/2}\). 
We denote by \(H^*\) the topological (continuous) dual \(H\) which is itself a separable Hilbert space consisting of all bounded linear functionals \(l : H \to \sR\)
with an inner-product induced by the Reisz map. Since it follows by the Riesz representation theorem that for any \(l \in H^*\), there exists a unique element 
\(v \in H\) such that \(l(u) = \langle u, v \rangle\) for any \(u \in H\), we define the Riesz map \(R : H^* \to H\) by \(l \mapsto v\).

\begin{table*}[!ht]
\small
\centering
\caption{\textbf{Score-based diffusion models:} from finite to infinite dimension.}
\captionsetup{justification=centering}
 \hspace{-1.1em}
 \begin{tabular}{l l l}
 \toprule
 \textbf{Setting} & \textbf{Finite dimension} & \textbf{Infinite dimension} \\ 
 \midrule
 Data space 
 & Euclidean spaces
 & Function spaces 
 \\
 Base measures 
 & Lebesgue measure
 & Gaussian random fields 
 \\
 Noise in diffusion 
 & Multivariate random variables
 & Gaussian random fields 
 \\
 Score 
 & Score function 
 & Score operator  
 \\
 Process 
 & Langevin process in finite dimensions 
 $y$ & Langevin process in function spaces 
 \\
 Learning loss 
 & Euclidean norm 
 & Norm on function spaces \\ & & (discretization invariant)  
 \\
 Controls 
 & Variance 
 & Length scale, variance, energy, etc.  
 \\
 Base model 
 & Neural networks 
 & Neural operators  
 \\
  \bottomrule
 \end{tabular}
 \label{table:Setting}
 \end{table*}

\section{Proofs of Theorem}

\subsection{Convolution of measures}
\label{sub:conv_measure}

The following results holds more generally for Radon Gaussian measures on locally convex spaces. We show them here in the Hilbert space setting to avoid introducing extra notation but refer the reader to \cite{bogachev2015gaussian} for a thorough overview of the more general setting.

Let \((H, \langle \cdot, \cdot \rangle)\) be a real, separable, Hilbert space and \(\mu\) and \(\nu\) be two probability measures on the Borel \(\sigma\)-algebra \(\mathcal{B}(H)\). Then the product measure \(\mu \otimes \nu \) is defined on \(\mathcal{B}(H) \otimes \mathcal{B}(H) = \mathcal{B}(H \times H)\). Define the mapping \(T : H \times H \to H\) by
\(T(u,v) = u + v\). Then the pushforward of \(\mu \otimes \nu \) under \(T\) is called the convolution of \(\mu\) and \(\nu\) and is denoted \(\mu * \nu\). In particular, given two independent random variables \(u \sim \mu\) and \(v \sim \nu\), the random variable \(u + v\) is distributed according to \(\mu * \nu\). It can be shown that, for any \(B \in \mathcal{B}(H)\), we have
\begin{equation}
    \label{eq:convolution_of_measures}
    (\mu * \nu)(B) = \int_H \mu(B-v) \: d\nu(v) = \int_H \nu(B - u) \: d\mu(u) = (\nu * \mu)(B),
\end{equation}
for example, see Appendix A.3 in \cite{bogachev2015gaussian} and references therein. The following result shows that if \(\nu\) is a centered Gaussian and \(\mu\) charges its Cameron-Martin space, the convolution is equivalent, in the sense of measures, to \(\nu\).

\begin{theorem}
    \label{thm:equivalent_convolution}
    Let \(\mu, \nu\) be two probability measures on \((H, \mathcal{B}(H))\) with \(\nu = N(0,C)\) for some \(C: H \to H\) self-adjoint, positive, and trace-class. If \(\mu (C^{1/2}(H)) = 1\), then \(\nu_u \sim \nu\) where $\nu_u$ is the conditional for $v|u$ and \(\nu * \mu \sim \nu\).
\end{theorem}
\begin{proof}
    For any \(B \in \mathcal{B}(H)\), we have by \eqref{eq:convolution_of_measures},
    \[(\nu * \mu)(B) = \int_H \nu(B-u) \: d\mu(u).\]
    Therefore \((\nu * \mu)(B) = 0\) if and only if \(\nu (B-u) = 0\) for \(\mu\)-almost any \(u \in H\) since \(\nu\) is non-negative. For any \(u \in H\), define the measures
    \[\nu_u (B) = \nu(B-u), \qquad \forall B \in \mathcal{B}(H)\]
    which are Gaussian \(\nu_u = N(u,C)\). By the Cameron-Martin Theorem, given as Proposition 2.26 in \cite{da1992stochastic}, \(\nu_u \ll \nu\) for any \(u \in C^{1/2}(H)\).
    Let \(B \in \mathcal{B}(H)\) be such that \(\nu(B) = 0\). Since, \(\nu_u \ll \nu\) for any \(u \in C^{1/2}(H)\), 
    we have that \(\nu(B-u) = 0\). Since \(\mu(C^{1/2}(H)) = 1\), \(\nu(B-u) = 0\) for \(\mu\)-almost any \(u \in H\) and therefore \(\nu * \mu \ll \nu\).

    Now let \(B \in \mathcal{B}(H)\) be such that \((\nu * \mu)(B) = 0\) then \(\nu_u (B) = 0\) for \(\mu\)-almost any \(u \in H\). Since \(\mu(C^{1/2}(H)) = 1\), again by the Cameron-Martin Theorem, \(\nu_u \ll \nu\). But, by Theorem 2.25 in \cite{da1992stochastic},  Gaussians are either equivalent or mutually singular , therefore \(\nu \ll \nu_u\) and thus \(\nu(B) = 0\) hence the result follows. 
\end{proof}

Let \(A : H \to H\) be a linear operator. If \(u \sim \mu\), then from definition, the random variable \(Au\) is distributed according to the measure \(\mu \circ A^{-1}\) where \(A^{-1}\) denotes the pre-image of \(A\). In particular, for any \(B \in \mathcal{B}(H)\),
\begin{equation}
    \label{eq:linear_operator_on_rv}
    (\mu \circ A^{-1})(B) = \mu(\{u \in H : Au \in B\}).
\end{equation}
The following corollary of Theorem~\ref{thm:equivalent_convolution} addresses random variables of the form \(Au + v\) where \(u \sim \mu\) and \(v \sim \nu\) are independent.

\begin{corollary}
    \label{cor:equivalent_convolution}
    Let \(\mu, \nu\) be two probability measures on \((H, \mathcal{B}(H))\) with \(\nu = N(0,C)\) for some \(C: H \to H\) self-adjoint, positive, and trace-class. Let \(A : H \to H\) be a linear operator such that \(A(H) \subseteq C^{1/2}(H)\) then \(\nu * (\mu \circ A^{-1}) \sim \nu\).
\end{corollary}
\begin{proof}
    From equation \eqref{eq:linear_operator_on_rv} and the assumption that \(A(H) \subseteq C^{1/2}(H)\), we have 
    \begin{align*}
        (\mu \circ A^{-1})\big ( C^{1/2}(H) \big)
        &= \mu \big ( \{u \in H : Au \in C^{1/2}(H)\} \big ) \\
        &= \mu(H) \\
        &= 1.
    \end{align*}
Thus the result follows by Theorem~\ref{thm:equivalent_convolution}.
\end{proof}

\subsection{Conditional scores}
\label{sub:conditioning}

Let \(H\) be a real, separable, Hilbert space and denote by \(\mathcal{B}(H)\) its Borel \(\sigma\)-algebra. Let \(\gamma\) be a probability measure on \(\big (H \times H, \mathcal{B}(H) \otimes \mathcal{B}(H) \big )\). We introduce the coordinates \((u,v) \sim \gamma\). Denote by \(\mu\) marginal of \(u\), by \(\nu\) the marginal of \(v\), and by \(\gamma^u\) the conditional \(v|u\) for \(\mu\)-almost any \(u \in H\). Let \(\mu_0\) be a probability measure on \(\big ( H, \mathcal{B}(H) \big )\) and suppose that \(\nu \ll \mu_0\) and \(\gamma^u \ll \mu_0\) for \(\mu\)-almost any \(u \in H\).

\begin{lemma}
    \label{lemma:marginalizing_density}
    The Radon–Nikodym derivatives of \(\nu\) and \(\gamma^u\) with respect to \(\mu_0\) satisfy
    \[\frac{d \nu}{d \mu_0}(v) = \E_{u \sim \mu} \frac{d \gamma^u}{ d \mu_0} (v), \qquad \text{for } \mu_0\text{-almost any } v \in H.\]
\end{lemma}
\begin{proof}
    Let \(B \in \mathcal{B}(H)\) then by definition of a conditional measure and Fubini's Theorem,
    \begin{align*}
        \nu (B) &= \int_H \int_B d \gamma^u (v) d \mu (u) \\
        &= \int_H \int_B \frac{d \gamma^u}{d \mu_0}(v) d \mu_0 (v) d \mu(u) \\
        &= \int _B \left ( \int_H \frac{d \gamma^u}{d \mu_0}(v) d \mu(u) \right ) d \mu_0(v).
    \end{align*}
    We also have,
    \[\nu (B) = \int_B \frac{d \nu}{d \mu_0}(v) d \mu_0(v).\]
    Therefore
    \[\int_B \frac{d \nu}{d \mu_0}(v) d \mu_0(v) = \int _B \left ( \int_H \frac{d \gamma^u}{d \mu_0}(v) d \mu(u) \right ) d \mu_0(v).\]
    Since \(B\) is arbitrary, we must have that 
    \[\frac{d \nu}{d \mu_0}(v) = \int_H \frac{d \gamma^u}{d \mu_0}(v) d \mu(u)\]
    for \(\mu_0\)-almost any \(v \in H\) which is the desired result.
\end{proof}

Let \(E \subseteq H\) be a Hilbert space continuously embedded in \(H\) and denote by \(D_E\) the Frechet differential operator  on \(H\) in the direction of \(E\). Suppose that \(\nu \sim \mu_0\) and \(\gamma^u \sim \mu_0\) so that all respective Radon–Nikodym derivatives exist and are positive. Define,
\[\Phi(v) \coloneqq \log \frac{d \nu}{d \mu_0}(v), \quad \Psi(v;u) \coloneqq \log \frac{d \gamma^u}{d \mu_0} (v)\]
for \(\mu_0\)-almost ant \(v \in H\) and \(\mu\)-almost any \(u \in H\). Suppose that \(\Phi\) and \(\Psi(\cdot; u)\) are once \(D_E\)-continuously differentiable . Furthermore assume
\[\E_{v \sim \nu} \|D_E\Phi(v)\|_{E^*}^2 < \infty, \quad \E_{u \sim \mu} \E_{v \sim \gamma^u} \|D_E \Psi(v;u)\|_{E^*}^2 < \infty\]
where \(E^*\) denotes the topological dual of \(E\). Let \(G_\theta : H \to E^*\) be a parametric mapping with parameters \(\theta \in \mathbb{R}^p\). Assume that, for all \(\theta \in \mathbb{R}^p\),
\[\E_{v \sim \nu} \|G_\theta (v) \|_{E^*}^2 < \infty.\]
Define the functionals,
\begin{align*}
    F(\theta) &\coloneqq \E_{v \sim \nu}  \|D_E \Phi(v) - G_\theta (v) \|^2_{E^*} \\
    J(\theta) &\coloneqq \E_{u \sim \mu} \E_{v \sim \gamma^u} \|D_E \Psi(v;u) - G_\theta(v) \|_{E^*}^2.
\end{align*}

\begin{theorem}
    \label{thm:conditioning_trick}
    There exists a constant \(C < \infty\) independent of \(\theta \in \mathbb{R}^p\) such that
    \[F(\theta) = J(\theta) + C, \qquad \forall \: \theta \in \mathbb{R}^p.\]
\end{theorem}
\begin{proof}
    We have
    \[F(\theta) = \E_{v \sim \nu} \left [ \|G_\theta (v)\|_{E^*}^2 - 2 \langle D_E \Phi (v), G_\theta (v) \rangle_{E^*} \right ] + C_1 \]
    where \(C_1 = \E_{v \sim \nu} \|D_E \Phi (v)\|_{E^*}^2 < \infty \) by assumption.
    Similarly,
    \[J(\theta) = \E_{u \sim \mu} \E_{v \sim \gamma^u} \left [ \|G_\theta (v)\|_{E^*}^2 - 2 \langle D_E \Psi(v;u), G_\theta (v) \rangle_{E^*} \right ] + C_2 \]
    where \(C_2 = \E_{u \sim \mu} \E_{v \sim \gamma^u} \|D_E \Psi (v;u)\|_{E^*}^2 < \infty \) by assumption.
    By definition of a conditional measure,
    \[\E_{v \sim \nu} \|G_\theta (v)\|_{E^*}^2 = \E_{u \sim \mu} \E_{v \sim \gamma^u} \|G_\theta (v)\|_{E^*}^2\]
    for any \(\theta \in \mathbb{R}^p\). Using Lemma~\ref{lemma:marginalizing_density}, the Leibniz integral rule, and Fubini's Theorem, we find
    \begin{align*}
        \E_{v \sim \nu} \langle D_E \Phi(v), G_\theta (v) \rangle_{E^*} &= \E_{v \sim \nu} \langle D_E \log \frac{d \nu}{d \mu_0} (v), G_\theta (v) \rangle_{E^*} \\
        &= \E_{v \sim \nu} \langle \frac{d \mu_0}{d \nu} (v) D_E \frac{d \nu}{d \mu_0} (v), G_\theta (v) \rangle_{E^*} \\
        &= \E_{v \sim \mu_0} \langle D_E \frac{d \nu}{d \mu_0} (v), G_\theta (v) \rangle_{E^*} \\
        &= \E_{v \sim \mu_0} \langle D_E \E_{u \sim \mu} \frac{d \gamma^u}{d \mu_0} (v), G_\theta (v) \rangle_{E^*} \\
        &= \E_{v \sim \mu_0} \langle \E_{u \sim \mu} D_E \frac{d \gamma^u}{d \mu_0} (v), G_\theta (v) \rangle_{E^*} \\
        &= \E_{v \sim \mu_0} \langle \E_{u \sim \mu} \frac{d \gamma^u}{d \mu_0} (v) D_E \log \frac{d \gamma^u}{d \mu_0} (v), G_\theta (v) \rangle_{E^*} \\
        &= \E_{v \sim \mu_0} \E_{u \sim \mu} \frac{d \gamma^u}{d \mu_0} (v)  \langle  D_E \log \frac{d \gamma^u}{d \mu_0} (v), G_\theta (v) \rangle_{E^*} \\
        &= \E_{u \sim \mu} \E_{v \sim \mu_0}  \frac{d \gamma^u}{d \mu_0} (v)  \langle  D_E \Psi (v;u), G_\theta (v) \rangle_{E^*} \\
        &=  \E_{u \sim \mu} \E_{v \sim \gamma^u}  \langle  D_E \Psi (v;u), G_\theta (v) \rangle_{E^*}.
    \end{align*}
    Setting \(C = C_1 - C_2\) completes the proof.
\end{proof}

It remains to show that approximating the score of the convolved measure ensues we are close to the true measures. The following lemma relates the Wasserstein distance of the two.

\begin{lemma}
\label{lemma:wasserstein_approx}
Let $\eta \sim \mu_\sigma$ be a noise random variable with finite $p$-moment, and let $v = u + \eta \sim \mu * \mu_\sigma$. Then, the Wasserstein-p distance for $p \geq 1$ satisfies $W_p(\mu * \mu_\sigma, \mu) \leq \| \eta\|_{L^p(\mu_\sigma)}$. 
\end{lemma}
\begin{proof}
Let $(u,\eta)$ follow the product coupling $(\mu \otimes \mu_\sigma)(du,d\eta)$. Then, let the coupling $(v,\eta)$ where $v = u + \eta$ be drawn according to $(u + \eta, u)_\sharp (\mu \otimes \mu_\sigma)(du,d\eta)$. Choosing this coupling $(v,\eta)$ to upper bound the Wasserstein-p distance, we have
$$W_p(\mu * \mu_\sigma, \mu)^p = \inf_{\gamma \in \Pi(\mu * \mu_\sigma,\mu)} \int |v - u|^p d\gamma(v,u) \leq \int |(u + \eta) - u|^p (\mu \otimes \mu_\sigma)(du,d\eta).$$
Given that the integrand is independent of $u$, we have
$W_p(\mu * \mu_\sigma, \mu)^p \leq \int |\eta|^p d\mu_\sigma = \|\eta \|_{L^p(\mu_\sigma)}^p$, and the result follows.
\end{proof}

\begin{remark} For a Gaussian measure $\mu_\sigma$, the $p$-th moment is finite by the Fernique Theorem. Moreover, by Theorem 6.6 in~\cite{stuart2010inverse} there is a constant $C_p > 0$ so that $\|\eta\|_{L^p} \leq C_p(\text{Tr}(C_\sigma))$. This can be used to establish the convergence rate of $W_p(\mu * \mu_\sigma, \mu) \rightarrow 0$ as $\sigma \rightarrow 0$.
\end{remark}

\begin{lemma}
\label{lemma:gaussian_derivative}
The Fr\'{e}chet derivative \(D_{H_{\mu_0}} \Psi(w;u)\) as defined in~\eqref{eq:gaussian_score} is in \(H_{\mu_0}^*\). 
\end{lemma}
\begin{proof}
Notice that since \(u \in H_{\mu_0}\), \(\mu\)-almost surely, we can find \(g \in H\) such that \(u = C^{1/2}g\). For any \(w \in H_{\mu_0}\), we can similarly find \(f \in H\) such that \(w = C^{1/2}f\). We can write both \(g\) and \(f\) in the othronormal basis \(\{\varphi_j\}_{j=1}^\infty\),
\[g = \sum_{j=1}^\infty \langle g, \varphi_j \rangle \varphi_j, \qquad f = \sum_{j=1}^\infty \langle f, \varphi_j \rangle \varphi_j \]
with both series converging in \(H\). Orthonormality implies
\[\langle u, \varphi_j \rangle = \lambda_j^{1/2} \langle g, \varphi_j \rangle, \qquad \langle w, \varphi_j \rangle = \lambda_j^{1/2} \langle f, \varphi_j \rangle\]
for any \(j \in \mathbb{N}\). Therefore equation \eqref{eq:gaussian_score_action} becomes
\begin{align*}
    D_{H_{\mu_0}} \Psi(v;u)w &= \sum_{j=1}^\infty  \langle g, \varphi_j \rangle \langle f, \varphi_j \rangle  \\
    &\leq \left (\sum_{j=1}^\infty  |\langle g, \varphi_j \rangle|^2 \right )^{1/2} \left (\sum_{j=1}^\infty  |\langle f, \varphi_j \rangle|^2 \right )^{1/2} \\
    &= \|g\| \|f\|
\end{align*}
which is finite hence the result follows.
\end{proof}

\section{Multiple Noise Scales}
\label{sec:timedeptnoise}

To satisfy our absolute continuity condition for the perturbed data measure in Sections~\ref{subsec:smoothing_operators} and~\ref{subsec:multiple_noise_scales}, we need $A_t(H) \subseteq \text{Im}(C_t^{1/2})$ for all $t$. The following lemmas show that this condition holds for different time-dependent scalar weightings of the forward and noise covariance operators.

\begin{lemma}  
\label{lemma:multiple_scales}
Let $\eta_t \sim \mu_t = \mathcal{N}(0,C_t)$ where $C_t = g(t)C$ and $v_t = A_tu + \eta_t$ where $A_t = f(t)u$ for all $t \in I$ for $u \in H$. Assuming the mappings \(f,g : \I \to \sR\) satisfy
\(0 < M_1 \leq f(t), g(t) \leq M_2\) for all \(t \in \I\) and
\(\mu \big( C^{1/2}(H) \big ) = 1\), then 
\(A_t(H) \subseteq C_t^{1/2}(H)\) for all \(t \in \I\)
\end{lemma}
\begin{proof} 
Let $L_t = g(t)/f(t) C$. We will first show that $L_t^{1/2}(H) = C^{1/2}(H)$. By Lemma 6.15 in~\cite{stuart2010inverse}, the image of the two positive-definite, and self-adjoint linear operators on a Hilbert space $H$ are equal if and only if there exists constants $K_1,K_2 > 0$ such that $K_1\langle u, Cu \rangle \leq \langle u, L_t u \rangle \leq K_2\langle u, Cu \rangle$ for all $u \in H$.

Under the conditions on $f,g$, for any $u \in H$, we have
\begin{align*}
\langle u, C u \rangle = \langle u, \frac{f(t)g(t)}{f(t)g(t)} C u \rangle \leq \frac{M_2}{M_1} \langle u, \frac{g(t)}{f(t)} C u \rangle 
&= K_1 \langle u, L_t u \rangle \\
&= K_1 \langle u, \frac{g(t)}{f(t)} C u \rangle \leq K_1 \frac{M_2}{M_1} \langle u, Cu \rangle = K_2 \langle u, C u \rangle,
\end{align*}
where $K_1 = M_2/M_1$ and $K_2 = K_1^2$. Then, for \(\mu \big( C^{1/2}(H) \big ) = 1\) we have $u \in C^{1/2}(H)$ for $u \in H$. From the image equivalence, we have $u \in L_t^{1/2}(H)$ and so $A_tu = f(t)u \in C_t^{1/2}(H)$. 
\end{proof}


\begin{lemma} Let $C_t = f(t)AA^*$ where $A\colon H \rightarrow H$ is a linear operator and $f \colon [0,T] \rightarrow \mathbb{R}$ is a function satisfiying $c = \sup_{t \in [0,T]} 1/f(t) < +\infty$. Then, for $K = AA^*$ we have that $\text{Im}(K^{1/2}) \subseteq \text{Im}(C_t^{1/2})$ for all $t \in [0,T]$.
\end{lemma}
\begin{proof} The image of $A$ is equivalent to image of the $K^{1/2}$ where $K = AA^*$ is a positive-definite, and self-adjoint operator. For any $u \in H$, we have
$$\langle u, K u \rangle = \langle u, AA^* f(t)/f(t) u \rangle \leq \sup_{t \in [0,T]} \frac{1}{f(t)} \langle u, AA^* f(t) u \rangle = c \langle u, C_t u \rangle.$$
The result on the image spaces follows by Lemma 6.15 in~\cite{stuart2010inverse}.
\end{proof}

\begin{example} The function $f(t) = e^{\gamma t}$ for $\gamma > 0$ satisfies the condition in the lemma above with $c = 1$. This choice motivates the following study. 
\end{example}

Alternatively, we can define the forward process for data corruption with multiple noise scales using a stochastic differential equation (SDE), as in~\cite{song2020score}. Let us consider the linear SDE $du_t = -Lu_tdt + dW_t$ for $u_t \in H$ where $W_t$ is a $Q$-Wiener process and $L\colon H \rightarrow H$ is a linear and positive-definite operator where its eigenvectors form an orthonormal basis for $H$. The solution of this SDE for any $t > 0$ is given by
$$u(t) = e^{-Lt}u(0) + \int_0^t e^{L(s-t)}dW_s.$$
Letting $u(0) = u$ and $A_t \coloneqq e^{-Lt}$, we can treat $u \mapsto A_t u$ as the forward model and the second term $\eta_t \coloneqq \int_0^t e^{L(s-t)}dW_s$ as the additive noise process, which is drawn independently of $u$. The following abridged theorem from~\cite{da1992stochastic} describes the statistical properties of the noise process.

\begin{theorem} Assuming $\int_0^{T} \text{Tr}[A_rQA_r^*]dr < \infty$, then (i) $\eta_t$ is Gaussian, (ii) has continuous paths, and (iii) its covariance is given by
$$C_t \coloneqq \text{Cov}(\eta_t) = \int_0^t A_rQA^*_rdr, \quad t \in [0,T].$$
\end{theorem}

To satisfy the absolute continuity conditions on the perturbed data measure for $u(t)$ as before, we need to show that for each $t$, $A_t(H) \subseteq \text{Im}(C_t^{1/2})$. As shown in Corollary B.7 of~\cite{da1992stochastic}, the image of $C_t^{1/2}$ for a covariance of the form above is equivalent to the image of the linear operator $B_t \colon H \rightarrow H$ defined as
$$B_t u \coloneqq \int_0^t L_{t-s} Q u ds.$$
$B_t$ and $A_t$ are both linear and self-adjoint operators, so the condition $\text{Im}(A_t) \subseteq \text{Im}(B_t)$ holds if and only if there exists a constant $K > 0$ so that $\langle u, A_t u \rangle \leq K \langle u, B_t u \rangle$ for all $u \in H$. Using the decomposition of $u = \sum_{j=1}^{\infty} \langle u, \psi_j \rangle \psi_j$ where $\{\psi_j\}_{j=1}^\infty$ are eigenvectors and $\{\lambda_j\}_{j=1}^\infty$ are eigenvalues of $L$, we have
$$A_t u = \sum_{j}e^{-Lt}\langle u,\psi_j\rangle \psi_j = \sum_{j}e^{-\lambda_j t}\langle u,\psi_j\rangle \psi_j.$$
Choosing the noise covariance to be $Q = L^\gamma$ for some scalar $\gamma$ such that draws remain in $H$, we have \begin{align*}
B_t u &= \int_0^t \sum_j e^{-L(t-s)}\langle u,\psi_j\rangle Q\psi_j  ds\\
&= \int_0^t \sum_{j} e^{-\lambda_j (t-s)} \langle u,\psi_j \rangle \lambda_j^\gamma  \psi_j ds \\
&= \sum_{j} \langle u,\psi_j \rangle \lambda_j^\gamma  \psi_j \int_0^t e^{-\lambda_j (t-s)} ds \\
&= \sum_{j} \langle u,\psi_j \rangle \lambda_j^{\gamma-1}(1 - e^{-\lambda_j t})  \psi_j .
\end{align*}
We can now compare the images of the operators. For $u \in H$ we have
$$\langle u, A_t u \rangle = \sum_j e^{-2\lambda_jt}|\langle u, \psi_j \rangle|^2, \qquad \langle u, B_t u \rangle = \sum_j \lambda_j^{2(\gamma-1)}(1 - e^{-2\lambda_j t})^2|\langle u, \psi_j \rangle|^2.$$
For each $\lambda_j > 0$, there exists a time $t_j$ such that $e^{-2\lambda_j t} \leq \lambda_j^{2(\gamma-1)}(1 - e^{-2\lambda_jt})^2$ for all $t > t_j$. For these times, we satisfy the condition required for our theory. Generalizing these results is an important direction for future work.

\begin{updaterequired}[black]
    
\section{Crank–Nicolson Discretization}
\label{sec:disc_langevin}

In \twosecrefs{subsec:lagevin_dynamics}{subsec:multiple_noise_scales}, we introduced our method based on the Euler-Maruyama discretization. Here, we will show the Crank–Nicolson discretization sampling method and relate it to existing methods in the literature. 
In particular, we will work in the setting of multiple noise scale as introduced in \secref{subsec:multiple_noise_scales}. 
Let \(F : H \times \I \to H\) be defined as \(F(u,t) = -u + R D_{H_{\mu_t}} \Phi(u,t) \). 
For a fixed \(t \in \I\), we apply the Crank–Nicolson method to the linear part of the drift in \eqref{eq:langevin} to obtain
\begin{equation}
    \label{eq:cn_scheme_first}
    (2 + h_t) u_{n+1} = (2 - h_t)u_n + 2 h_t G(u_n, t) + \sqrt{8 h_t} \eta^{(t)}_n
\end{equation}
 where we define \(G : H \times \I \to H\)
by \(G(u,t) = R D_{H_{\mu_t}} \Phi(u,t) \) and \(h_t > 0\). For any \(h_t \in (0,2)\), \eqref{eq:cn_scheme_first}
can be written as
\begin{equation}
    \label{eq:cn_scheme}
    u_{n+1} =  \alpha_t u_n + (1 - \alpha_t) G(u_n, t) + \beta_t \eta^{(t)}_n
\end{equation}
with the transformation \(\beta_t^2 = 8 h_t / (2+h_t)^2\) where \(\beta_t \in (0,1)\) and we define \(\alpha_t = \sqrt{1 - \beta_t^2}\).
\eqref{eq:cn_scheme} is a type of Metropolis-adjusted Langevin proposal in the function space setting and is related 
to the celebrated pre-conditioned Crank–Nicolson MCMC method \cite{cotter2013mcmc}. We remark that \eqref{eq:cn_scheme} resembles
the exact, single-step, Gaussian approximation sampling method of~\cite{ho2020denoising}. We leave the design and analysis of 
algorithms based on this approach for future work.
\end{updaterequired}

\section{Denoising Diffusion Probabilistic Models}
\label{sec:ddpm}

In \secref{subsec:multiple_noise_scales} we showed that for a particular choice of data and noise 
scaling, we may recover the forward process of the DDPM framework proposed in~\citet{ho2020denoising}. Let us recall the noise process in DDPM: for some sequence \(0 < \beta_1 \leq \dots \leq \beta_T < 1\), let \(\alpha_t = \prod_{s=1}^t (1-\beta_s)\). Then, for \(u_0 \sim \mu\) we define 
\begin{equation} \label{eq:forwardprocess_DDPM}
    u_t = \sqrt{\alpha_t} u_0 + \sqrt{1 - \alpha_t} \eta, \qquad \eta \sim \mu_0 =  N(0,C)
\end{equation}
for \(t \geq 1\) with \(\eta \perp u_0\). Under the assumption that \(\mu \big( C^{1/2}(H) \big ) = 1\), we have that the measure \(\nu_t\) 
defined as the law of \(u_t\), is equivalent to the measure \(\mu_t = N(0, (1-\alpha_t)C)\)
for any \(t\). 
Furthermore the law of the conditional
\(u_t | u_0\) is equivalent to \(\mu_t\), \(\mu\)-almost surely. We may therefore apply the theory presented in \secref{sec:methodology}
to obtain a sequence of tractable score-matching problems. Once solved, we obtain a sequence of approximate scores which can
be used within an annealed Langevin algorithm similar to Algorithm~\ref{alg:sampling_alg} to obtain samples.
This procedure, however, is not equivalent to the sampling procedure in \cite{ho2020denoising} which compares 
the backwards conditionals \(u_{t-1}|u_t,u_0\) to Gaussian parameterizations and therefore an exact backwards 
sampling method is derived. We show now how a similar scheme may be derived in infinite dimensions.

We may compute directly that the law of \(u_{t-1}|u_t,u_0\), denoted by \(\pi_{t-1}\), is the Gaussian \(N(u_{t-1}; m_t, c_tC)\) where 
\[m_t = \frac{\sqrt{\alpha_{t-1}} \beta_t}{1-\alpha_t} u_0 + \frac{\sqrt{1-\beta_t}(1-\alpha_{t-1})}{1 - \alpha_t} u_t, \qquad c_t = \frac{(1-\alpha_{t-1})\beta_t}{1 - \alpha_t}. \]
We can therefore consider the parametric Gaussian measure \(\rho_{t-1} (u_{t-1}; u_t,t) = N(u_{t-1}; G_\theta(u_t,t), c_tC)\) for some \(G_\theta\colon H \times \I \to H\).
If we are to compare \(\pi_{t-1}\) and \(\rho_{t-1}\) using the Kullback-Liebler (KL) divergence, as done in \cite{ho2020denoising},
we need that \(\pi_{t-1}\) and \(\rho_{t-1}\) are equivalent otherwise their KL divergence is infinite. Since \(\pi_{t-1}\) and \(\rho_{t-1}\)
have the same covariance, by the Feldman-H\'{a}jek theorem we need only that \(m_t - G_\theta (u_t,t) \in C^{1/2}(H)\) for the measures to be equivalent. Using the forward process 
in~\eqref{eq:forwardprocess_DDPM}, we can also write
\[m_t = \frac{\sqrt{\alpha_{t-1}} \beta_t + \sqrt{(1-\beta_t)\alpha_t} (1 - \alpha_{t-1}) }{1-\alpha_t} u_0 + \frac{\sqrt{1 - \beta_t} (1-\alpha_{t-1})}{\sqrt{1 - \alpha_t}} \eta,\]
and even with the assumption \(\mu \big( C^{1/2}(H) \big ) = 1\), we have that \(m_t \notin C^{1/2}(H)\), \(\mu \otimes \mu_0\)-almost surely because \(\eta \notin C^{1/2}(H)\), \(\mu_0\)-almost surely; see \appref{sec:noise_regularity}.
It is therefore not enough to constrain the range of \(G_\theta\) to \(C^{1/2}(H)\); we need instead that, for every realization of the data and noise, \(G_\theta\) yields from \(u_t\) 
precisely a direction so that \(m_t - G_\theta (u_t,t) \in C^{1/2}(H)\). We may accomplish this with the following re-parameterization,
\[G_\theta (u_t,t) = \frac{\sqrt{1-\beta_t} (1-\alpha_{t-1})}{1 - \alpha_t} u_t + F_\theta(u_t,t)\]
for some \(F_\theta : H \times \I \to C^{1/2}(H)\). Then
\[m_t - G_\theta (u_t,t) = \frac{\sqrt{\alpha_{t-1}} \beta_t}{1-\alpha_t} u_0 - F_{\theta}(u_t,t)\]
which is an element of \(C^{1/2}(H)\), \(\mu\)-almost surely. It now follows that \(\pi_{t-1}\) and \(\rho_{t-1}\)
are equivalent measures and we may therefore compute their KL divergence, in particular,
\begin{align*}
    D_{\text{KL}}(\pi_{t-1},\rho_{t-1}) &= \|C_t^{-1/2} \big( m_t - G_\theta(u_t,t) \big ) \|^2 \\
    &= \frac{1 - \alpha_t}{(1-\alpha_{t-1})\beta_t} \left \| C^{-1/2} \left ( \frac{\sqrt{\alpha_{t-1}} \beta_t}{1-\alpha_t} u_0 - F_{\theta}(u_t,t) \right ) \right \|^2 .
\end{align*}
Moreover, we may optimize the following joint objective that minimizes the KL divergence at all times $t > 1$,
\[\sum_{t > 1} \E_{u \sim \mu} \E_{\eta \sim \mu_0}  \frac{1 - \alpha_t}{(1-\alpha_{t-1})\beta_t} \left \| C^{-1/2} \left ( \frac{\sqrt{\alpha_{t-1}} \beta_t}{1-\alpha_t} u - F_{\theta} \big ( \sqrt{\alpha_t}u + \sqrt{1-\alpha_t} \eta, t \big ) \right ) \right \|^2 \]
We note  that \(F_\theta\) cannot be further re-parameterized so that it learns the noise from the signal, even in the case \(C\) is positive so
that \(C^{-1/2}\) can be dropped from the objective, because re-parameterizing it this way will violate its range condition which is crucial for obtaining measure equivalence. Once \(F_\theta\) is learned,
sampling from the approximate backwards conditional \(u_{t-1}|u_t\) amounts to evaluating
\[u_{t-1} = \frac{\sqrt{1-\beta_t} (1-\alpha_{t-1})}{1 - \alpha_t} u_t + F_\theta(u_t, t) + \sqrt{c_t} \eta, \qquad \eta \sim \mu_0.\]

An alternative approach is to not require measure equivalence and instead compare \(\pi_{t-1}\) and \(\rho_{t-1}\) in a metric which is 
finite when comparing singular measures. One example of such a metric is the Wasserstein-\(p\) distance. Since we do not require equivalence,
we may parameterize \(\rho_{t-1}\) more generally by allowing a different scaling for the covariance or even learning a different covariance operator. We will
not pursue this here for the sake of simplicity in exposition and will consider \(\pi_{t-1}\) and \(\rho_{t-1}\) to have the same covariance. In this case, the Wasserstein-\(2\) distance~\cite{gelbrich1990onaformula} is 
\[W_2^2 (\pi_{t-1}, \rho_{t-1}) = \|m_t - G_\theta(u_t,t)\|^2.\]
Similarly to before, we may use the forward process in~\eqref{eq:forwardprocess_DDPM} and write
\[m_t = \frac{1}{\sqrt{1 - \beta_t}} \left ( u_t - \frac{\beta_t}{\sqrt{1-\alpha_t}} \eta \right ).\]
Therefore, by re-parameterizing, 
\[G_\theta(u_t,t) = \frac{1}{\sqrt{1 - \beta_t}} \left ( u_t - \frac{\beta_t}{\sqrt{1-\alpha_t}} F_\theta (u_t,t) \right )\]
for some \(F_\theta : H \times \I \to H\) yields 
\[m_t - G_\theta(u_t,t) = \frac{\beta_t}{\sqrt{(1-\beta_t)(1-\alpha_t)}} \bigg ( F_\theta(u_t,t) - \eta \bigg ).\]
Moreover, we may optimize the following joint objective for all times $t > 1$,
\begin{equation}\label{eq:objective_DDPM}
\sum_{t > 1} \E_{u \sim \mu} \E_{\eta \sim \mu_0} \frac{\beta_t^2}{(1-\beta_t)(1-\alpha_t)} \|F_\theta \big( \sqrt{\alpha_t}u + \sqrt{1-\alpha_t} \eta , t \big ) - \eta \|^2.
\end{equation}
Once \(F_\theta\) is learned,
sampling from the approximate backwards conditional \(u_{t-1}|u_t\) amounts to evaluating
\begin{equation} \label{eq:backwardprocess_DDPM}
u_{t-1} = \frac{1}{\sqrt{1 - \beta_t}} \left ( u_t - \frac{\beta_t}{\sqrt{1-\alpha_t}} F_\theta (u_t, t) \right ) + \sqrt{c_t} \eta, \qquad \eta \sim \mu_0.
\end{equation}
The above derivation precisely yields the framework in \cite{ho2020denoising} with the only difference that \(\mu_0\) does not have an identity covariance.
Note that we did not even require the assumption \(\mu \big( C^{1/2}(H) \big ) = 1\) because we allowed ourselves to compare mutually singular measures.
It is unclear whether such a formulation is preferred in practical applications or if it may eventually yields instabilities in the algorithm or lack of convergence.
Furthermore, we note that both derivations worked directly with conditionals \(u_{t-1}|u_t,u_0\) instead of the true backwards 
conditionals \(u_{t-1}|u_t\) and did not establish equivalence between their optimization objectives. Empirical comparisons of  these two objectives for diffusion modeling is an important direction to explore in future work.

\begin{updaterequired}[black]

\section{Examples of \(\mu(H_{\mu_0}) = 1\)}
\label{sec:further_examples}

In this section, we give examples of data distributions $\mu$ that satisfies the condition \(\mu(H_{\mu_0}) = 1\), where $H_{\mu_0}$ is a Cameron-Martin space of a perturbation noise. These will help us gaining intuition on the condition.

\subsection{Gaussian}
\label{subsec:gaussian_example}

When we can expect the assumption \(\mu(H_{\mu_0}) = 1\) to hold, one of the simplest examples is when the data measure is Gaussian. 
Consider the space
\[H = \dot{L}^2(\mathbb{T}^d;\mathbb{R}) \coloneqq \left \{ u \in L^2(\mathbb{T}^d;\mathbb{R}) : \int_{\mathbb{T}^d} u \: dx = 0 \right \}, \]
 where \(\mathbb{T}^d\) is the \(d\)-dimensional unit torus. We denote by \(\dot{H}^s (\mathbb{T}^d;\mathbb{R})\) for any \(s > 0\) as the corresponding
 periodic, mean-zero Sobolev spaces \citep{adams2003sobolev}. Let \(\mu = N(0, C_1)\) where
 \begin{equation}
    \label{eq:example_C1}
     C_1 = \sigma_1^2 (-\Delta + \tau_1^2 I)^{-\alpha_1}.
 \end{equation}
Here \(-\Delta\) is the negative Laplacian with periodic boundary conditions, \(I\) is the identity operator, and \(\sigma_1, \tau_1, \alpha_1\) are positive scalars.
Covariances of the type \eqref{eq:example_C1} are said to be of the Mat\'{e}rn-type 
because Gaussian processes defined by Mat\'{e}rn kernels are the only stationary solutions
to certain SPDEs with differential operator \(C_1^{-1}\) \citep{whittle1954onstationary,lindgren2011anexplicit}. We make extensive use of 
such covariances throughout the rest of this work as the Gaussian measures defined by them 
are amenable to analysis and efficient sampling.
When \(\alpha_1 > d/2\), Lemma 6.27 in \cite{stuart2010inverse}
implies that \(\mu \big( \dot{H}^{s}(\mathbb{T}^d;\mathbb{R}) \big) = 1 \) for any \(s \in [0, \alpha_1 - d/2)\). We will assume that 
\(\alpha_1 > d\). Let \(\mu_0 = N(0,C_2)\) where
\begin{equation}
    \label{eq:example_C2}
    C_2 = \sigma_2^2 (-\Delta + \tau_2^2 I)^{-\alpha_2}
\end{equation}
 with \(\alpha_2 > d/2\) so that \(C_2\) is trace-class. It is easy to compute that \(C_2^{1/2} (H) = \dot{H}^{\alpha_2}(\mathbb{T}^d;\mathbb{R})\).
Therefore, the assumption \(\mu (H_{\mu_0}) = 1\) is satisfied for any \(\alpha_2 \in (d/2, \alpha_1 - d/2)\).

The above analysis reveals that there is a gap of size \(d/2\) between the regularity of the data and the noise.
In particular, \(\mu \big( \dot{H}^s(\mathbb{T}^d;\sR) \big ) = 1\) for \(s \in [0, \alpha_1 - d/2)\) while \(\mu_0 \big( \dot{H}^m(\mathbb{T}^d;\sR) \big ) = 1\)
for \(m \in [0, \alpha_1 - d)\).
Therefore, in order to consider perturbations with Gaussians of the form \eqref{eq:rough_perturbation}, the noise 
must be at least \(d/2\) \say{less smooth} than the data, in a Sobolev sense. Furthermore since we want to consider noise with a trace-class covariance 
so that it is amenable to approximation, we have a fundamental limit on the regularity of the data. That is, the data must 
live in \(\dot{H}^s (\mathbb{T}^d;\sR)\) for some \(s > d/2\). This assumption can be satisfied, for example, when the data measure is defined  
as the pushforward of some PDE solution operator; we show explicit examples below (in \appref{sec:further_examples}). 
\end{updaterequired}

\subsection{Gaussian Mixture}
\label{subsec:further_gaussian_mixture}

Let \(D \subset \mathbb{R}^d\) be a bounded, open set with Lipschitz boundary and consider \(H = L^2(D;\mathbb{R})\).
We will consider the covariances \eqref{eq:example_C1} and \eqref{eq:example_C2} where \(-\Delta\) is instead the 
negative Laplacian with zero Dirichlet boundary conditions on \(D\). Suppose \(\alpha_1 - d/2 > s\) for some \(s \geq d\) and let \(f_1, f_2 \in H^1_0(D;\mathbb{R}) \cap H^s(D;\mathbb{R})\). Define \(\mu\) so that, if \(u \sim \mu\), then
\[
\begin{cases}
u \sim N(f_1, C_1) & \text{w.p. } p, \\
u \sim N(f_2, C_1) & \text{w.p. } 1-p,
\end{cases}
\]
for some \(0 \leq p \leq 1\). By Lemma 6.27 in \cite{stuart2010inverse}, \(u \in H^1_0(D;\mathbb{R}) \cap H^s(D;\mathbb{R})\) \(\mu\)-almost surely. 
Therefore \(\mu_0 = N(0,C_2)\) with \(\alpha_2 \in (1/2, s]\) implies \(\mu(H_{\mu_0}) = 1\).

\subsection{Pushforwards}
\subsubsection{Navier-Stokes}
\label{subsec:further_navier_stokes}

 Consider the vorticity form of the two-dimensional Navier-Stokes equations on the unit torus,
 \begin{align}
    \label{eq:ns_eq}
     \begin{split}
         \partial_t u + \nabla^\perp \phi \cdot u - \epsilon \Delta u  &= f, \qquad \text{in } \mathbb{T}^2 \times (0,\infty), \\
         -\Delta \phi &= u, \qquad \text{in } \mathbb{T}^2 \times (0,\infty),
     \end{split}
 \end{align}
 with initial condition \(u(\cdot, 0) = u_0\) for some \(u_0, f \in \dot{L}^2(\mathbb{T}^2;\mathbb{R})\) and \(\epsilon > 0\). It is shown in \cite{temam1988infinite} 
 that for any \(\epsilon > 0\), \eqref{eq:ns_eq} has a unique weak solution such that \(u(\cdot, t) \in \dot{H}^s(\mathbb{T}^2;\mathbb{R})\) for any \(s > 0\) and \(t > 0\).
 We may thus define the flow map \(Q : \dot{L}^2(\mathbb{T}^2;\mathbb{R}) \times (0,\infty) \to \dot{H}^s(\mathbb{T}^2;\mathbb{R})\)
 for any \(s > 0\) by \((u_0, t) \mapsto u(\cdot, t)\). Let \(\rho = N(0,C_1)\) with \(C_1\) given by \eqref{eq:example_C1} for any \(\alpha_1 > 1\). 
 Let \(\mu = Q(\cdot, T) _\sharp \rho\) for some \(T > 0\). Then, for \(u \sim \mu\), we have \(u \in \dot{H}^s(\mathbb{T}^2;\mathbb{R})\), for any \(s > 0\), \(\mu\)-almost surely.
 Therefore \(\mu_0 = N(0,C_2)\) with any \(\alpha_2 > 1\) implies \(\mu(H_{\mu_0}) = 1\) where \(C_2\) is given by \eqref{eq:example_C2}.

 \subsubsection{Burgers' Equation}
 Consider the one-dimensional Burgers' equation on the unit torus,
 \begin{align}
    \label{eq:burgers_eq}
     \begin{split}
         \partial_t u + \frac{1}{2} \partial_x (u^2) - \epsilon \partial^2_{xx} u &= f, \qquad \text{in } \mathbb{T} \times (0,\infty), \\
         u(\cdot, 0) &= u_0, \quad \:\: \text{in } \mathbb{T},
     \end{split}
 \end{align}
 for some \(u_0, f \in \dot{L}^2(\mathbb{T};\mathbb{R})\) and \(\epsilon > 0\). By Theorem 1.1 in \cite{kiselev2008blow}, the solution 
 \(u(\cdot, t)\) is real analytic for all times \(t > 0\), so we may define the flow map \(Q : \dot{L}^2(\mathbb{T};\mathbb{R}) \times (0,\infty) \to \dot{H}^s(\mathbb{T};\mathbb{R})\)
 for any \(s > 0\) by \((u_0, t) \mapsto u(\cdot, t)\). Let \(\rho = N(0,C_1)\) with \(C_1\) given by \eqref{eq:example_C1} for any \(\alpha_1 > 1/2\). 
 Let \(\mu = Q(\cdot, T) _\sharp \rho\) for some \(T > 0\). Then, for \(u \sim \mu\), we have \(u \in \dot{H}^s(\mathbb{T};\mathbb{R})\), for any \(s > 0\), \(\mu\)-almost surely.
 Therefore \(\mu_0 = N(0,C_2)\) with any \(\alpha_2 > 1/2\) implies \(\mu(H_{\mu_0}) = 1\) where \(C_2\) is given by \eqref{eq:example_C2}.

 \subsubsection{Darcy Flow}
Let \(D \subset \mathbb{R}^d\) be a bounded, open set with Lipschitz boundary and consider the steady-state of the Darcy flow equation,
 \begin{align}
    \label{eq:darcy_eq}
     \begin{split}
         \nabla \cdot (a \nabla u)  &= f, \qquad \text{in } D, \\
         u &= 0, \qquad \text{in } \partial D,
     \end{split}
 \end{align}
 for some \(a \in L^\infty (D;\mathbb{R}_+)\) and \(f \in L^2(D;\mathbb{R})\). It is shown in \cite{evans2010partial} that 
 \eqref{eq:darcy_eq} has a unique weak solution \(u \in H^1_0(D;\mathbb{R})\) and thus we can define the mapping
 \(Q: L^\infty(D;\mathbb{R}_+) \to H_0^1(D;\mathbb{R})\) by \(a \mapsto u\). Let \(\rho = N(0,C_1)\) with \(C_1\) given 
 by \eqref{eq:example_C1} where \(-\Delta\) is instead the negative Laplacian with zero Neumann boundary conditions on \(D\).
 Assume that \(\alpha_1 > d/2\).
For some \(0 < c_- < c_+ < \infty\), define \(T: \mathbb{R} \to \mathbb{R}_+\) by
\[T(x) = \begin{cases}
    c_-, & x < 0, \\
    c_+, & x \geq 0.
\end{cases}\]
We may view \(T : L^2(D;\mathbb{R}) \to L^\infty(D;\mathbb{R}_+)\) as a Nemistkii operator, that is, 
\[(Tf)(x) = T(f(x)), \qquad \forall \: f \in L^2(D;\mathbb{R}).\]
Let \(\mu = (Q \circ T)_\sharp \rho\) then, for \(u \sim \mu\), we have \(u \in H^1_0(D;\mathbb{R})\) \(\mu\)-almost surely.
Therefore \(\mu_0 = N(0,C_2)\) with \(\alpha_2 \in (1/2, 1]\) implies \(\mu(H_{\mu_0}) = 1\) where \(C_2\) is given by \eqref{eq:example_C2}
and \(-\Delta\) is instead the negative Laplacian with zero Dirichlet boundary conditions on \(D\). Notice that will this condition on
\(\alpha_2\), \(C_2\) is trace-class only when \(d=1\).

\section{The Karhunen-Lo\'{e}ve Expansion}
\label{sec:kl_expansion}

Let $C$ be a self-adjoint, positive, semi-definite operator in a Hilbert space $H$ with an orthonormal set of eigenvectors (functions) $\phi_j \in H$ and corresponding eigenvalues $\lambda_j$ in a decreasing order, i.e., $\lambda_1 \geq \lambda_2 \geq ...$. The Karhunen-Lo\'{e}ve (KL) expansion represents a Gaussian random variable $u \sim \mathcal{N}(m,C)$ with mean $m \in H$ and covariance operator $C$ as
$$u = m + \sum_{j=1}^{\infty} \sqrt{\lambda_j} \phi_j \xi_j,$$
where $\{\xi_j\}_{j=1}^\infty$ is an i.i.d.\thinspace sequence of $\mathcal{N}(0,1)$ random variables. This construction allows us to sample a Gaussian measure on a Hilbert space and, for certain domains and boundary conditions, can be implemented with fast FFT-based methods \cite{lord2014anintoroduction}. For more details on the KL expansion, we refer the reader to~\cite{da2006introduction, adler2010geometry}.

\section{Noise Regularity}
\label{sec:noise_regularity}
Consider the Gaussian \(\mu_0\) as defined in \secref{sec:methodology}. It follows by Lemma 6.10 in \cite{stuart2010inverse} that \(\mu_0 (H_{\mu_0}) = 0\).
In particular, any random variable \(u \sim \mu_0\) is not contained in \(H_{\mu_0}\) with probability one. This makes quantities of the form
\[\|C^{-1/2}(u - g)\| = \infty, \qquad \mu_0\text{-almost surely}\]
for any fixed \(g \in H\) since \(u - g \notin H_{\mu_0}\) with probability one.
To see this, consider the following formal calculation. Let
\[C \phi_j = \lambda_j \phi_j, \quad \|\phi_j\| = 1\]
by an eigendecomposition of \(C\). By the spectral theorem, \(\{\phi_j\}_{j=1}^\infty\) forms a complete orthonormal basis for \(H\) and for positive-definite $C$ we have \(\lambda_j > 0\). Suppose \(u \sim N(0,C)\). Then, by the Karhunen-Loeve expansion (see \appref{sec:kl_expansion}) we have 
\[u = \sum_{j=1}^\infty \sqrt{\lambda_j} \xi_j \phi_j, \quad \xi_j \sim N(0,1).\]
From $C^{-\frac{1}{2}} \phi_j = \frac{1}{\sqrt{\lambda_j}} \phi_j$, we have that 
$C^{-\frac{1}{2}} u = \sum_{j=1}^\infty \xi_j \phi_j$,
and thus
\[\|C^{-\frac{1}{2}}u \|^2 = \sum_{j=1}^\infty |\xi_j|^2 = \infty, \qquad \text{ a.s. }\]
The same result holds if we subtract some element \(g \in H\) from \(u\). Writing 
$g = \sum_{j=1}^\infty g_j \phi_j$, where $g_j = \langle g, \phi_j \rangle,$ then 
\[\|C^{-\frac{1}{2}} (u-g)\|^2 = \sum_{j=1}^\infty \big | \xi_j - \frac{g_j}{\sqrt{\lambda_j}} \big  |^2 = \infty, \qquad \text{ a.s. }\]
And the situation is not improved even if \(g \in H_{\mu_0}\) so that \(g_j/\sqrt{\lambda_j} \to 0\) and the series \(\sum_{j=1}^\infty (g_j/\sqrt{\lambda_j}) \phi_j\) converges in \(H\).
This is a fundamental difficulty of the infinite-dimensional setting. We refer the reader to Section 3.5 in \cite{stuart2010inverse} for a further discussion.

We perform a simple numerical experiment to demonstrate this. We fix the data measure \(\mu = N(0,C_1)\) where \(C_1\) has form \eqref{eq:example_C1} with \(\alpha_1 = 3\), \(\sigma_1 = 4\), and \(\tau_1 = 1\).
We set \(\mu_0 = N(0,C_2)\) where \(C_2\) has the form \eqref{eq:example_C2} with \(\alpha_2 = 2\), \(\sigma_2 = 0.2\), and \(\tau_2 = 1\). We fix a FNO architecture wich retains 32 modes and has a width of 64 and re-train it at different 
resolutions of the data and noise so as to minimize either
\[\E_{u \sim \mu} \E_{\eta \sim \mu_0} \|\eta - G_\theta(u + \eta)\|^2\]
or
\[\E_{u \sim \mu} \E_{\eta \sim \mu_0} \|C^{-1/2} \big ( \eta - G_\theta(u + \eta) \big )\|^2.\]
We then compare the test errors, simply defined as the same quantity as the training loss but evaluated on new draws from the data and noise distributions.
The results are shown in Figure~\ref{fig:kerrigan_loss}.  We see that the blue curve stays constant, confirming that the FNO can learn to represent 
noise from the Sobolev space \(\dot{H}^{3/2}(\mathbb{T};\sR)\) in a discretization invariant way. On the other hand, when training with the pre-conditioner \(C^{-1/2}\), we see 
the test error grow as we increase the resolution. This demonstrates the effect of the infinity in the loss function.
\begin{figure}[!ht]
     \centering
     \includegraphics[width=0.5\textwidth]{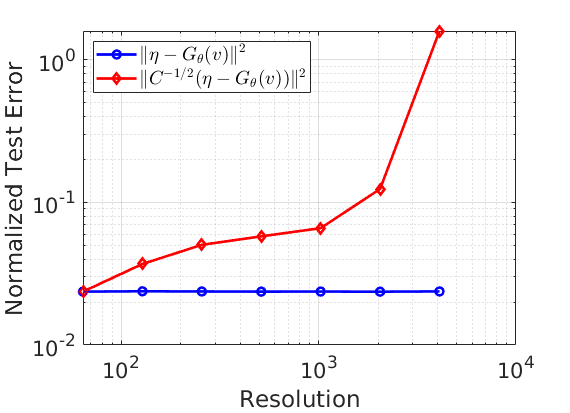}
     \caption{\textbf{Noise Regularity}~~Test error when training with two different loss functions across different resolutions. Red curve is re-scaled so that it matches the error of the blue curve at the lowest resolution for the sake of visualization.}
     \label{fig:kerrigan_loss}
\end{figure}

\section{\updated{Smoothing Operators}}
\label{sec:num_smoothing}


To illustrate the effect of the choice of the Camerion-Martin space (CM-space) of a (forward) noise process, we train an FNO architecture with varying training conditions. Specifically, we simulate that samples from data distribution, \(u \sim \mu\), which may or may not be in the noise's CM-space.

Similar to \secref{subsec:further_gaussian_mixture}, we consider a mixture of two Gaussians as data distribution \(\mu\) where \(d=1\), \(D = (0,2\pi)\), \(f_1 = -10/6 x + 5\), \(f_2 = -f_1\), and \(p=0.5\). For its covariance \(C_1\), we choose \(\alpha_1 = 1.5\), \(\sigma_1 = 3\), and \(\tau_1 = 3\).

We compare four training conditions by varying the noise covariance \(C_2\); (a) noise process uses white noise, (b) all data samples lie on the CM-space of the noise covariance,  (c) there exist samples \(u \notin C_2^{1/2}(H)\), and (d) apply a smoothing operator \(A\), while the same covariance as in (c). For (b), we choose  \(\alpha_1 = 1\), \(\sigma_1 = 1.73\), and \(\tau_1 = 3\), and for (c) and (d), we use \(\alpha_1 = 2\), \(\sigma_1 = 10\), and \(\tau_1 = 3\). For smoothing operator \(A\) in (d), we use a Gaussian blur such that $A(H) \subseteq C_2^{1/2}(H)$.

We train models at a resolution of \(512\) for \(5,000\) iterations. Unlike \secref{subsec:further_gaussian_mixture}, we only trained in the resolution \(512\) and sample with varying resolutions (See Figure~\ref{fig:smoothing}).

From (a), we can observe that the trained model successfully generates the samples in resolution 512, the same resolution during training. However, when the trained model tries to generate higher-resolution samples, its samples collapse into modes. While the parametric score operator is discretization-invariant, due to independent Gaussian noise, the induced distribution from the model is not in function-valued space.

On the contrary, when \(C_2^{1/2}(H)\) is sufficiently large enough to include all samples from the data distribution, the proposed method learns the data distribution. Moreover, the model successfully generates samples in higher dimensions, as the model distribution is a measure in a function-valued space. 

If samples are not in \(C_2^{1/2}(H)\), the model fails to learn the data distribution as in Figure~\ref{fig:smoothing} (c). As we discussed in \secref{subsec:smoothing_operators}, however, one can apply a smoothing operator $A$ so that \(A(H) \subseteq C_2^{1/2}(H)\). This results in losing some information about the data, depending on the choice of the smoothing operator. Here, high-frequency noises are cut out as we use a Gaussian blur.

The results demonstrate that the proposed method will learn the distributions in function space; thus, it is\ discretization invariant. Moreover, the result further implies that the choice of noise process and smoothing operator will determine which perspectives of data distributions the models will learn. Furthermore, the experiment led to several open questions for choosing noising processes most suitable for applications.

\begin{figure}[!htb]

     \centering
     \begin{minipage}[b]{0.04\textwidth}
         \centering
         \includegraphics[width=\textwidth]{figs/background.png}
     \end{minipage}
     \hfill
     \begin{minipage}[b]{0.18\textwidth}
         \centering
         \includegraphics[width=\textwidth]{figs/background.png}
     \end{minipage}
     \hfill
     \begin{minipage}[b]{0.23\textwidth}
         \centering
         \includegraphics[width=\textwidth]{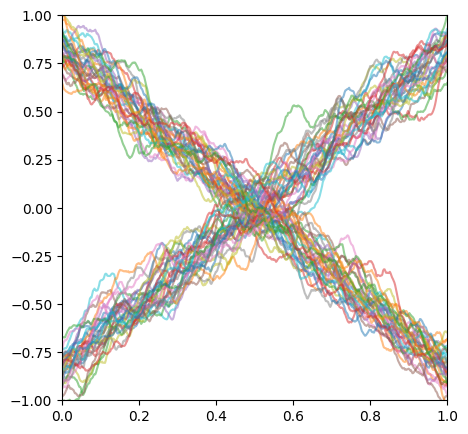}
     \end{minipage}
     \hfill
     \begin{minipage}[b]{0.23\textwidth}
         \centering
         \includegraphics[width=\textwidth]{figs/background.png}
     \end{minipage}
     \vspace{-0.5em}

     \begin{minipage}[b]{0.04\textwidth}
         \centering
         \includegraphics[width=\textwidth]{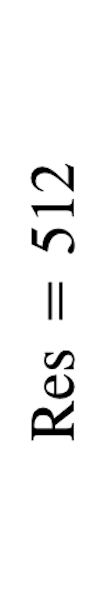}
     \end{minipage}
     \hfill
     \begin{minipage}[b]{0.23\textwidth}
         \centering
         \includegraphics[width=\textwidth]{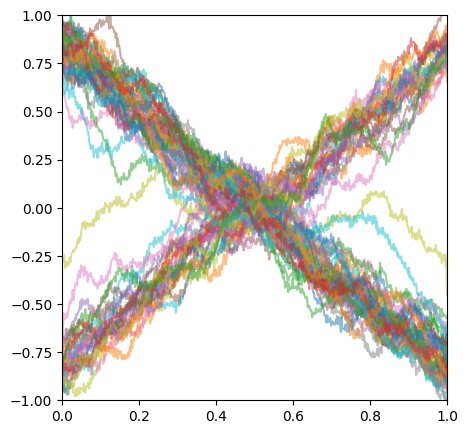}
     \end{minipage}
     \hfill
     \begin{minipage}[b]{0.23\textwidth}
         \centering
         \includegraphics[width=\textwidth]{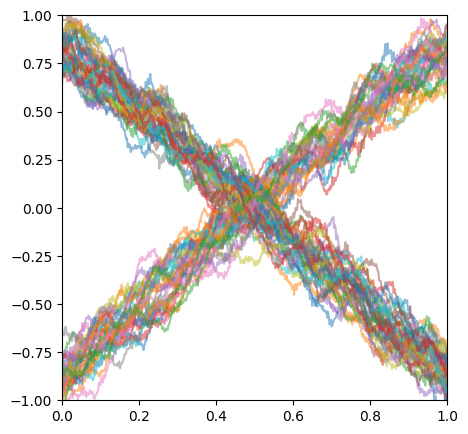}
     \end{minipage}
     \hfill
     \begin{minipage}[b]{0.23\textwidth}
         \centering
         \includegraphics[width=\textwidth]{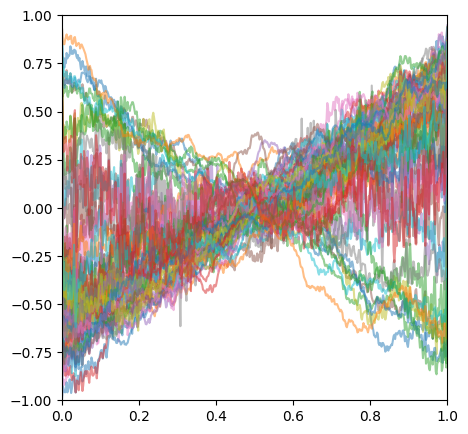}
     \end{minipage}
     \hfill
     \begin{minipage}[b]{0.23\textwidth}
         \centering
         \includegraphics[width=\textwidth]{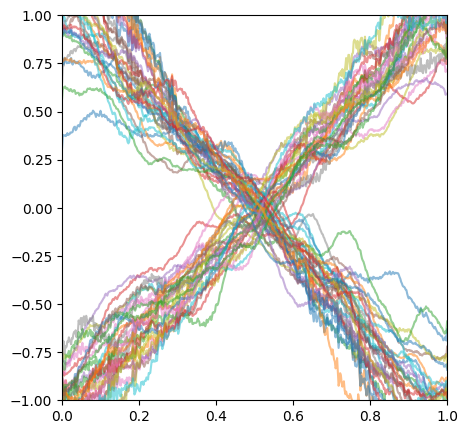}
     \end{minipage}
     \vspace{-0.5em}

     \begin{minipage}[b]{0.04\textwidth}
         \centering
         \includegraphics[width=\textwidth]{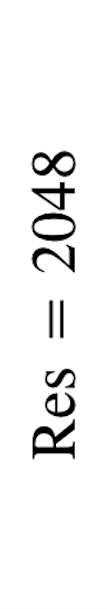}
         \vspace{-2em}
         \caption*{\textcolor{white}{()}}
     \end{minipage}
     \hfill
     \begin{minipage}[b]{0.23\textwidth}
         \centering
         \includegraphics[width=\textwidth]{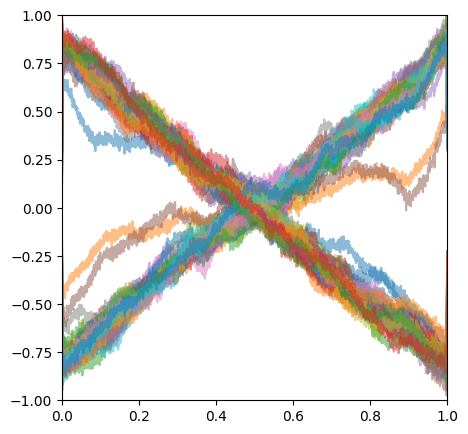}
         \vspace{-2em}
         \caption*{(a)}
     \end{minipage}
     \hfill
     \begin{minipage}[b]{0.23\textwidth}
         \centering
         \includegraphics[width=\textwidth]{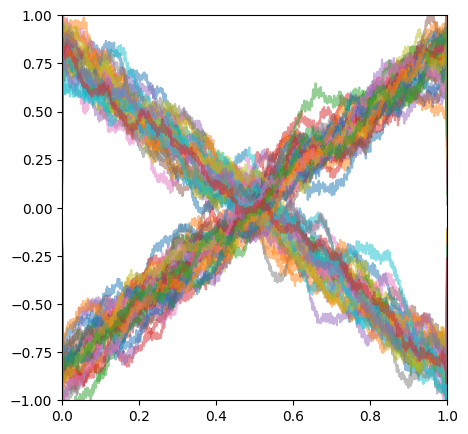}
         \vspace{-2em}
         \caption*{(b)}
     \end{minipage}
     \hfill
     \begin{minipage}[b]{0.23\textwidth}
         \centering
         \includegraphics[width=\textwidth]{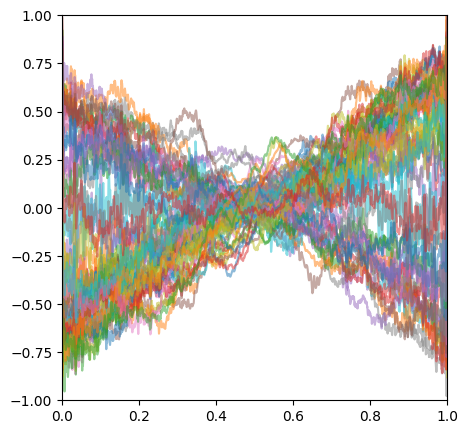}
         \vspace{-2em}
         \caption*{(c)}
     \end{minipage}
     \hfill
     \begin{minipage}[b]{0.23\textwidth}
         \centering
         \includegraphics[width=\textwidth]{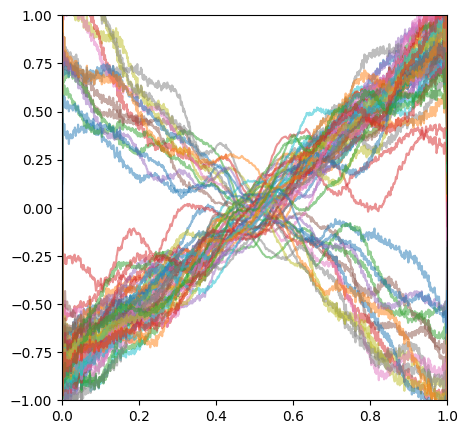}
         \vspace{-2em}
         \caption*{(d)}
     \end{minipage}


        \caption{\textbf{Smoothing Operators (\secref{subsec:smoothing_operators})} First row illustrates 128 sample paths from the data distribution \(\mu\), a Mixture of two Gaussians. The second and third row illustrate generated samples from trained models. We train a FNO architecture with varying choice of the noise' covariance \(C_2\); (a) independent Gaussian noise process, (b) for all \(u\in \mu\), \(u \in C_2^{1/2}(H)\),  (c) \(\exists \,u \notin C_2^{1/2}(H)\), and (d) apply a smoothing operator while the same covariance as in (c). For smoothing operator \(A\) in (d), we use a Gaussian blur such that \(A(H) \subseteq C_2^{1/2}(H)\). The models are trained in a \(512\) resolution, and generate with varying resolutions, such as \(512\) (Second row) and \(2048\) (Third row).}
        \label{fig:smoothing}
\end{figure}




\begin{figure}[!htb]
     \centering
     \begin{minipage}[b]{0.05\textwidth}
         \centering
         \includegraphics[width=\textwidth]{figs/background.png}
     \end{minipage}
     \hfill
     \begin{minipage}[b]{0.40\textwidth}
         \centering
         \includegraphics[width=\textwidth]{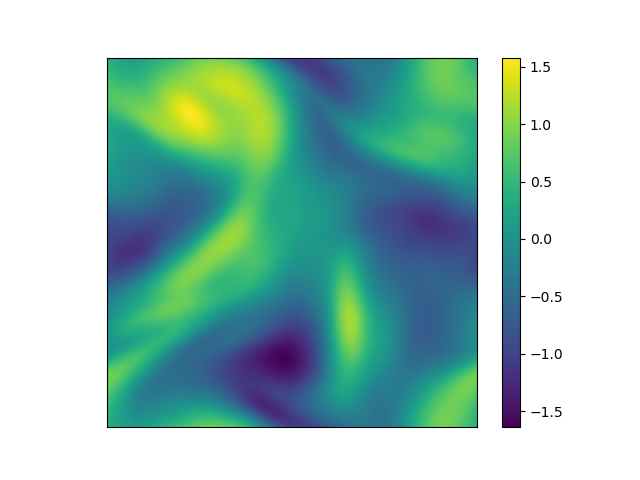}
     \end{minipage}
     \hspace{-6em}
     \hfill
     \begin{minipage}[b]{0.40\textwidth}
         \centering
         \includegraphics[width=\textwidth]{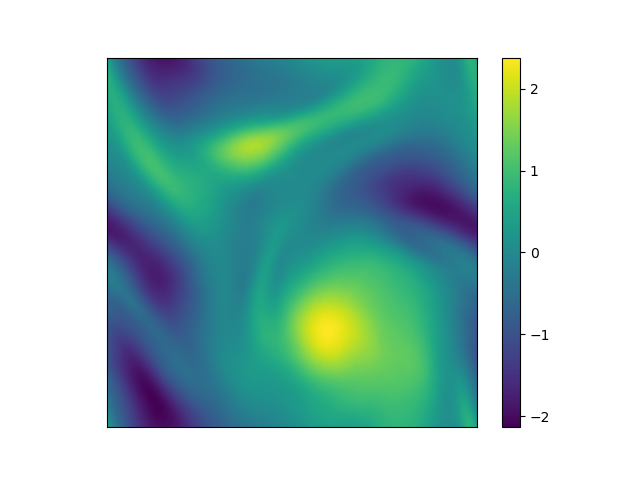}
     \end{minipage}
     \hfill
     \begin{minipage}[b]{0.05\textwidth}
         \centering
         \includegraphics[width=\textwidth]{figs/background.png}
     \end{minipage}
     \vspace{-2em}

     \begin{minipage}[b]{0.05\textwidth}
         \centering
         \includegraphics[width=\textwidth]{figs/background.png}
     \end{minipage}
     \hfill
     \begin{minipage}[b]{0.40\textwidth}
         \centering
         \includegraphics[width=\textwidth]{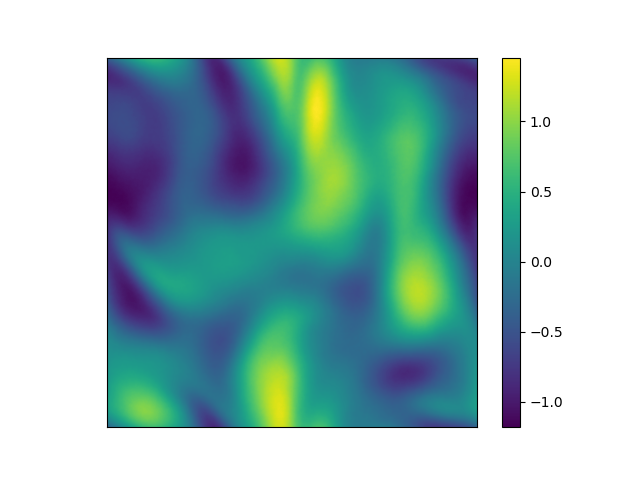}
     \end{minipage}
     \hspace{-6em}
     \hfill
     \begin{minipage}[b]{0.40\textwidth}
         \centering
         \includegraphics[width=\textwidth]{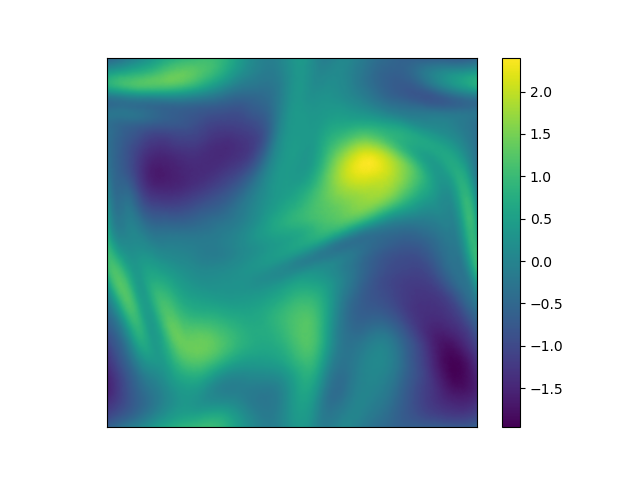}
     \end{minipage}
     \hfill
     \begin{minipage}[b]{0.05\textwidth}
         \centering
         \includegraphics[width=\textwidth]{figs/background.png}
     \end{minipage}
     \vspace{-2em}

     \begin{minipage}[b]{0.05\textwidth}
         \centering
         \includegraphics[width=\textwidth]{figs/background.png}
     \end{minipage}
     \hfill
     \begin{minipage}[b]{0.40\textwidth}
         \centering
         \includegraphics[width=\textwidth]{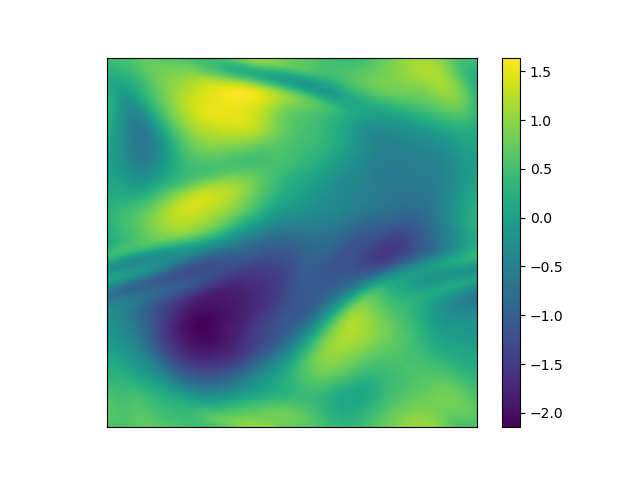}
     \end{minipage}
     \hspace{-6em}
     \hfill
     \begin{minipage}[b]{0.40\textwidth}
         \centering
         \includegraphics[width=\textwidth]{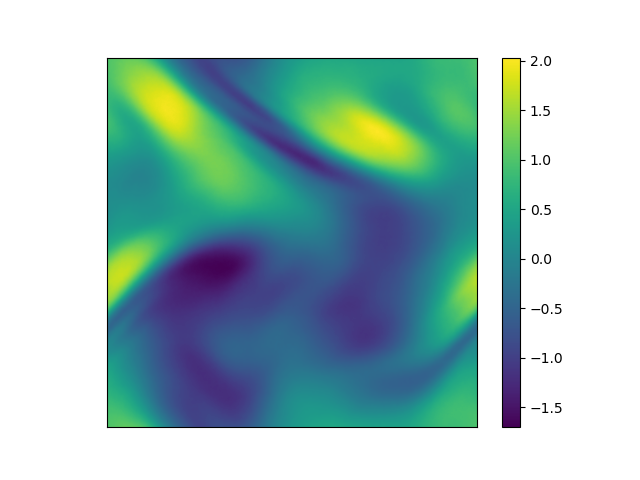}
     \end{minipage}
     \hfill
     \begin{minipage}[b]{0.05\textwidth}
         \centering
         \includegraphics[width=\textwidth]{figs/background.png}
     \end{minipage}
     \vspace{-2em}

     \begin{minipage}[b]{0.05\textwidth}
         \centering
         \includegraphics[width=\textwidth]{figs/background.png}
     \end{minipage}
     \hfill
     \begin{minipage}[b]{0.40\textwidth}
         \centering
         \includegraphics[width=\textwidth]{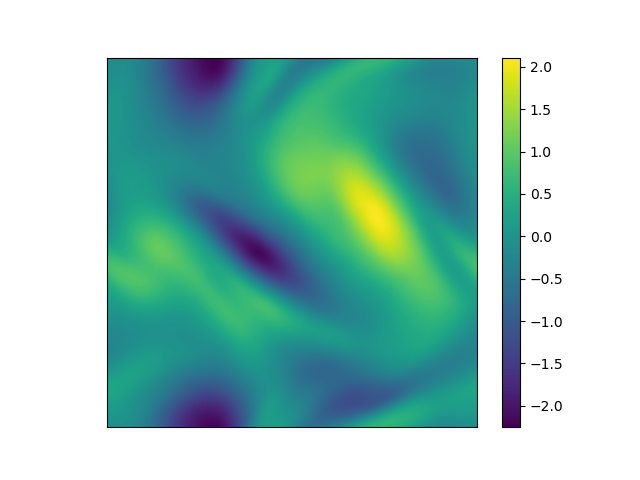}
     \end{minipage}
     \hspace{-6em}
     \hfill
     \begin{minipage}[b]{0.40\textwidth}
         \centering
         \includegraphics[width=\textwidth]{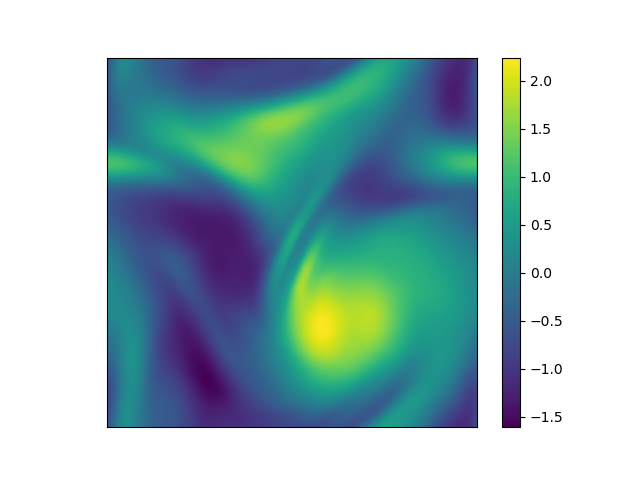}
     \end{minipage}
     \hfill
     \begin{minipage}[b]{0.05\textwidth}
         \centering
         \includegraphics[width=\textwidth]{figs/background.png}
     \end{minipage}
     \vspace{-0.5em}

     \caption{\textbf{Navier-Stokes.} Uncurated samples at the resolution \(1024 \times 1024\) from our diffusion model trained on a dataset at the resolution \(128 \times 128\).}
    \label{fig:ns_allvis}
\end{figure}

\section{Experimental Details}
\label{sec:experimental_details}
In all examples we train by picking \(\I = [10]\) and sample with Algorithm~\ref{alg:sampling_alg} by fixing \(M=200\) and \(\epsilon = 2 \times 10^{-5} \). We choose \(\sigma_1 = 1.0\) and \(\sigma_{10} = 0.01\) with all other \(\sigma\) parameters defined by a geometric sequence.
We train with a combined loss defined by \eqref{eq:learn_noise} where we re-scale the the noise by \(\sigma_t^{-1}\) and the score by \(\sigma_t\), following \cite{song2019generative}. In particular,
our model learns to approximate \(v \mapsto \sigma_t^{-2} \big( R D_{H_{\mu_t}} \Phi(u,t) - v \big )\). Note that the \(\sigma_t^{-2}\) term is canceled by the adaptive time-step in Algorithm~\ref{alg:sampling_alg}, however, as
in \cite{song2019generative}, we find that this re-scaling significantly improves performance for all models. We leave a theoretical analysis of this for future work. To be explicit, our loss function is
$$\min_{\theta} \frac{1}{|\I|}\sum_{t \in \I} \mathbb{E}_{u \sim \mu}\mathbb{E}_{\eta_t \sim \mu_t} \left\| \frac{\eta_t}{\sigma_t} + \sigma_t F_\theta(u + \eta_t,t)\right\|^2$$
where \(\|\cdot\|\) is the \(L^2(D;\sR)\) norm and \(D\) is problem-dependent. We train with the Adam optimizer for a total of \(300\) epochs and an initial learning rate \(10^{-3}\), which is decayed by half every \(50\) epochs.

\subsection{Gaussian Mixture}
\label{subsec:exp_gaussianmixutre}

We consider the problem setting of \secref{subsec:further_gaussian_mixture} with \(d=1\) and \(D = (0,2\pi)\). We pick \(f_1 = \sin (x/2)\) and \(f_2 = -f_1\) as well as \(p=0.5\).
We generate \(N=10,000\) samples for training at a resolution of \(2048\) and subsample these to obtain all other datasets.
For the data covariance \(C_1\), we choose \(\alpha_1 = 3\), \(\sigma_1 = 3\), and \(\tau_1 = 3\), and, for the noise covariance \(C_2\), we choose \(\alpha_1 = 0.6\), \(\sigma_1 = 0.5\), and \(\tau_1 = 0.1\).
We train with a FNO architecture wich retains 48 modes and has a width of 128. We re-train the model at each separate resolution. In Figure~\ref{fig:gm_allvis} we visualize samples from each of the models trained 
with white noise as well as the trace-class noise with covariance \(C_2\) at the resolutions \(64,512\), and \(2048\). We point out that the models trained at high resolutions with white noise completely
fail to capture the right distribution.

\begin{figure}
     \centering
     \begin{minipage}[b]{0.24\textwidth}
         \centering
         \includegraphics[width=\textwidth]{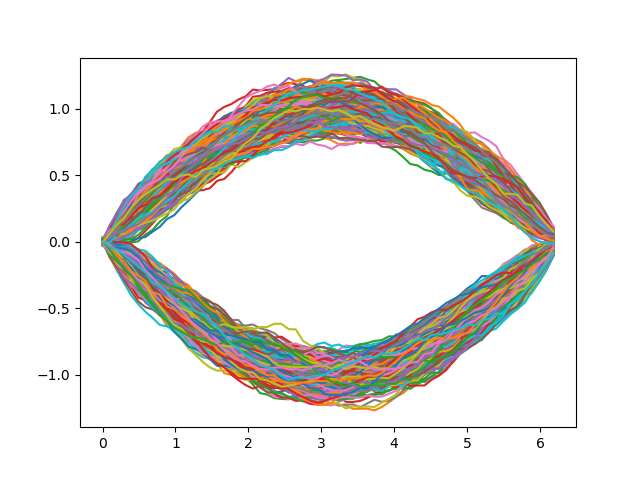}
     \end{minipage}
     \hfill
     \begin{minipage}[b]{0.24\textwidth}
         \centering
         \includegraphics[width=\textwidth]{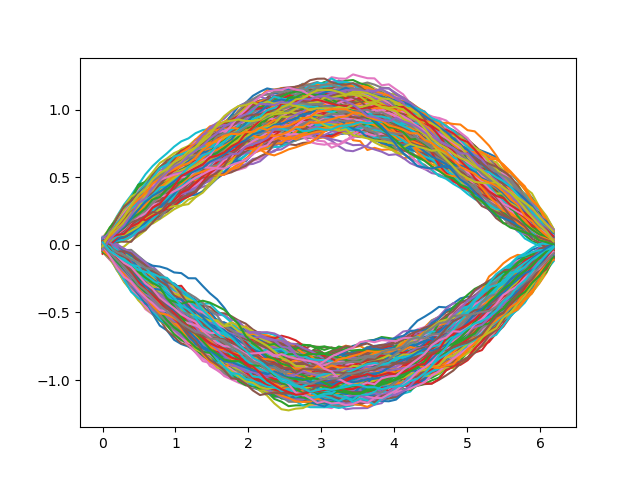}
     \end{minipage}
     \hfill
     \begin{minipage}[b]{0.24\textwidth}
         \centering
         \includegraphics[width=\textwidth]{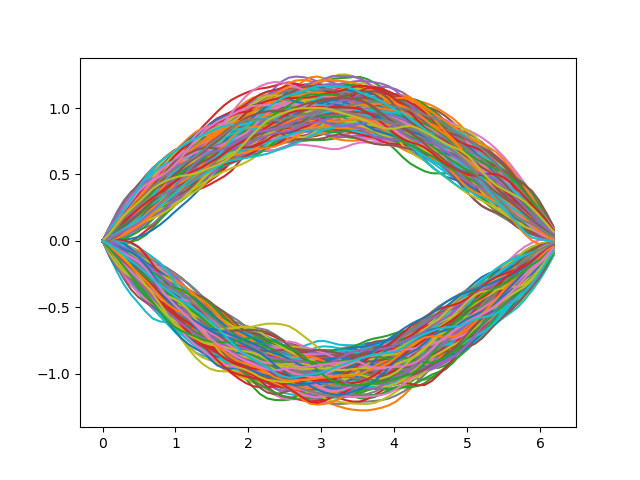}
     \end{minipage}
     \hfill
     \begin{minipage}[b]{0.24\textwidth}
         \centering
         \includegraphics[width=\textwidth]{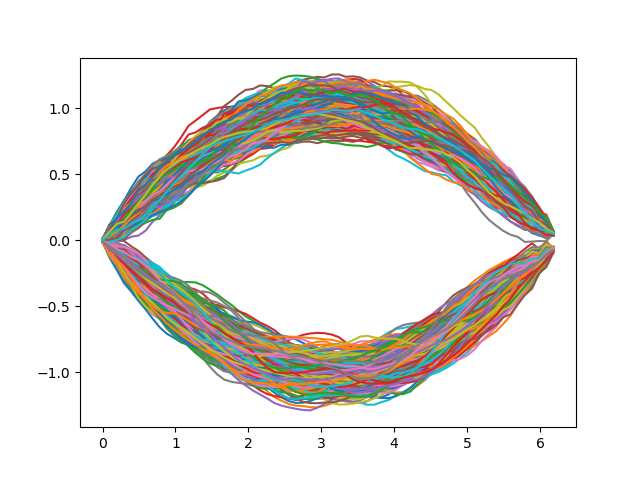}
     \end{minipage}
     \vspace{-0.5em}

     \begin{minipage}[b]{0.24\textwidth}
         \centering
         \includegraphics[width=\textwidth]{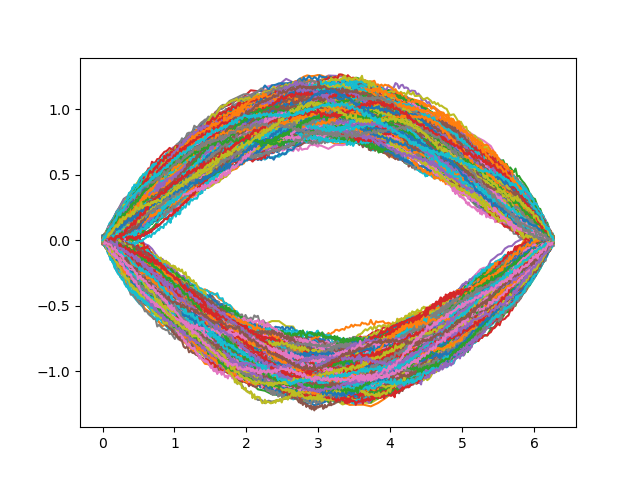}
     \end{minipage}
     \hfill
     \begin{minipage}[b]{0.24\textwidth}
         \centering
         \includegraphics[width=\textwidth]{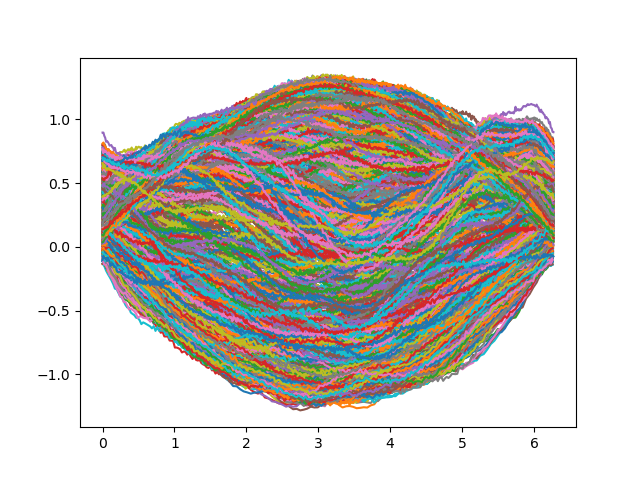}
     \end{minipage}
     \hfill
     \begin{minipage}[b]{0.24\textwidth}
         \centering
         \includegraphics[width=\textwidth]{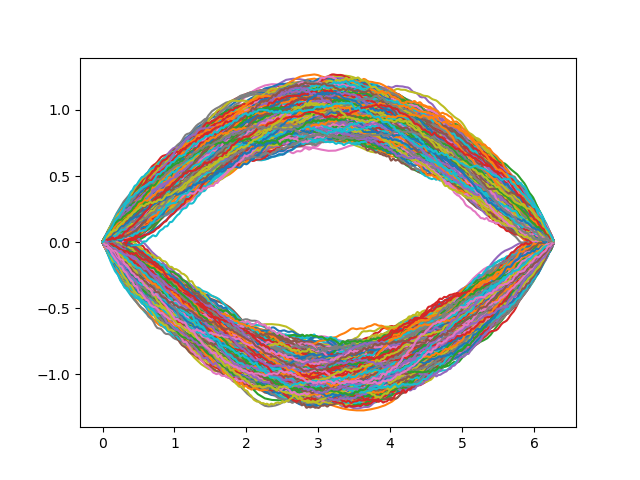}
     \end{minipage}
     \hfill
     \begin{minipage}[b]{0.24\textwidth}
         \centering
         \includegraphics[width=\textwidth]{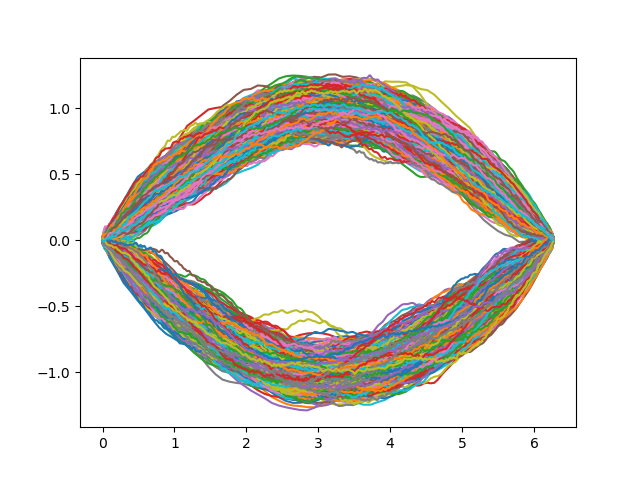}
     \end{minipage}
     \vspace{-0.5em}

     \begin{minipage}[b]{0.24\textwidth}
         \centering
         \includegraphics[width=\textwidth]{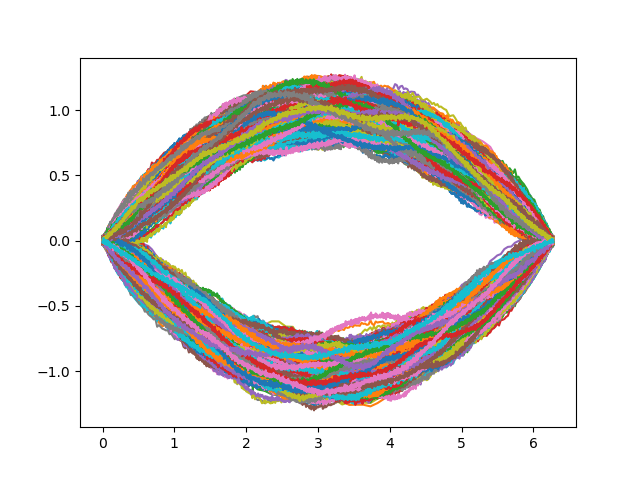}
     \end{minipage}
     \hfill
     \begin{minipage}[b]{0.24\textwidth}
         \centering
         \includegraphics[width=\textwidth]{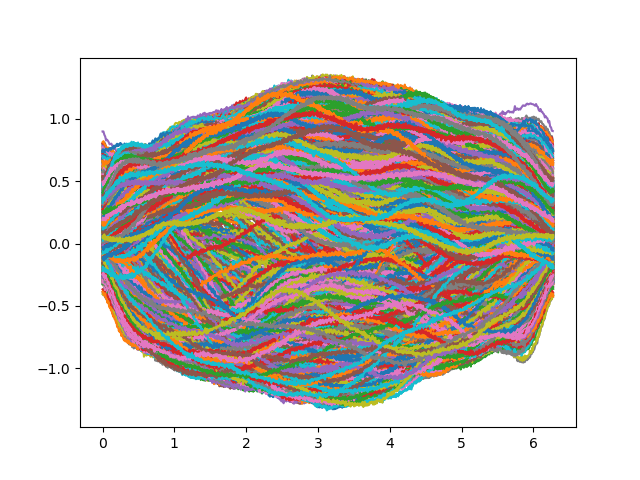}
     \end{minipage}
     \hfill
     \begin{minipage}[b]{0.24\textwidth}
         \centering
         \includegraphics[width=\textwidth]{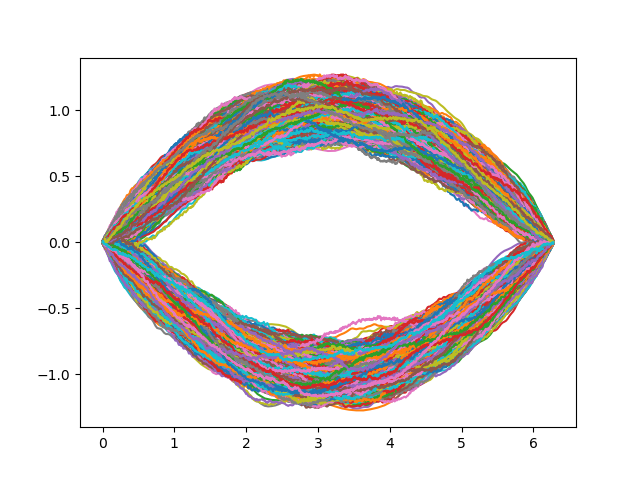}
     \end{minipage}
     \hfill
     \begin{minipage}[b]{0.24\textwidth}
         \centering
         \includegraphics[width=\textwidth]{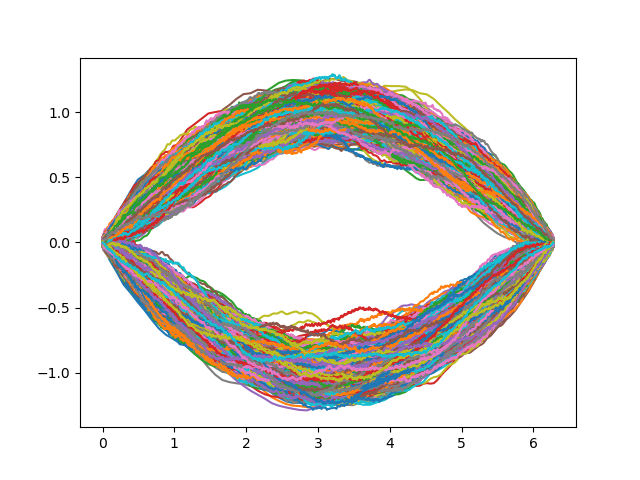}
     \end{minipage}
     \vspace{-0.5em}
     
    \caption{\textbf{Guassian mixture.} Each row represents the resolutions \(64,512,2048\) respectively. The first and third columns show samples from \(\nu_T\) (the perturbed data measure with the smallest amount of noise added) for the white noise and trace-class noise respectively. The second and fourth columns show samples generated from the trained model using white noise and trace-class noise respectively.}
        \label{fig:gm_allvis}
\end{figure}

\subsection{Navier-Stokes}
\label{subsec:exp_navierstokes}

We consider a problem setting similar to \secref{subsec:further_navier_stokes}. We fix the initial condition \(u(\cdot,0) = 0\) and instead generate random forcings \(f\) from the Gaussian \(\rho\). 
The same theory as in \secref{subsec:further_navier_stokes} still applies and the pushforward is non-Gaussian due to the non-linearity of the PDE. We solve it up to the final time \(T=5\) with the pseudo-spectral scheme of \cite{chandler2013invariant} 
with \(\epsilon = 1/500\) and \(N=10,000\) samples for training. We pick \(\alpha_1 = 4\), \(\sigma_1 = 3\sqrt{3}\) and \(\tau_1 = 3\) for the reference Gaussian \(\rho\), following \citep{de2022cost}. All data is generated with a \(1024 \times 1024\) resolution 
and the \(128 \times 128\) is created from it by sub-sampling. With train our mode with a FNO architecture retaining \(80 \times 80\) modes with a width of \(64\). We pick the parameters \(\alpha_2 = 1.5\), \(\sigma_2 = 4\), and \(\tau_2 = 5\)
for our noise covariance. In Figure~\ref{fig:ns_allvis}, we show more samples generated by the model, performing zero-shot super-resolution.

\subsection{Volcano Dataset}
\label{app:volcano}

For the volcano experiments, we use the loss formulation from Equation (\ref{eq:learn_noise}) but employ a noise schedule $\{\sigma_1, \dots, \sigma_L\}$ (as is standard with SBGMs). Inspired by \cite{song2020improved} we optimize a weighted variant of it where $F_{\theta}$ is preconditioned with $\sigma_i$:
\begin{subequations} \label{eq:sm}
\begin{align}
& \min_{\theta \in \mathbb{R}^{p}} \mathbb{E}_{u \sim \mu} \mathbb{E}_{\eta_i} \| \eta_i + \sigma_i \Ft(u+\eta_i, \sigma_i) \|^{2} \\
= & \min_{\theta \in \mathbb{R}^{p}} \mathbb{E}_{u \sim \mu} \mathbb{E}_{\epsilon \sim \mathcal{N}(0,\mathbf{I})} \mathbb{E}_{i \sim U(1,L)} \|  \mathbf{L}_i \epsilon + \sigma_i \Ft(u+\mathbf{L}_i \epsilon, \sigma_i) \|^{2} \\
= & \min_{\theta \in \mathbb{R}^{p}} \mathbb{E}_{u \sim \mu} \mathbb{E}_{\epsilon \sim \mathcal{N}(0,\mathbf{I})} \mathbb{E}_{i \sim U(1,L)} \|  \sigma_i \mathbf{L} \epsilon + \sigma_i \Ft(u+ \sigma_i \mathbf{L} \epsilon, \sigma_i) \|^{2}, \label{eq:sm_1c} \\
= & \min_{\theta \in \mathbb{R}^{p}} \mathbb{E}_{u \sim \mu} \mathbb{E}_{\epsilon \sim \mathcal{N}(0,\mathbf{I})} \mathbb{E}_{i \sim U(1,L)} \|  \sigma_i\big( \mathbf{L} \epsilon + \Ft(u+ \sigma_i \mathbf{L} \epsilon, \sigma_i) \big) \|^{2}
\end{align}
\end{subequations}
where $\eta_i \sim \mathcal{N}(0, \mathbf{C}_i^{2}) = \mathcal{N}(0, \sigma_i^{2}\mathbf{C}^{2})$, and via the reparameterisation trick this can be re-written as $\eta_i = \sigma_i \mathbf{L} \epsilon$, where $\mathbf{L} = \text{chol}(\mathbf{C})$ and $\epsilon \sim \mathcal{N}(0,\mathbf{I})$. Here, we use a covariance computed using the RBF kernel over a 2D meshgrid representing the coordinates of the image (function): 
\begin{align}
\mathbf{C}(x)_{ij} = \exp\Big( -\frac{\| x_i - x_j \|^{2}}{2\gamma^{2}} \Big), \ \, i, j \in \{1, \dots, s^2\}
\end{align}
where $x \in [0,1]^{s^2 \times 2}$ for a spatial resolution of $s$, and $\mathbf{C} \in \mathbb{R}^{s^2 \times s^2}$. The hyperparameter $\gamma$ controls the smoothness of the noise, with larger values indicating higher levels of smoothness. Example sample noises are illustrated in Figure \ref{fig:rbf_samples}.

\begin{figure}[H]
    \centering
    \begin{subfigure}[b]{0.3\textwidth}
        \centering
        \includegraphics[width=\textwidth,trim=570 0 0 0,clip]{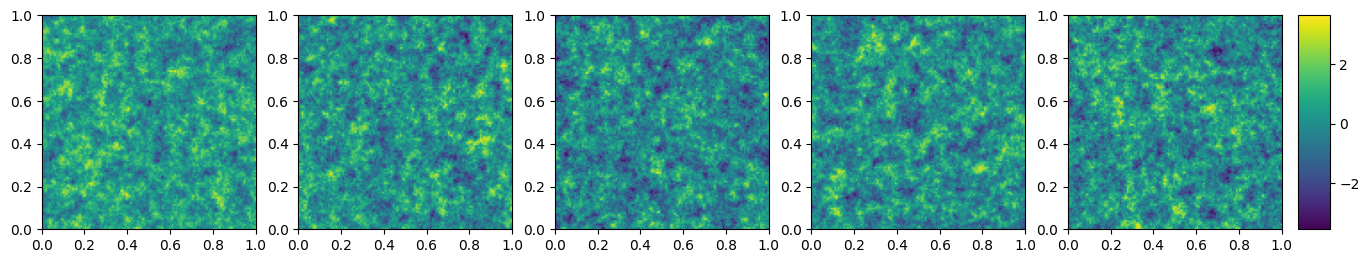}
        \caption{$\gamma = 0.1$.}
        \label{fig:rbf_samples_01}
    \end{subfigure}
    \begin{subfigure}[b]{0.3\textwidth}
        \centering
        \includegraphics[width=\textwidth,trim=570 0 0 0,clip]{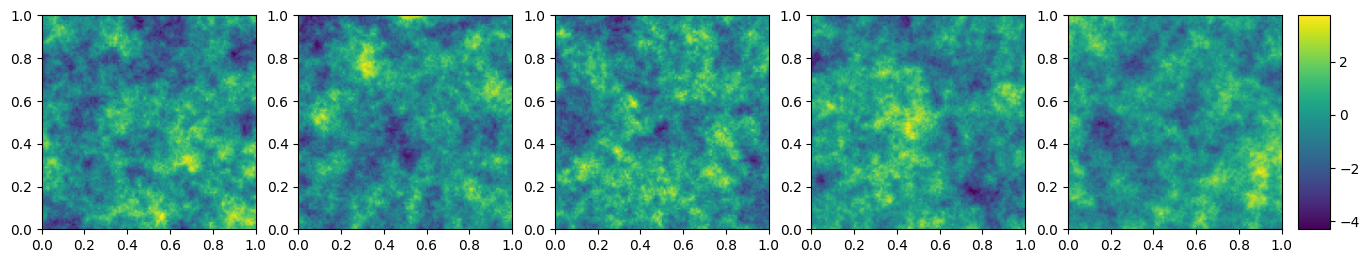}
        \caption{$\gamma = 0.2$.}
        \label{fig:rbf_samples_02}
    \end{subfigure}
    \begin{subfigure}[b]{0.3\textwidth}
        \centering
        \includegraphics[width=\textwidth,trim=570 0 0 0,clip]{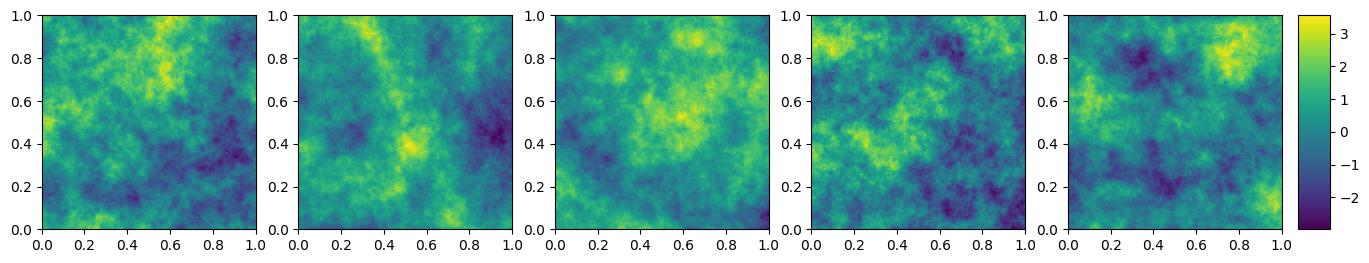}
        \caption{$\gamma = 0.3$.}
        \label{fig:rbf_samples_03}
    \end{subfigure}
    \caption{\textbf{Samples with RBF kernels}: Sample noises for varying smoothness parameters $\gamma$ of RBF kernels.}
    \label{fig:rbf_samples}
\end{figure}

As suggested in \cite{song2020improved}, we redefine $\Ft(\cdot, \sigma) = \Ft(\cdot) / \sigma$ since the authors noted that this makes the noise prediction task more robust to a wide range of noise scales. 

\paragraph{Architecture} The architecture we use is a U-shaped neural operator (UNO). This architecture consists of a series of Fourier neural operator blocks (FNOBlocks) which progressively downsample or upsample the input. We use the implementation of FNOBlocks from the Neural Operators library \cite{li2020fourier, kovachki2021neural}. Concretely, for the Volcano dataset the example $u$ has a spatial dimension of $120 \times 120$ with two channels, and is lifted from $2 \rightarrow 128$ channels with a preprocessing convolutional layer, which also pads the 120 pixel image (function) to 128 pixels. Afterwards, the lifted input is run through four FNOBlocks which progressively lift the channel dimension and spatial dimensions to $(128 \rightarrow 256 \rightarrow 512 \rightarrow 512)$ and $(128 \rightarrow 96 \rightarrow 64 \rightarrow 32)$, respectively. A similar set of blocks is used in the decoder block along with skip connections.

For FNOBlock, we use Tucker factorisation with a rank of 0.1. In order to ensure the number of learnable parameters does not explode, we also constrain the number of Fourier modes in each FNOBlock to be 50\% of the number of input channel dimensions for each block. The resulting U-shaped neural operator contains a total of ~142M learnable parameters.

\paragraph{Training Details}
We train the models using ADAM \cite{kingma2014adam} with default moving average hyperparameters $\beta = (0.9, 0.999)$, with a learning rate of $2e^{-4}$. For the noise schedule, we employ a geometric schedule using $(\sigma_1, \sigma_L) = (30, 0.01)$, for $L = 500$ time steps. The number of SGLD iterations per timestep is $T = 3$ and we use a step size of $6e^{-6}$. At generation time, we use annealed SGLD and run the Markov chain on $u_0 \sim \mathcal{N}(0, \mathbf{C})$.


\paragraph{Hyperparameters} We use circular skewness and circular variance, moments of circular variables which were originally proposed in GANO \citep{rahman2022generative} Given an image (function) $u \in \mathbb{R}^{s^2 \times 2}$ we can define its angle as $\theta = \text{atan2}(u_{:,1}, u_{:,2})$. If we define $R_p(\theta) = \frac{1}{s^2} \sqrt{z_p(\theta)}$ for $z_p(\theta) = (\sum_{k}^{s^2}( \cos(p\theta_k) + i \sin(p\theta_k))^2$ and $\varphi_p(\theta) = \text{arg}(z_p(\theta))$ (for $s^2$ spatial dimensions) then:
\begin{align} \label{eq:metrics}
\wvar(\theta) = 1 - R_1(\theta), \ \ \ \wskew(\theta) = \frac{R_2(\theta)\sin(\varphi_{2}(\theta) - 2\varphi_{1}(\theta))}{(1 - R_1(\theta))^{3/2}}
\end{align}
where $i = \sqrt{-1}$ and $k$ is a summation over the spatial dimensions $s^2$. 


\newpage
\null
\vfill
\begin{figure}[htb!]
    \centering
    \begin{subfigure}[b]{0.7\textwidth}
        \centering
        \includegraphics[width=0.5\textwidth]{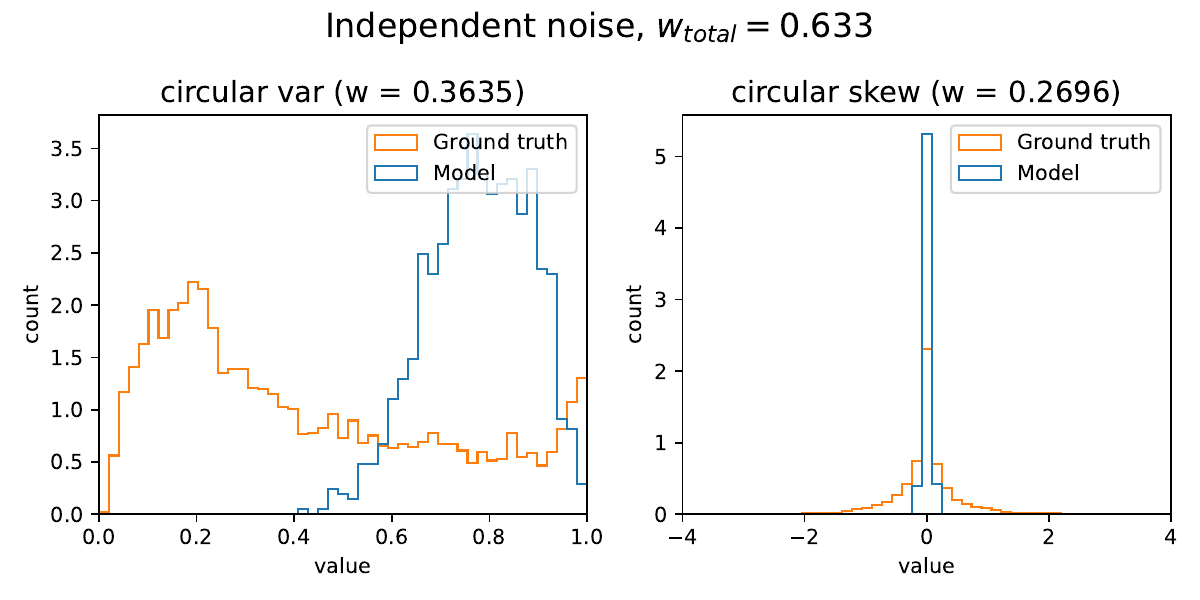} \ \ \includegraphics[width=0.4\textwidth,trim=0 280 140 0,clip]{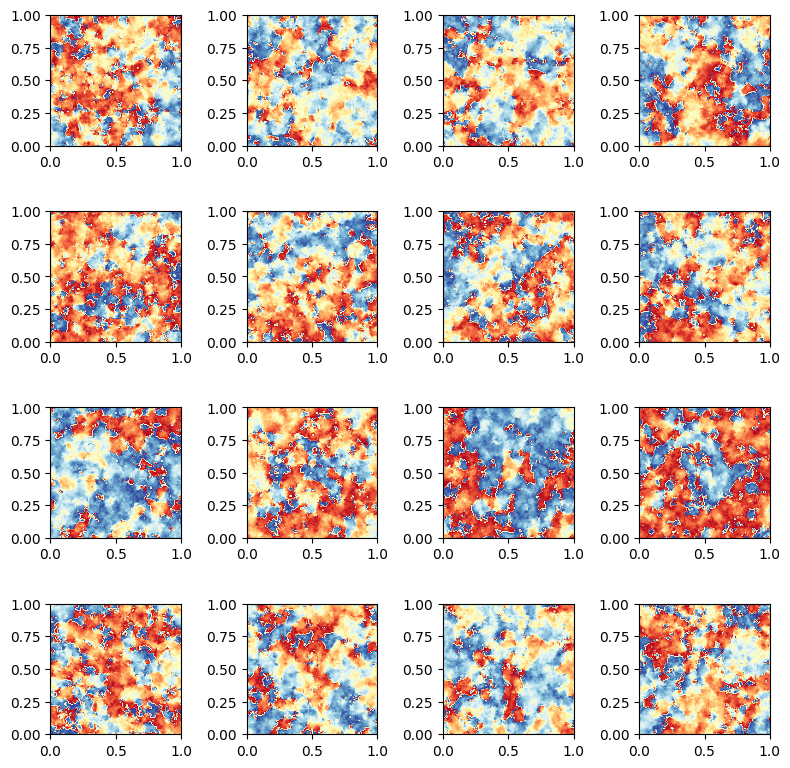}
        \caption{Independent Gaussian noise, $\wtotal = 0.633$}
        \label{fig:d_sr_hist_wn}
    \end{subfigure} \\
    \vspace{0.4cm}
    \begin{subfigure}[b]{0.7\textwidth}
        \centering
        \includegraphics[width=0.5\textwidth]{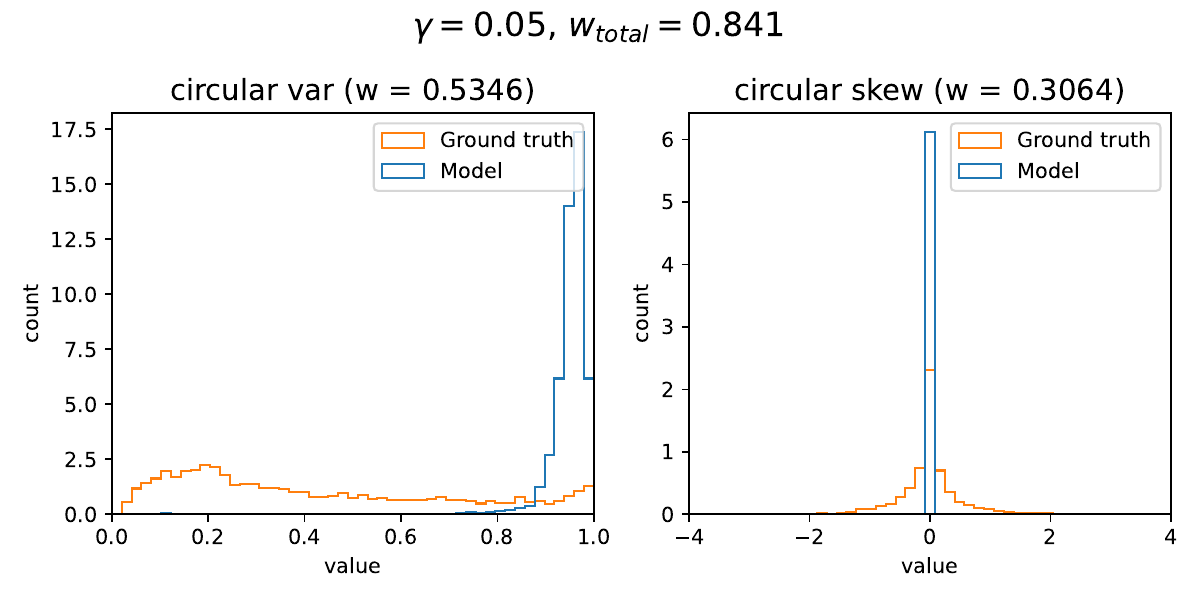} \ \ \includegraphics[width=0.4\textwidth,trim=0 280 140 0,clip]{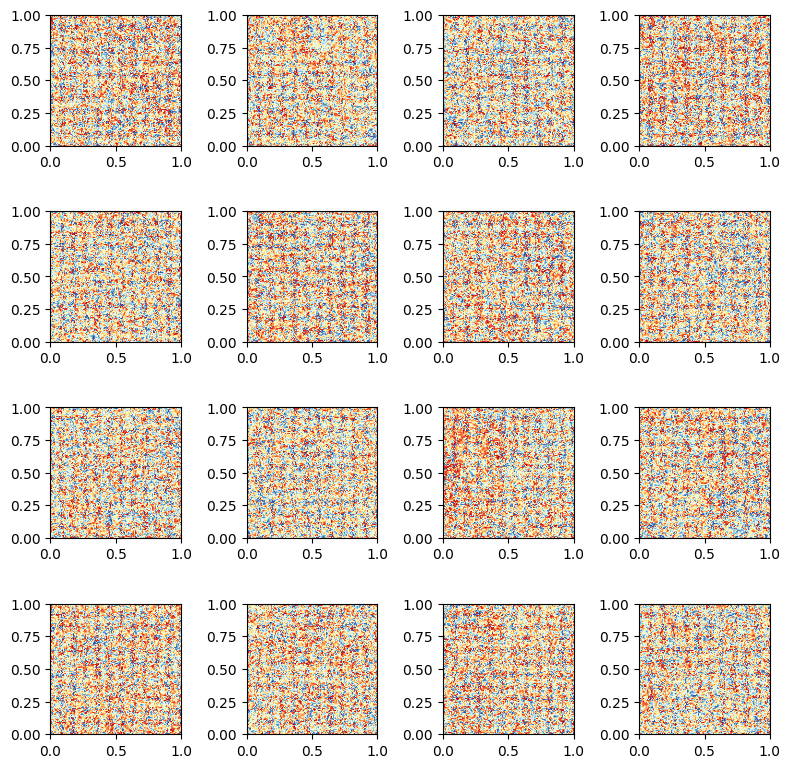}
        \caption{RBF noise, $\gamma = 0.05$, $\wtotal = 0.841$}
        \label{fig:d_sr_hist_005}
    \end{subfigure} \\
    \vspace{0.4cm}
    \begin{subfigure}[b]{0.7\textwidth}
        \centering
        \includegraphics[width=0.5\textwidth]{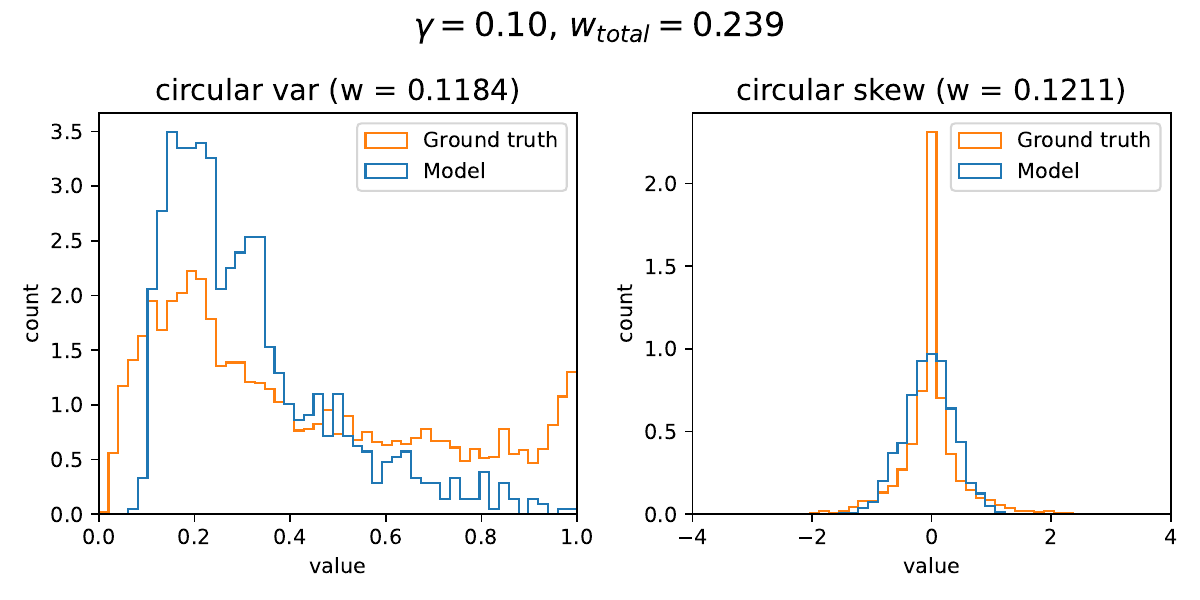} \ \ \includegraphics[width=0.4\textwidth,trim=0 280 140 0,clip]{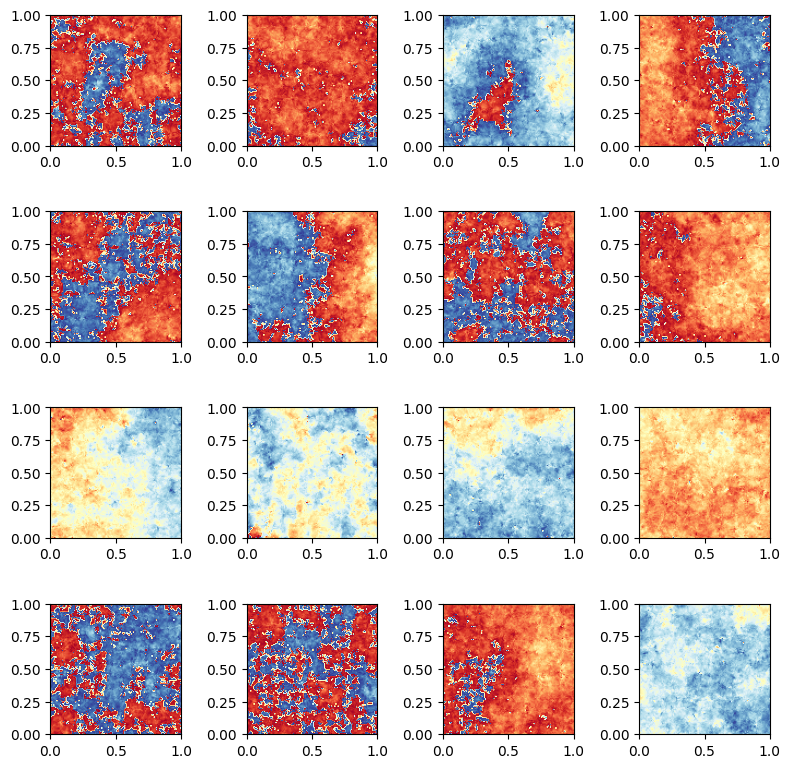}
        \caption{RBF noise, $\gamma = 0.1$, $\wtotal = 0.239$}
        \label{fig:d_sr_hist_01}
    \end{subfigure} \\
    \vspace{0.4cm}
    \begin{subfigure}[b]{0.7\textwidth}
        \centering
        \includegraphics[width=0.5\textwidth]{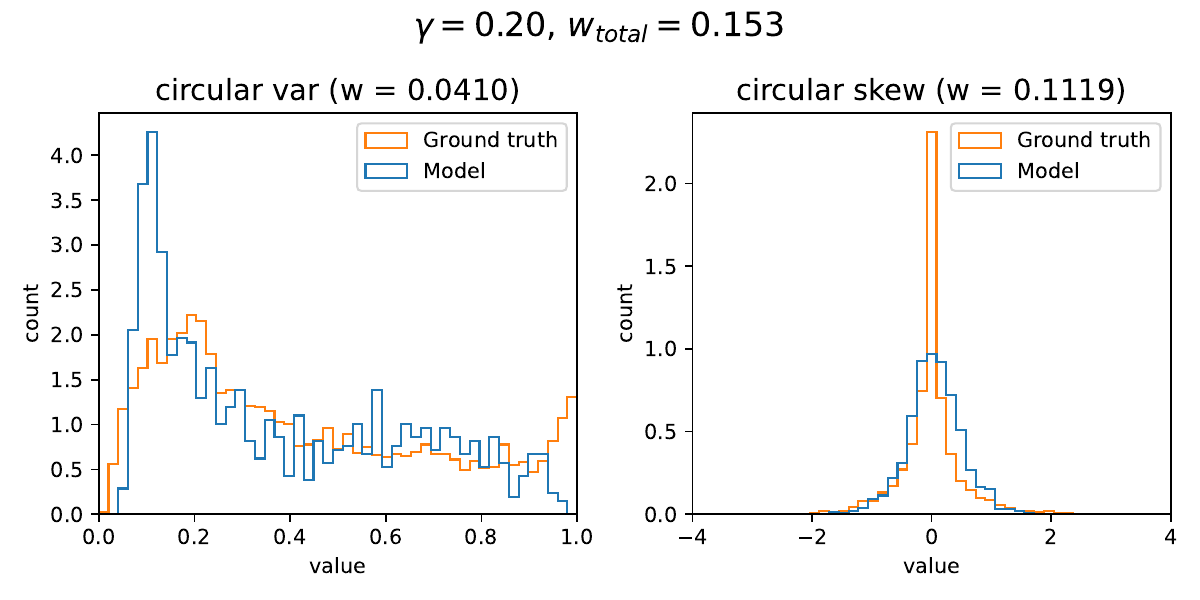} \ \ \includegraphics[width=0.4\textwidth,trim=0 280 140 0,clip]{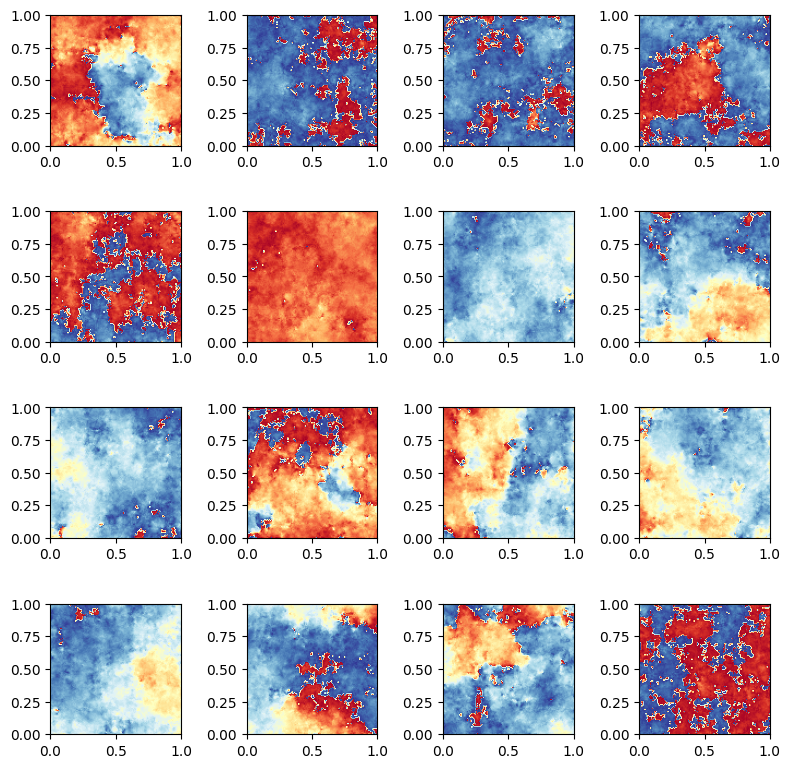}
        \caption{RBF noise, $\gamma = 0.2$, $\wtotal = 0.153$}
        \label{fig:d_sr_hist_02}
    \end{subfigure}
    \caption{Super-resolution experiments for DDO. Each subplot corresponds to a different level of RBF smoothness, which is denoted by $\gamma$ (larger $\gamma$ corresponds to smoother noise). DDO is trained on $60 \times 60$ resolution functions, and $120 \times 120$ functions are produced as per \secref{subsec:volcano}. Variance and skew statistics are computed from $M=1024$ samples and the Wasserstein distance is computed between those samples and the corresponding statistics from the $120 \times 120$ resolution version of the training set (i.e. the original resolution). We can see that both the independent noise experiment and $\gamma = 0.05$ experiment fail to produce plausible examples.}
    \label{fig:diffusion_superres}
\end{figure}
\vfill

\newpage
\null
\vfill
\begin{figure}[htb!]
    \centering
    \begin{subfigure}[b]{0.7\textwidth}
        \centering
        \includegraphics[width=0.55\textwidth]{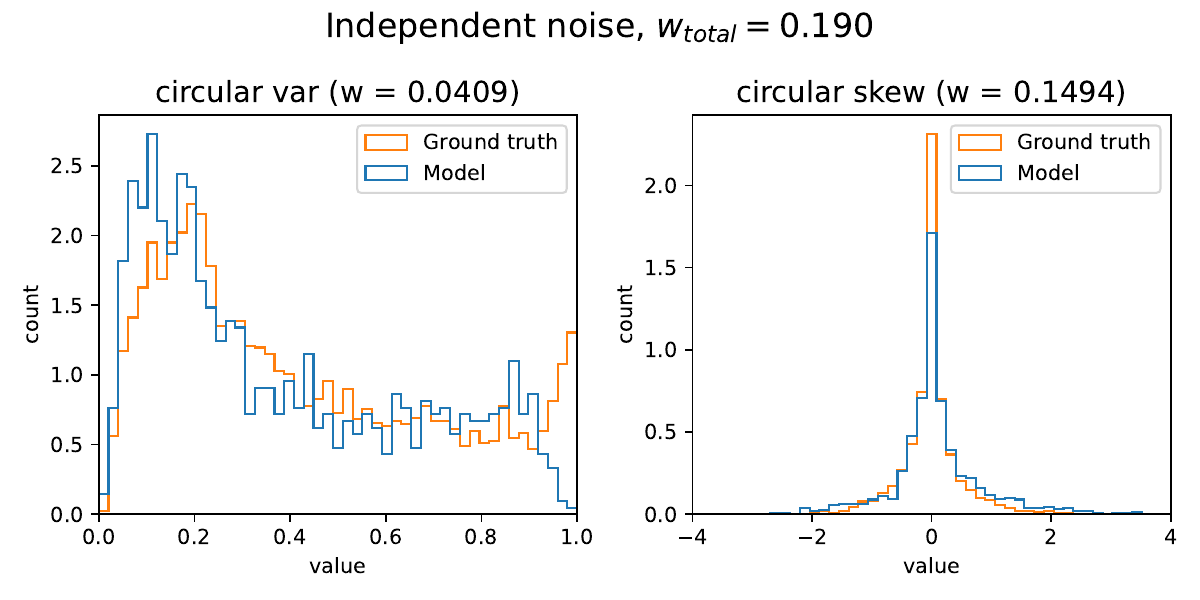} \ \ \includegraphics[width=0.4\textwidth,trim=0 280 140 0,clip]{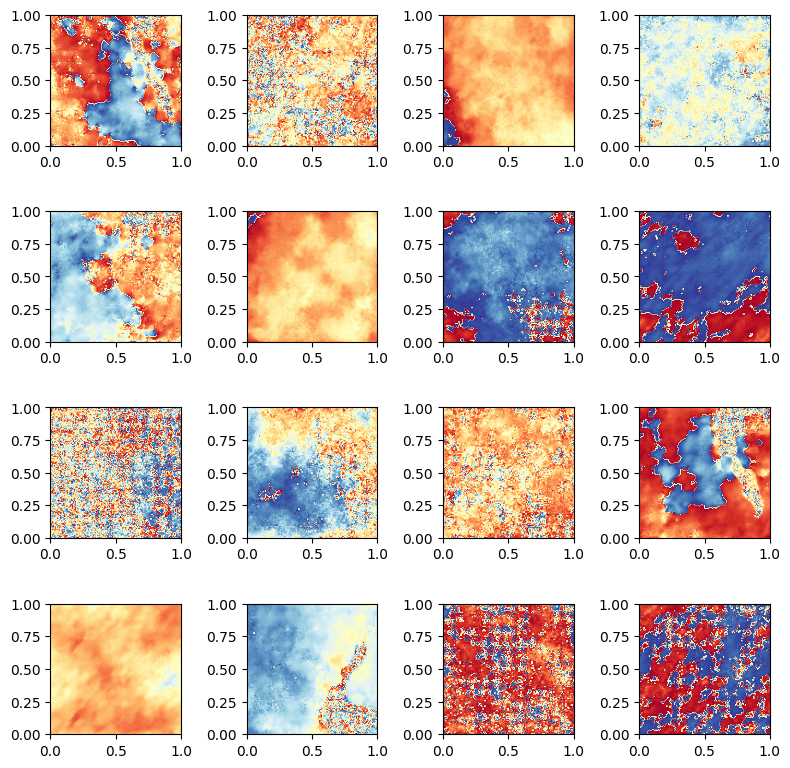}
        \caption{Independent Gaussian noise, $\wtotal = 0.187$}
        \label{fig:g_sr_hist_wn}
    \end{subfigure} \\
    \vspace{0.4cm}
    \begin{subfigure}[b]{0.7\textwidth}
        \centering
        \includegraphics[width=0.55\textwidth]{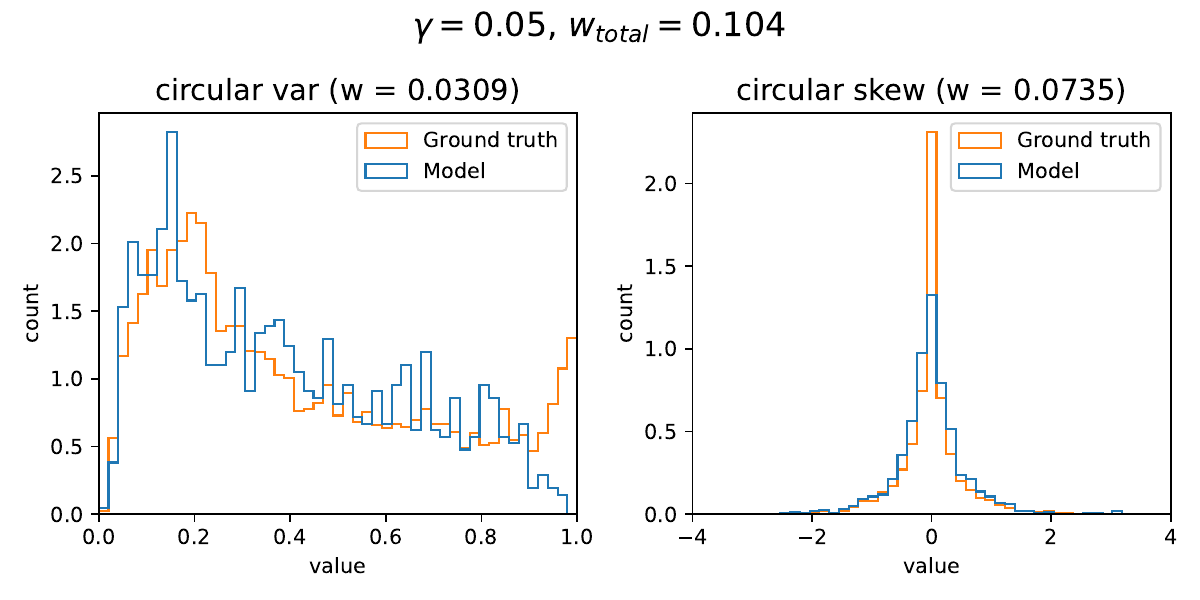} \ \ \includegraphics[width=0.4\textwidth,trim=0 280 140 0,clip]{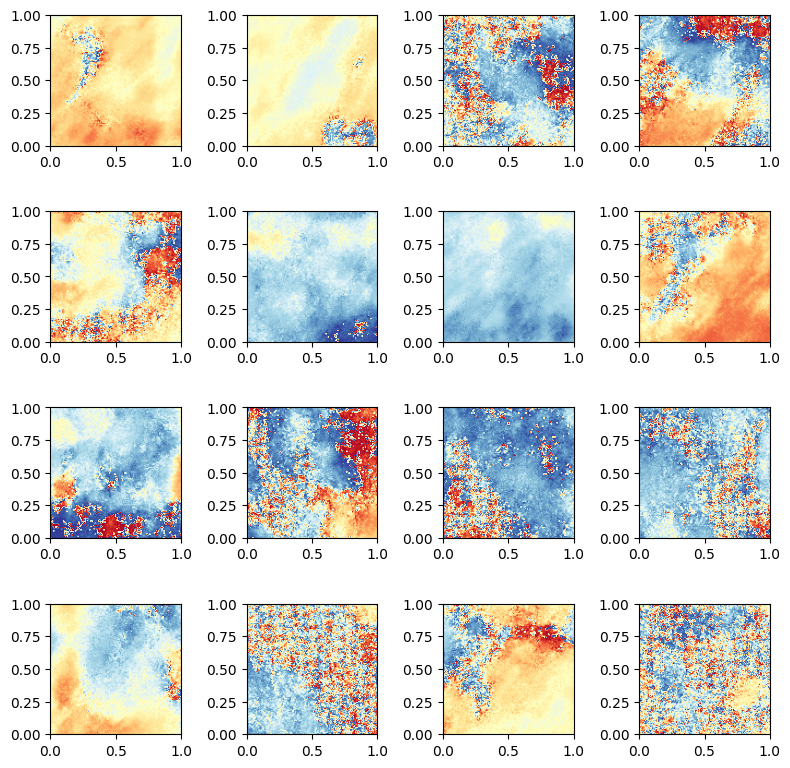}
        \caption{RBF noise, $\lambda = 0.05$, $\wtotal = 0.104$}
        \label{fig:g_sr_hist_005}
    \end{subfigure} \\
    \vspace{0.4cm}
    \begin{subfigure}[b]{0.7\textwidth}
        \centering
        \includegraphics[width=0.55\textwidth]{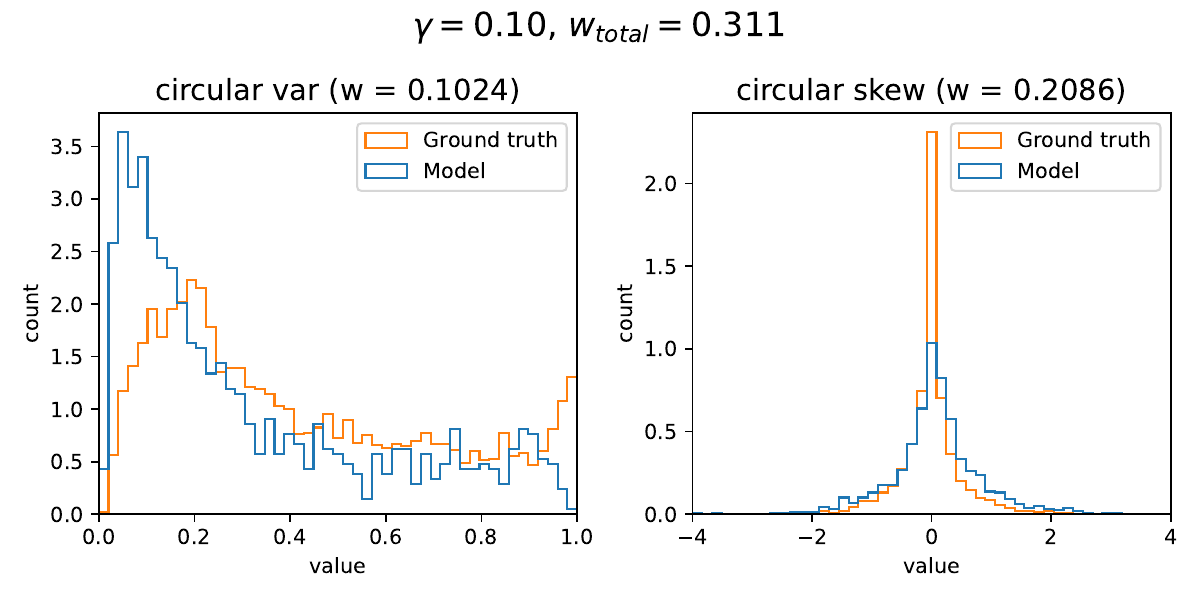} \ \ \includegraphics[width=0.4\textwidth,trim=0 280 140 0,clip]{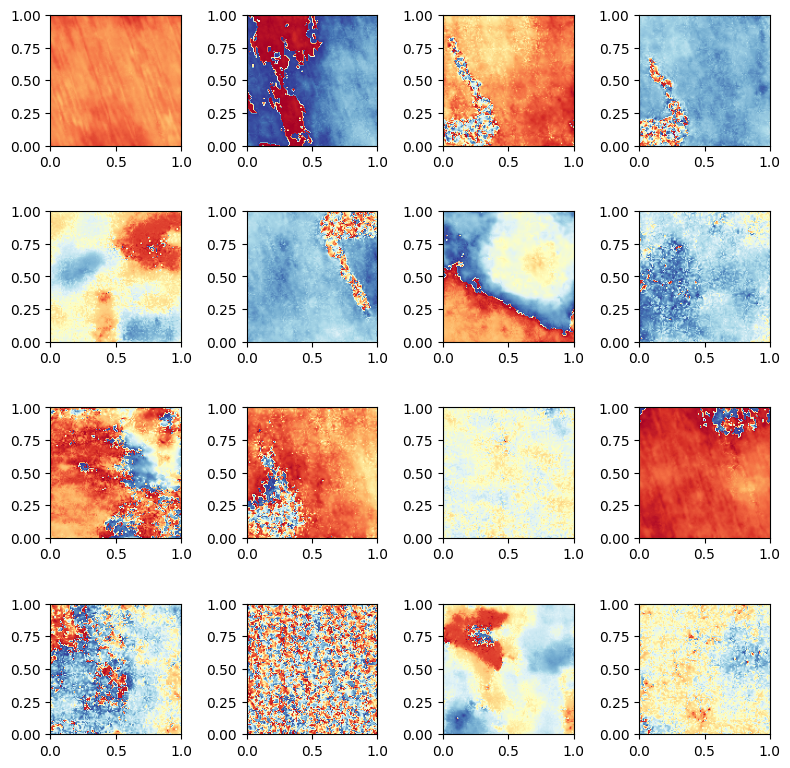}
        \caption{RBF noise, $\lambda = 0.1$, $\wtotal = 0.311$}
        \label{fig:g_sr_hist_01}
    \end{subfigure} \\
    \vspace{0.4cm}
    \begin{subfigure}[b]{0.7\textwidth}
        \centering
        \includegraphics[width=0.55\textwidth]{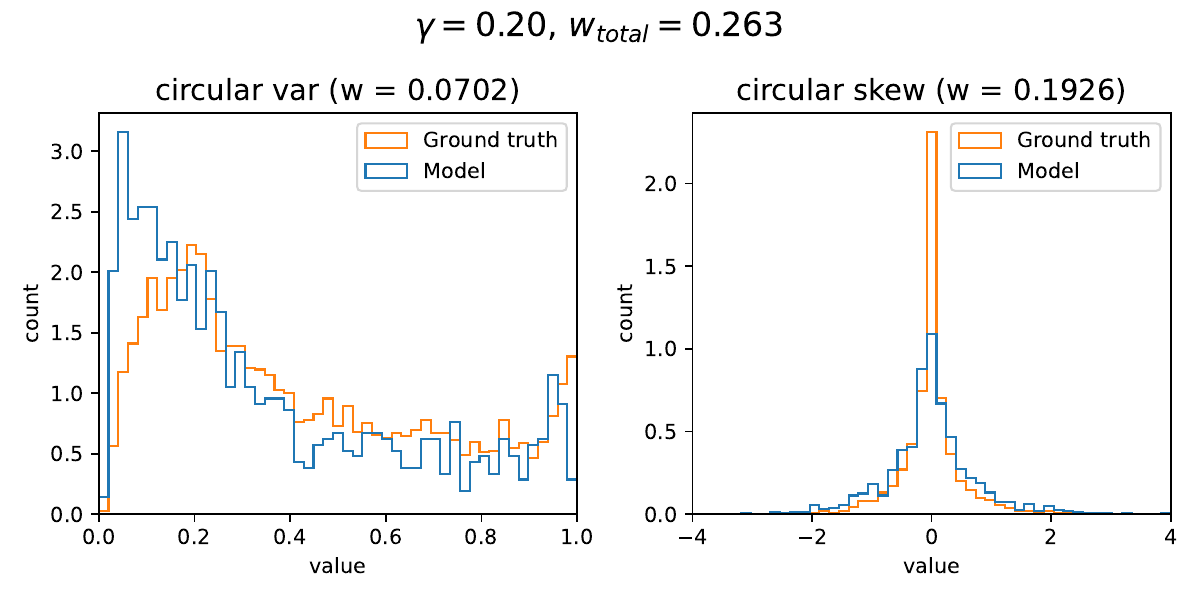} \ \ \includegraphics[width=0.4\textwidth,trim=0 280 140 0,clip]{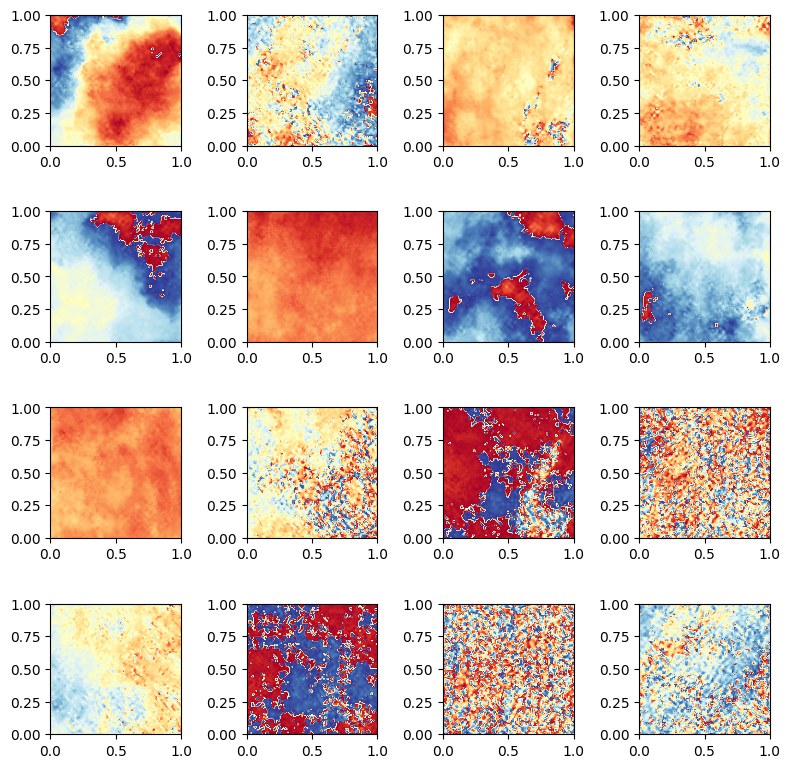}
        \caption{RBF noise, $\lambda = 0.2$, $\wtotal = 0.263$}
        \label{fig:g_sr_hist_02}
    \end{subfigure}
    \caption{Super-resolution experiments for GANO. Each subplot corresponds to a different level of RBF smoothness, which is denoted by $\gamma$ (larger $\gamma$ corresponds to smoother noise). GANO is trained on $60 \times 60$ resolution functions, and $120 \times 120$ resolution functions are produced as per \secref{subsec:volcano}. Variance and skew statistics are computed from $M=1024$ samples and the Wasserstein distance is computed between those samples and the corresponding statistics from the $120 \times 120$ resolution version of the training set (i.e. the original resolution).}
    \label{fig:gano_superres}
\end{figure}
\vfill

\begin{figure}[htb!]
    \centering
    \begin{minipage}[b]{0.02\textwidth}
        \centering
        \includegraphics[width=\textwidth]{figs/background.png}
    \end{minipage}
    \begin{minipage}[b]{0.05\textwidth}
        \centering
        \includegraphics[width=\textwidth]{figs/mnistsdf/gano.png}
    \end{minipage}
    \begin{minipage}[b]{0.29\textwidth}
        \centering
        \includegraphics[width=1.05\textwidth]{figs/mnistsdf/gano_64x64_masked_n64.pdf}

        \vspace{1mm}

    \end{minipage}
    \hfill
    \begin{minipage}[b]{0.29\textwidth}
        \centering
        \includegraphics[width=0.985\textwidth]{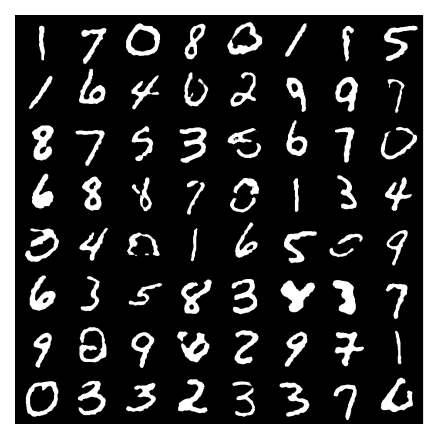}

        \vspace{2.5mm}

    \end{minipage}
    \hfill
    \begin{minipage}[b]{0.29\textwidth}
        \centering
        \includegraphics[width=0.945\textwidth]{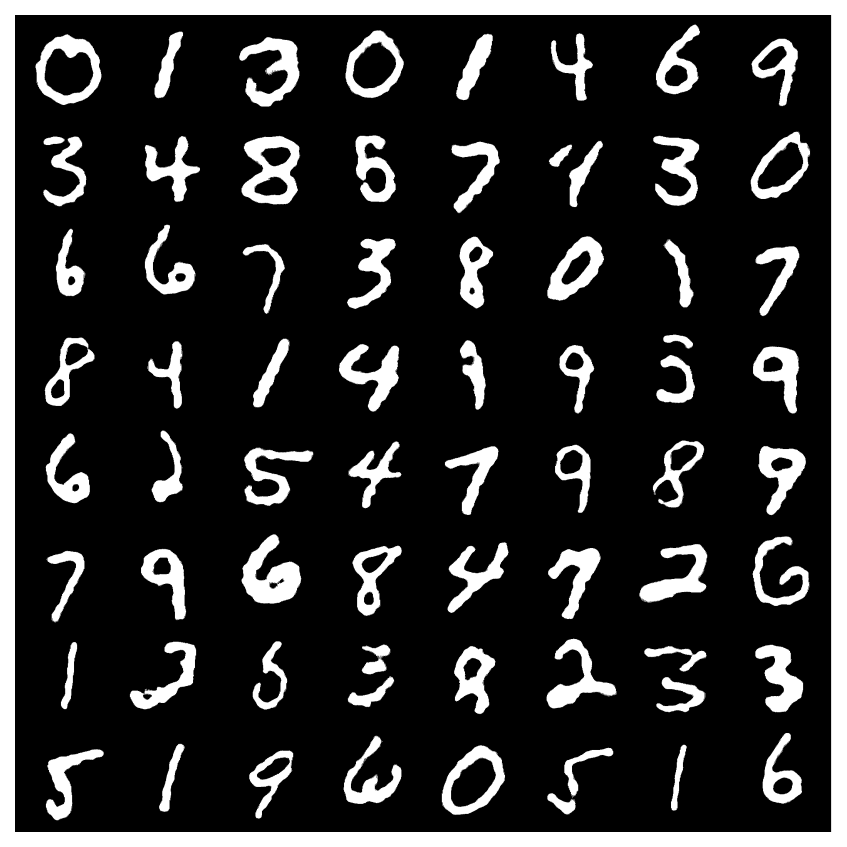}

        \vspace{3.5mm}

    \end{minipage}
    \hfill
    \begin{minipage}[b]{0.02\textwidth}
        \centering
        \includegraphics[width=\textwidth]{figs/background.png}
        \vspace{-5mm}
        \caption*{\textcolor{white}{()}}
    \end{minipage}

    \begin{minipage}[b]{0.02\textwidth}
        \centering
        \includegraphics[width=\textwidth]{figs/background.png}
    \end{minipage}
    \begin{minipage}[b]{0.05\textwidth}
        \centering
        \includegraphics[width=\textwidth]{figs/mnistsdf/multileveldiff.png}
    \end{minipage}
    \begin{minipage}[b]{0.29\textwidth}
        \centering
        \includegraphics[width=1.05\textwidth]{figs/mnistsdf/multileveldiff_64x64_masked_n64.pdf}

        \vspace{1mm}

    \end{minipage}
    \hfill
    \begin{minipage}[b]{0.29\textwidth}
        \centering
        \includegraphics[width=0.985\textwidth]{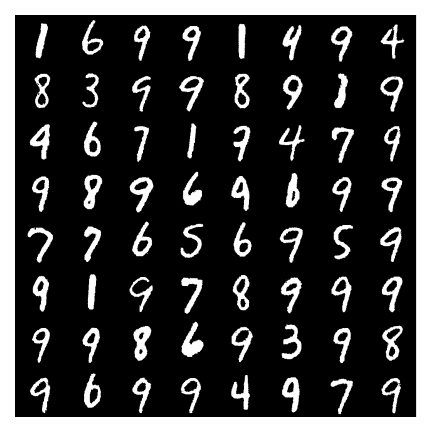}

        \vspace{2.5mm}

    \end{minipage}
    \hfill
    \begin{minipage}[b]{0.29\textwidth}
        \centering
        \includegraphics[width=0.945\textwidth]{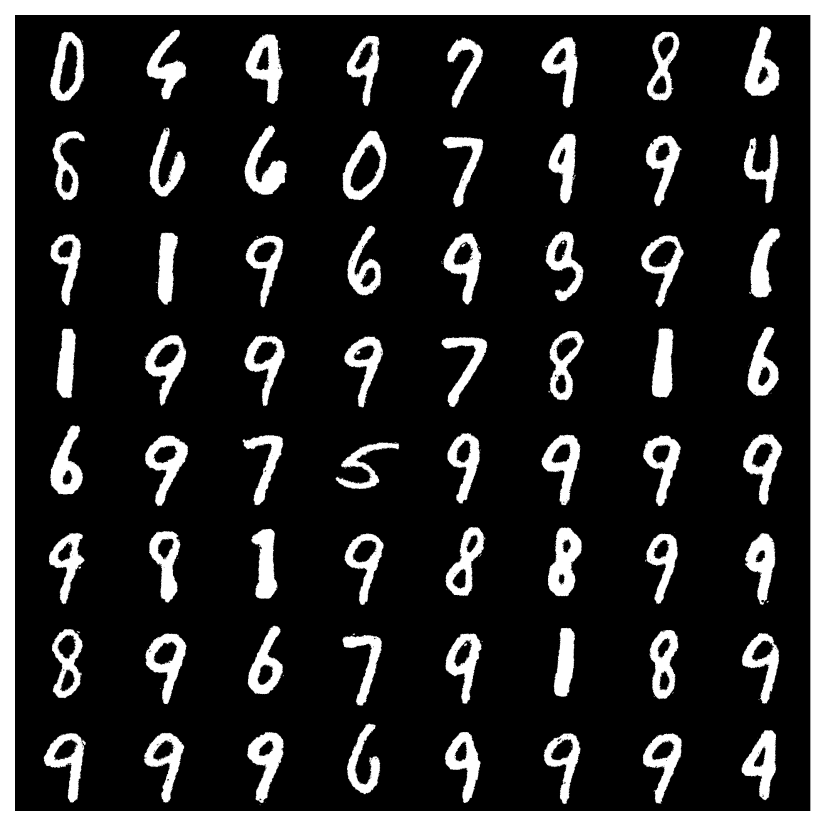}

        \vspace{3.5mm}

    \end{minipage}
    \hfill
    \begin{minipage}[b]{0.02\textwidth}
        \centering
        \includegraphics[width=\textwidth]{figs/background.png}
        \vspace{-5mm}
        \caption*{\textcolor{white}{()}}
    \end{minipage}

    \begin{minipage}[b]{0.02\textwidth}
        \centering
        \includegraphics[width=\textwidth]{figs/background.png}
        \vspace{-5mm}
        \caption*{\textcolor{white}{()}}
    \end{minipage}
    \begin{minipage}[b]{0.05\textwidth}
        \centering
        \includegraphics[width=\textwidth]{figs/mnistsdf/ddo.png}
        \vspace{-5mm}
        \caption*{\textcolor{white}{()}}
    \end{minipage}
    \begin{minipage}[b]{0.29\textwidth}
        \centering
        \includegraphics[width=1.05\textwidth]{figs/mnistsdf/ddo_64x64_masked_n64.pdf}

        \vspace{-4mm}

        \caption*{(a) \,$64\times64$}
    \end{minipage}
    \hfill
    \begin{minipage}[b]{0.29\textwidth}
        \centering
        \includegraphics[width=0.985\textwidth]{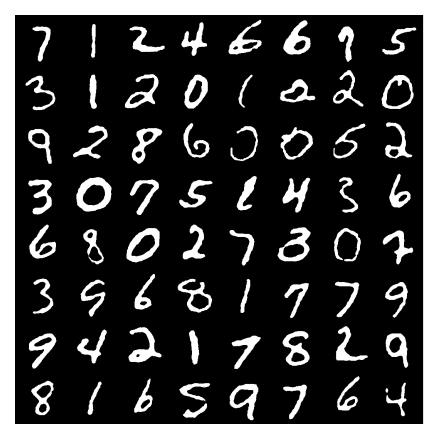}

        \vspace{-2.5mm}

        \caption*{(b) \,$128\times128$}
    \end{minipage}
    \hfill
    \begin{minipage}[b]{0.29\textwidth}
        \centering
        \includegraphics[width=0.945\textwidth]{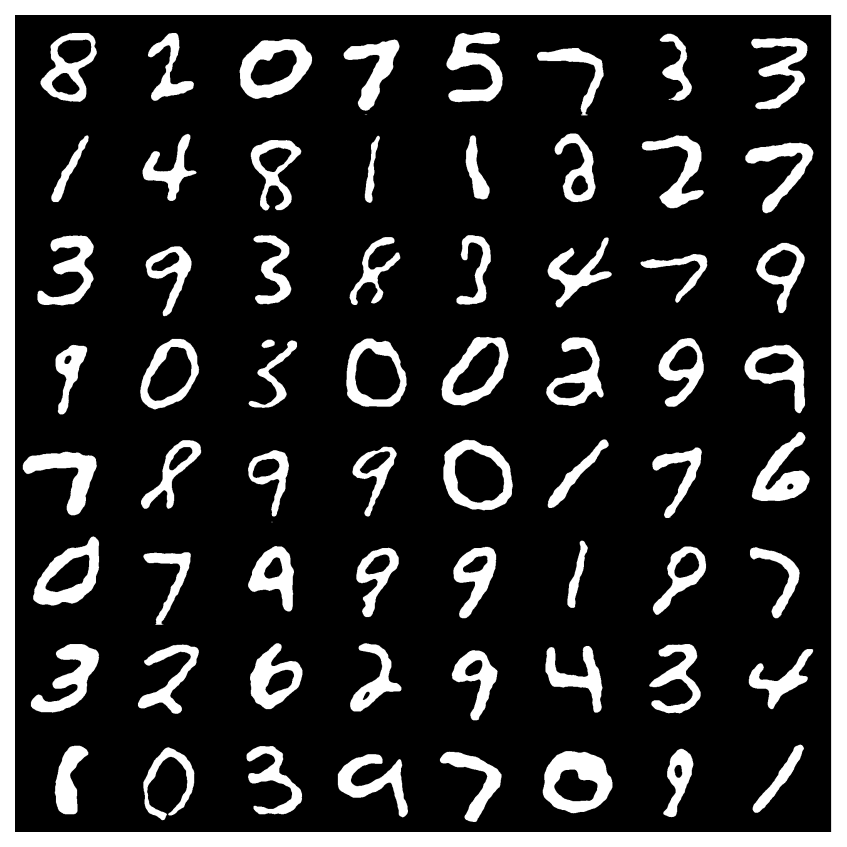}

        \vspace{-1.5mm}

        \caption*{(c) \,$256\times256$}
    \end{minipage}
    \hfill
    \begin{minipage}[b]{0.02\textwidth}
        \centering
        \includegraphics[width=\textwidth]{figs/background.png}
        \vspace{-5mm}
        \caption*{\textcolor{white}{()}}
    \end{minipage}

    \caption{\textbf{Additional generated Samples} (\secref{subsec:mnist-sdf}): Additional samples (masked) of the learned GANO, MultilevelDiff, and DDO models at various resolutions; (a) 64$\times$64, (b) 128$\times$128, and (c) 256$\times$256-resolutions. 64 samples are plotted for all resolutions. All models are trained on 64$\times$64-resolution images, which are upsampled from 32$\times$32-resolution observations of 2D SDFs.}
    \label{fig:mnist-sdf-samples-8x8}
\end{figure}

\subsection{MNIST-SDF Dataset}
\label{app:mnist-sdf}



For the noising processes in the MNIST-SDF experiments, we adopt \eqref{eq:forwardprocess_DDPM} in \appref{sec:ddpm}. More specifically, we follow a more generalized formulation described in \citet{kingma2021variational}. For $t \in (0, 1]$, we first define $u_t$s as in,
\begin{align} \label{eq:forwardprocess_vdm}
    u_t = \alpha_t u + \sigma_t \eta,
\end{align}
where $\eta \sim \mu_0 =  N(0,C)$ and $u \sim \mu$. In this experiment, we use \textit{variance-preserving} setting for the forward process; thus, $\alpha_t$ and $\sigma_t$ will satisfy $\alpha_t^2 + \sigma_t^2 = 1$ for $t \in [0, 1]$.
The approximate backwards conditional in \eqref{eq:backwardprocess_DDPM} will be re-written as \(u_{s}|u_t\) for $\forall \, s, t\,$ s.t. $0 < s < t \le 1$;
\begin{align} \label{eq:backwardprocess_vdm}
    u_s = M_\theta(u_t, t) + \sigma_{\textup{back},s,t} \eta,
\end{align}
where $\eta \sim \mu_0$. $M_\theta$ and $\sigma_{\textup{back},s,t}$ are written as
\begin{align}
    M_{\theta}(u_t, t) &= \frac{\alpha_{t \vertrule s} \sigma_s^2}{\sigma_t^2} u_t - \frac{\alpha_{s \vertrule t}  \sigma_{t \vertrule s}^2}{\sigma_t^2} \left( u_t - \sigma_t F_\theta(u_t, t) \right) \\
    \sigma_{\textup{back},s,t}  &= \frac{\sigma_{t \vertrule s}^2 \sigma_{s}^2}{\sigma_{t}^2},
\end{align}
where $\alpha_{t \vertrule s} = \alpha_t  / \alpha_s$ and $\sigma_{t \vertrule s} = \sigma_t^2 - \alpha_{t \vertrule s}^2 \sigma_s^2$.

The aforementioned formulations provide us several benefits. Most importantly, we can choose arbitrary sequence of discrete time steps in order to generate the learned distribution. This implies that we can train $F_\theta$ at all $t \in (0, 1]$ and choose the sampling sequence later after training. As a result, instead of the training loss \eqref{eq:objective_DDPM}, we obtimize the following objective, 
\begin{align}
    \E_{u \sim \mu, \eta \sim \mu_0, t \sim U(0,1)} \|F_\theta \big( \alpha_t u + \sigma_t \eta , t \big ) - \eta \|^2.
\end{align}
In the following sections, we will continue on the trace-class noises used in the experiments and other details.

\paragraph{GRFs with $p$-th Order Exponential Covariance Functions}
We observe that the models often generate artifacts with the blurring diffusions when the noise process is non-periodic Gaussian measures (\secref{subsec:gaussian_example}), as the Gaussian smoothing operator is periodic. Thus, we use periodic Gaussian measures with $p$-th order exponential covariance function, whose kernel function is written as,
\begin{align}
\label{eq:pth-exp-kernel}
    k(x, x') = \sigma \exp \left( - \left\vert \frac{x-x'}{l} \right\vert^p \right),
\end{align}
where $l$ is a length scale, $p$ is order, and $\sigma$ is a magnitude. Note that the RBF kernel used in the volcano experiments is a special case of \eqref{eq:pth-exp-kernel}

As the data is assumed to be sampled on a lattice in $\mathbb{R}^d$, we employ the convolution construction of the Gaussian measures for efficient sampling \citep{higdon2002space}. For a given exponential kernel $k$, the corresponding convolution kernel $c$ is a function of the absolute difference $x$ and is written as
\[
c(x) = \mathcal{F}^{-1}\left[ \sqrt{\mathcal{F}[k]} \right](x)
\]
where $\mathcal{F}$ and $\mathcal{F}^{-1}$ are Fourier and inverse Fourier transform, respectively. We denote the closed form $\mathcal{F}(k)(\omega)$ for $p$-th order exponential kernel as $\phi_p(\omega)$, and $\phi_p$ is written as \[
\phi_p(\omega) = \begin{cases}
    \frac{2}{1 + \omega^2}, & p = 1, \\
    2 \pi c_p \frac{p |\omega|^{\frac{1}{1-p}}}{2 |p-1|} \int_{0}^{1} U_{p}(x) \exp \left( -|\omega|^{\frac{p}{p-1}} U_p(x)\right) dx, & p \in (0, 2] \setminus \{1\}, \\
\end{cases}
\]
where $c_p$ and $U_p$ are defined as\[
c_p = \frac{ p }{ 2^{\frac{p+1}{p}} \Gamma \left( \frac{1}{p}\right) }\quad\,\, \textrm{and}\quad\,\,
U_p(x) = \left( \frac{\sin\left( \frac{\pi x p }{2}\right)}{\cos \left( \frac{\pi x}{2} \right)} \right)^{\frac{p}{1-p}} \frac{ \cos \left( \frac{\pi x (p-1) }{2} \right) }{ \cos \left( \frac{\pi x }{2} \right) },
\]
and $\Gamma$ is the Gamma function.
The detailed information about $\phi_p$, including its derivation, is described in \citet{dytso2018analytical}.


\paragraph{Architecture and Training Details}
This section provides a detailed description of the experimental configuration used in our study. 
Before elaborating on the hyperparameters, we introduce a neural operator in the image experiments. We modify the UNet network architectures \citep{ronneberger2015u}, which has been widely adopted in the context of diffusion-based generative models \citep{ho2020denoising, song2020score, nichol2021improved} and elaborated with regards to neural operators~\citep{rahman2022generative, rahman2023u}. We adopt the Improved Denoising Diffusion Probabilistic Models (IDDPM, \citealt{nichol2021improved}), whose time-conditional modulation uses elementwise affine transformations (shift and scale) instead of shifting-only modulation introduced in \cite{ho2020denoising}. To do that, we introduce four major modifications. First, we replace all regular convolutional layers with spectral convolutions. 
Second, we apply group normalizations on Fourier spaces instead of Euclidean spaces. Here, the normalization coefficients will be computed using the first $k$ modes in Fourier spaces, but the coefficient will be applied to all modes; thus, the normalization works in a resolution-invariant manner. 
Third, we deprecate the self-attention layers and dropouts.
Moreover, for the downsample and upsample operations, which are one of the key components of the UNet-like hierarchical network architectures, we adopt filtered downsample and upsample algorithms discussed in \citet{karras2021alias}.
Therefore, we conclude that the resulting deep learning architecture is a valid neural operator as it is a combination of convolutional operators, normalizations defined in Fourier spaces, and point-wise transformations.

Specifically for the MNIST-SDF experiment, we set the number of base channels to 64 and use the three-stage model. The channel multipliers for each stage are set to 1, 2, and 2, respectively. For each stage, four residual blocks are used. At the first stage, all spectral convolution layers comprise 32 modes, and we halve the modes as the stage increases.

We chose the cosine noise scheduling following IDDPM with a variance-preserving form for the noise schedule.
We follow the sine schedule introduced in \citet{hoogeboom2022blurring} for the blurring schedule. We perform early stopping based on the FID of 5000 generated samples relative to a subset of the training data. We use the exponential moving average (EMA) technique for the evaluation and test; we set the EMA rate to be 0.999. We describe the hyperparameter details in the Table \ref{table:mnist-sdf-exp}.

\begin{updaterequired}[black]
\paragraph{Baseline Models}

For GANO \citep{rahman2022generative}, we follow the methodology outlined in its paper and codebase,\hspace{-0.1em}\footnote{\href{https://github.com/neuraloperator/GANO}{https://github.com/neuraloperator/GANO}} with specific configurations adapted to the study. In particular, the input random field provided to the generator is the Gaussian Random Field (GRF) described in \eqref{eq:example_C1}, with parameter settings of $\sigma_1 = 1$, $\alpha_1 = 1.5$, and $\tau_1 = 1.0$. For the generator utilizing the U-NO architecture \citep{rahman2023u}, the number of modes is set to 32 (with half-modes set to 16), and the number of channels is set to 64.

Training of GANO is conducted using a batch size of 32, with a gradient penalty parameter $\lambda = 10.0$, to stabilize the adversarial learning process. The learning rate is set to 0.0001, and optimization is performed using the default ADAM optimizer, consistent with the configuration in DDO. For evaluation, the Exponential Moving Average (EMA) technique from DDO's experimental framework is employed.

The experiments involving MultilevelDiff are conducted using the official codebase available at its repository.\hspace{-0.1em}\footnote{\href{https://github.com/PaulLyonel/multilevelDiff/commit/9bddd34abcf26591532a125f79ff420807af9c72}{https://github.com/PaulLyonel/multilevelDiff}} For the score operator, the Fourier Neural Operator (FNO) is employed with a configuration of 32 modes (half-modes set to 16) and 256 channels. The noise perturbation is applied with the prior, which combines a spectral convolution-based noise as well as a random field with a fixed kernel. Further details about the hyperparameters used can be found in Table 1 of \citet{hagemann2023multilevel}.

The training process of MultilevelDiff follows the same approach as in DDO to ensure consistency with established methodologies. For evaluation, the Exponential Moving Average (EMA) technique is not used. Instead, evaluations are performed directly without applying EMA, yielding results that remain robust despite this omission.

\begin{table}
\small
\caption{Training details of DDO for the MNIST-SDF experiments}
\label{table:mnist-sdf-exp}
\centering
\begin{tabular}{lll}
\toprule
\textbf{Architecture} & Base channels             & 64  \\
                           & \# of ResBlocks per stage & 4  \\
                           & Channel multiplier        & 1,2,2  \\
                           & \# of modes per stage     & 32, 16, 8  \\
                           & Activation function       & GeLU$^\ddagger$  \\ 
                           &\# of params      & 258M  \\
\midrule
\textbf{Diffusion}   & Noise schedule            & Cosine  \\
                           & $\log( \alpha_0^2 / \sigma_0^2 )$ & 10  \\
                           & $\log( \alpha_1^2 / \sigma_1^2 )$ & -10  \\
\midrule
\textbf{Blurring}    & Blurring schedule         & Sine  \\
                           & $d_0$$^{\mathsection}$   & 0.05  \\
                           & $d_1$                    & 0.25  \\
\midrule
\textbf{Learning}    & Optimizer                 & Adam, $\beta_1$=0.9, $\beta_2$=0.999  \\
                           & Learning rate             & 0.0001  \\
                           & Batch size                & 32  \\
                           & \# of iterations          & 2M  \\
\midrule
\textbf{Sampling}    & \# of steps               & 250  \\
\midrule
\textbf{GRFs}    & Length scale & 0.05  \\
                 & Magnitude    & 1 \\
                 & Order        & 2 \\
\bottomrule
\multicolumn{3}{l}{\scriptsize $^\mathsection$$d_0$ and $d_1$ are frequency scalings at $t=0$ and $d_1$, respectively.} \\
\multicolumn{3}{l}{\scriptsize $^\ddagger$Gaussian Error Linear Units function\citep{hendrycks2016gelu}.}
\end{tabular}
\end{table}

\end{updaterequired}

\begin{updaterequired}[black]
\subsection{Darcy Flow Bayesian Inverse Problem}
\label{app:darcyflow}

In the Darcy Flow Bayesian Inverse Problem, we follow the methodology outlined in \appref{app:mnist-sdf}, with adjustments made only to specific hyperparameters to suit the dataset better. Note that the data in this experiment is presented at a $64\times64$ resolution, while the MNIST-SDF experiment uses a $32\times32$ resolution for training. To condition on the $8\times8$-size observation using the neural operators, we upsample the observation to $64\times64$ and concatenate it with the original input to the operators.

In GANO, the input Gaussian Random Field (GRF) is configured with $\sigma_1 = 1$, $\alpha_1 = 1.0$, and $\tau_1 = 1.0$. Additionally, only modes up to 64 (with respect to DCT) are utilized, and thus the model will ignore any higher modes when it is asked to generate at higher resolutions. For the U-NO architecture in the generator, the number of modes is set to 64 (with half-modes set to 32), and the number of channels is configured to 64. Although experiments with increased channel numbers were tested, this configuration yields the best results.

For MultilevelDiff, the Fourier Neural Operator (FNO) is used with a configuration of 64 modes (half-modes set to 32) and 128 channels. For modifications to DDO, Table \ref{table:darcyflow-exp} provides a comprehensive overview of the changes applied.

As mentioned in Section \ref{subsec:darcyflow}, we also present the generated samples at $128\times128$ and $256\times256$ resolutions in Figures \ref{fig:darcyflow-samples-128x128} and \ref{fig:darcyflow-samples-256x256}, respectively.

\begin{table}
\small
\caption{Training details of DDO for Bayesian Inverse Problems}
\label{table:darcyflow-exp}
\centering
\begin{tabular}{lll}
\toprule
\textbf{Architecture} & Base channels             & 32  \\
                        & \# of ResBlocks per stage & 4  \\
                        & Channel multiplier        & 1,2,2  \\
                        & \# of modes per stage     & 48, 32, 16  \\
                        & Activation function       & Swish$^\ddagger$ \\
                        &\# of params      & 139M  \\
\midrule
\textbf{Diffusion}   & Noise schedule            & Cosine  \\
                           & $\log( \alpha_0^2 / \sigma_0^2 )$ & 10  \\
                           & $\log( \alpha_1^2 / \sigma_1^2 )$ & -10  \\
\midrule
\textbf{Blurring}    & Blurring schedule         & Sine  \\
                           & $d_0$$^{\mathsection}$   & 0.01  \\
                           & $d_1$                    & 0.05  \\
\midrule
\textbf{Learning}    & Optimizer                 & Adam, $\beta_1$=0.9, $\beta_2$=0.999  \\
                           & Learning rate             & 0.00001  \\
                           & Batch size                & 16  \\
                           & \# of iterations          & 300,000 \\
\midrule
\textbf{Sampling}    & \# of steps               & 500  \\
\midrule
\textbf{GRFs}    & Length scale & 0.02  \\
                 & Magnitude    & 1 \\
                 & Order        & 2 \\
\bottomrule
\multicolumn{3}{l}{\scriptsize $^\mathsection$$d_0$ and $d_1$ are frequency scalings at $t=0$ and $d_1$, respectively.} \\
\multicolumn{3}{l}{\scriptsize $^\ddagger$Swish function \citep{ramachandran2017searching}.}
\end{tabular}
\end{table}

\begin{figure}[htb!]
\captionsetup{skip=8pt}
\centering
\begin{subfigure}[b]{0.92\textwidth}
    \centering
    \begin{minipage}[b]{0.029\textwidth}
        \centering
        \includegraphics[width=\textwidth, trim={0 2.5cm 0 3cm}, clip=true]{figs/darcyflow/mcmc.png}
    \end{minipage}
    \begin{minipage}[b]{0.95\textwidth}
        \centering
        \includegraphics[width=\textwidth]{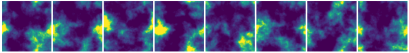}
    \end{minipage}

    \begin{minipage}[b]{0.029\textwidth}
        \centering
        \includegraphics[width=\textwidth, trim={0 3.2cm 0 2.3cm}, clip=true]{figs/darcyflow/gano.png}
    \end{minipage}
    \begin{minipage}[b]{0.95\textwidth}
        \centering
        \includegraphics[width=\textwidth]{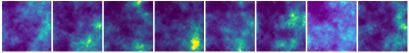}
    \end{minipage}

    \begin{minipage}[b]{0.029\textwidth}
        \centering
        \includegraphics[width=0.9\textwidth, trim={0 2.5cm 0 1.6cm}, clip=true, bmargin=0.0mm]{figs/darcyflow/multileveldiff.png}
    \end{minipage}
    \begin{minipage}[b]{0.95\textwidth}
        \centering
        \includegraphics[width=\textwidth]{figs/darcyflow/darcy_multileveldiff_samples_64x64.pdf}
    \end{minipage}

    \begin{minipage}[b]{0.029\textwidth}
        \centering
        \includegraphics[width=\textwidth, trim={0 3.0cm 0 3.0cm}, clip=true]{figs/darcyflow/ddo.png}
    \end{minipage}
    \begin{minipage}[b]{0.95\textwidth}
        \centering
        \includegraphics[width=\textwidth]{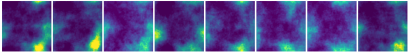}
    \end{minipage}
\end{subfigure}
\begin{subfigure}[b]{0.063\textwidth}
    \includegraphics[width=\textwidth]{figs/darcyflow/colorbar_mean_2.png}
\end{subfigure}
\caption{\textbf{Posterior Samples at 128$\times$128 resolution} (\secref{subsec:darcyflow}): The samples from MCMC as well as the learned GANO, MultilevelDiff, and DDO models at 128$\times$128 resolution. The models are trained at 64$\times$64 resolution. 
}
\label{fig:darcyflow-samples-128x128}
\end{figure}

\begin{figure}[htb!]
\captionsetup{skip=8pt}
\centering
\begin{subfigure}[b]{0.92\textwidth}
    \centering
    \begin{minipage}[b]{0.029\textwidth}
        \centering
        \includegraphics[width=\textwidth, trim={0 2.5cm 0 3cm}, clip=true]{figs/darcyflow/mcmc.png}
    \end{minipage}
    \begin{minipage}[b]{0.95\textwidth}
        \centering
        \includegraphics[width=\textwidth]{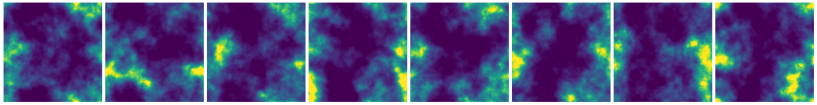}
    \end{minipage}

    \begin{minipage}[b]{0.029\textwidth}
        \centering
        \includegraphics[width=\textwidth, trim={0 3.2cm 0 2.3cm}, clip=true]{figs/darcyflow/gano.png}
    \end{minipage}
    \begin{minipage}[b]{0.95\textwidth}
        \centering
        \includegraphics[width=\textwidth]{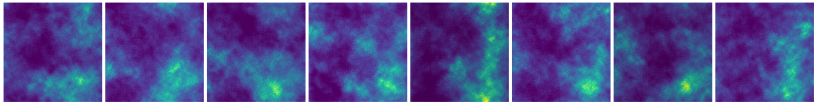}
    \end{minipage}

    \begin{minipage}[b]{0.029\textwidth}
        \centering
        \includegraphics[width=0.9\textwidth, trim={0 2.5cm 0 1.6cm}, clip=true, bmargin=0.0mm]{figs/darcyflow/multileveldiff.png}
    \end{minipage}
    \begin{minipage}[b]{0.95\textwidth}
        \centering
        \includegraphics[width=\textwidth]{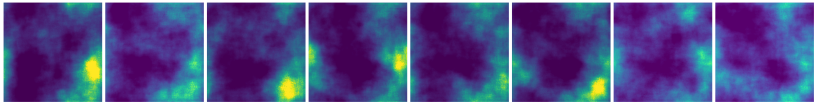}
    \end{minipage}

    \begin{minipage}[b]{0.029\textwidth}
        \centering
        \includegraphics[width=\textwidth, trim={0 3.0cm 0 3.0cm}, clip=true]{figs/darcyflow/ddo.png}
    \end{minipage}
    \begin{minipage}[b]{0.95\textwidth}
        \centering
        \includegraphics[width=\textwidth]{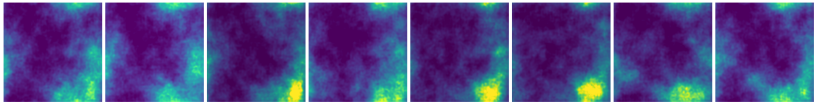}
    \end{minipage}
\end{subfigure}
\begin{subfigure}[b]{0.063\textwidth}
    \includegraphics[width=\textwidth]{figs/darcyflow/colorbar_mean_2.png}
\end{subfigure}
\caption{\textbf{Posterior Samples at 256$\times$256 resolution} (\secref{subsec:darcyflow}): The samples from MCMC as well as the learned GANO, MultilevelDiff, and DDO models at 256$\times$256 resolution. The models are trained at 64$\times$64 resolution. 
}
\label{fig:darcyflow-samples-256x256}
\end{figure}
\end{updaterequired}

\begin{updaterequired}[black]
\section{Neural Operators and Alias-free Models}
\label{app:alias-free}

This paper has conducted various experiments to evaluate the efficacy of DDO. In particular, GANO and MultilevelDiff were chosen as baselines. Except for a few cases, DDO mostly outperformed these methods, demonstrating the effectiveness of the proposed approach.

However, it is concerning that some notable methods, which have received significant attention in the field and propose function-valued diffusion models, were excluded from the baselines. For example, $\infty$-Diff \citep{bond2024infty} successfully modeled high-fidelity image datasets such as CelebAHQ \citep{karras2018progressive} and FFHQ \citep{karras2019style}, yet their method was left out of the baselines (see also \secref{subsec:mnist-sdf}). Briefly speaking, this method proposes a discrete-time latent diffusion model on function spaces, where the latent space is also a function space akin to the data space, and encoder-decoder structure maps between those two spaces; thus, function-valued diffusion models will model the latent distribution mapped from the data. Such a framework, including an encoder-decoder structure tailored for high-fidelity image modeling, makes direct comparison with DDO less fair. 

To address these concerns and provide a meaningful comparison that includes DDO and related methods like $\infty$-Diff, we designed a simplified toy experiment. For this final experiment, we conducted a straightforward comparison to analyze the resulting differences in generated images from variations in neural operator design. Each model was given a set of two-dimensional coordinate pairs and trained to predict the RGB values of pixels at these coordinates by minimizing the mean squared error with respect to the true RGB values. As will be discussed below, the design of the neural operator plays a critical role in function-valued generative models. Therefore, this experimental setup abstracts away the influence of the application-specific frameworks, aiming to highlight the relative strengths and limitations of each neural operator. 

\paragraph{Neural Operators in Function-valued Deep Generative Models}
In deep generative modeling, defining the random variable for the generation and training of the parametric model is crucial. However, the design of the parametric model is also a very important factor. This is particularly true for function-valued deep generative models as compared to the finite-dimensional case because, in function space models, the parametric model---often referred to as a neural operator---must not only enhance expressivity but also satisfy a special property known as discretization invariance. The importance of discretization invariance lies in the fact that, while function-valued objects are theoretically handled, actual observations are ultimately discretizations of these functions.

Consequently, the design choices for the neural operator determine the performance characteristics of the parametric model on function space and the trade-offs involved. For example, spectral convolution, which leverages the Fourier series, theoretically and practically guarantees discretization invariance, but it is often susceptible to aliasing issues, such as ringing artifacts. On the other hand, continuous convolution-based methods may be more robust against aliasing but tend to overfit to specific discretizations, necessitating some form of regularization during training. Ultimately, the difference in the parametric model can have a more pronounced effect than the differences between generative models themselves.

For these reasons, conducting experiments to compare neural operators is both valuable and insightful. By doing so, we aim to provide a clearer understanding of the relative strengths and limitations of our proposed method in comparison to these baselines. 

\paragraph{Experiment details}
In the experiment, each model was learned to predict the RGB values of pixels at specified two-dimensional coordinates by minimizing the mean squared error relative to the true RGB values. 

\begin{figure}[htb!]
\captionsetup{skip=4pt}
    \centering
    \begin{minipage}[b]{0.02\textwidth}
        \centering
        \includegraphics[width=\textwidth]{figs/background.png}
    \end{minipage}
    \begin{minipage}[b]{0.05\textwidth}
        \centering
        \includegraphics[width=\textwidth]{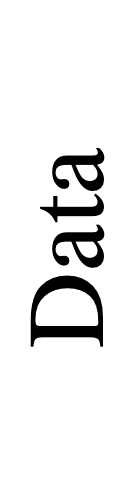}
    \end{minipage}
    \begin{minipage}[b]{0.20\textwidth}
        \centering
        \includegraphics[width=\textwidth]{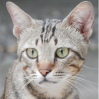}
    \end{minipage}
    \hfill
    \begin{minipage}[b]{0.20\textwidth}
        \centering
        \includegraphics[width=\textwidth]{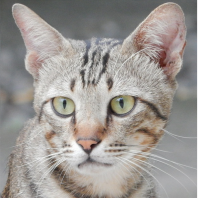}
    \end{minipage}
    \hfill
    \begin{minipage}[b]{0.20\textwidth}
        \centering
        \includegraphics[width=\textwidth]{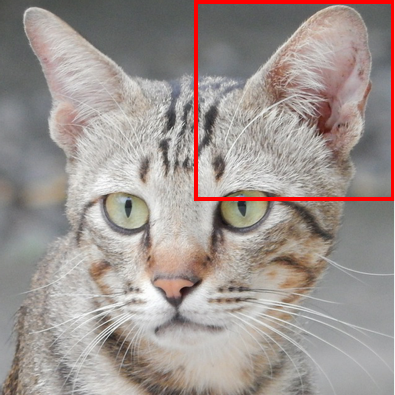}
    \end{minipage}
    \hfill
    \begin{minipage}[b]{0.20\textwidth}
        \centering
        \includegraphics[width=\textwidth]{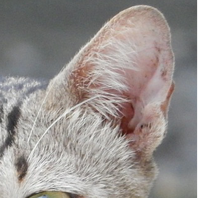}
    \end{minipage}
    \hfill
    \begin{minipage}[b]{0.02\textwidth}
        \centering
        \includegraphics[width=\textwidth]{figs/background.png}
    \end{minipage}

    \begin{minipage}[b]{0.02\textwidth}
        \centering
        \includegraphics[width=\textwidth]{figs/background.png}
    \end{minipage}
    \begin{minipage}[b]{0.05\textwidth}
        \centering
        \includegraphics[width=\textwidth]{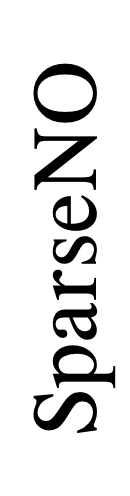}
    \end{minipage}
    \begin{minipage}[b]{0.20\textwidth}
        \centering
        \includegraphics[width=\textwidth]{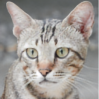}
    \end{minipage}
    \hfill
    \begin{minipage}[b]{0.20\textwidth}
        \centering
        \includegraphics[width=\textwidth]{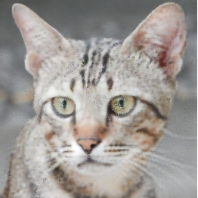}
    \end{minipage}
    \hfill
    \begin{minipage}[b]{0.20\textwidth}
        \centering
        \includegraphics[width=\textwidth]{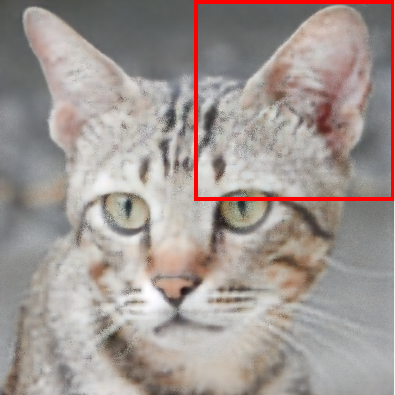}
    \end{minipage}
    \hfill
    \begin{minipage}[b]{0.20\textwidth}
        \centering
        \includegraphics[width=\textwidth]{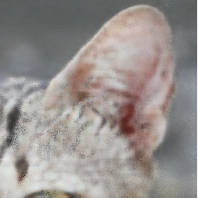}
    \end{minipage}
    \hfill
    \begin{minipage}[b]{0.02\textwidth}
        \centering
        \includegraphics[width=\textwidth]{figs/background.png}
    \end{minipage}

    \begin{minipage}[b]{0.02\textwidth}
        \centering
        \includegraphics[width=\textwidth]{figs/background.png}
    \end{minipage}
    \begin{minipage}[b]{0.05\textwidth}
        \centering
        \includegraphics[width=\textwidth]{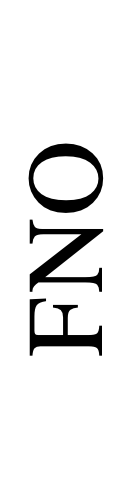}
    \end{minipage}
    \begin{minipage}[b]{0.20\textwidth}
        \centering
        \includegraphics[width=\textwidth]{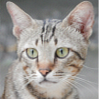}
    \end{minipage}
    \hfill
    \begin{minipage}[b]{0.20\textwidth}
        \centering
        \includegraphics[width=\textwidth]{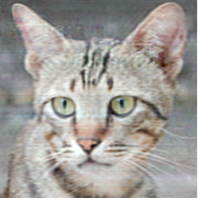}
    \end{minipage}
    \hfill
    \begin{minipage}[b]{0.20\textwidth}
        \centering
        \includegraphics[width=\textwidth]{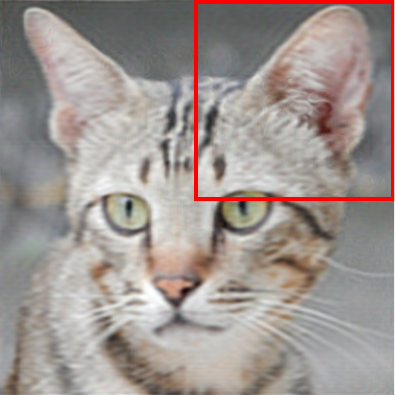}
    \end{minipage}
    \hfill
    \begin{minipage}[b]{0.20\textwidth}
        \centering
        \includegraphics[width=\textwidth]{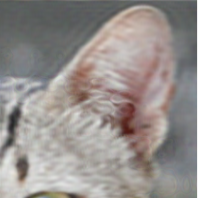}
    \end{minipage}
    \hfill
    \begin{minipage}[b]{0.02\textwidth}
        \centering
        \includegraphics[width=\textwidth]{figs/background.png}
    \end{minipage}

    \begin{minipage}[b]{0.02\textwidth}
        \centering
        \includegraphics[width=\textwidth]{figs/background.png}
    \end{minipage}
    \begin{minipage}[b]{0.05\textwidth}
        \centering
        \includegraphics[width=\textwidth]{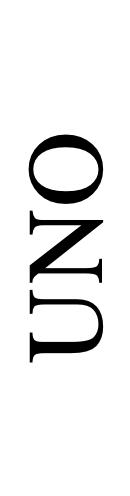}
    \end{minipage}
    \begin{minipage}[b]{0.20\textwidth}
        \centering
        \includegraphics[width=\textwidth]{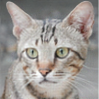}
    \end{minipage}
    \hfill
    \begin{minipage}[b]{0.20\textwidth}
        \centering
        \includegraphics[width=\textwidth]{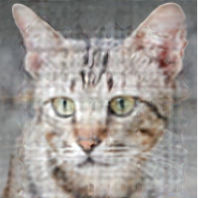}
    \end{minipage}
    \hfill
    \begin{minipage}[b]{0.20\textwidth}
        \centering
        \includegraphics[width=\textwidth]{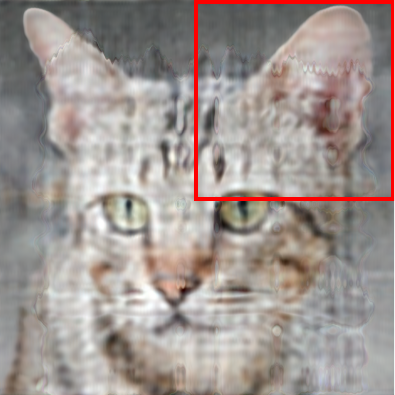}
    \end{minipage}
    \hfill
    \begin{minipage}[b]{0.20\textwidth}
        \centering
        \includegraphics[width=\textwidth]{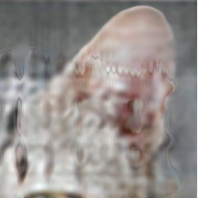}
    \end{minipage}
    \hfill
    \begin{minipage}[b]{0.02\textwidth}
        \centering
        \includegraphics[width=\textwidth]{figs/background.png}
    \end{minipage}

    \begin{minipage}[b]{0.02\textwidth}
        \centering
        \includegraphics[width=\textwidth]{figs/background.png}
    \end{minipage}
    \begin{minipage}[b]{0.05\textwidth}
        \centering
        \includegraphics[width=\textwidth]{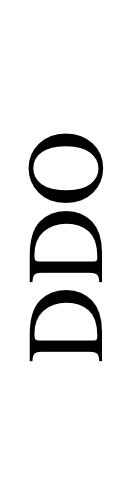}
    \end{minipage}
    \begin{minipage}[b]{0.20\textwidth}
        \centering
        \includegraphics[width=\textwidth]{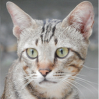}
    \end{minipage}
    \hfill
    \begin{minipage}[b]{0.20\textwidth}
        \centering
        \includegraphics[width=\textwidth]{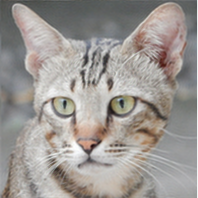}
    \end{minipage}
    \hfill
    \begin{minipage}[b]{0.20\textwidth}
        \centering
        \includegraphics[width=\textwidth]{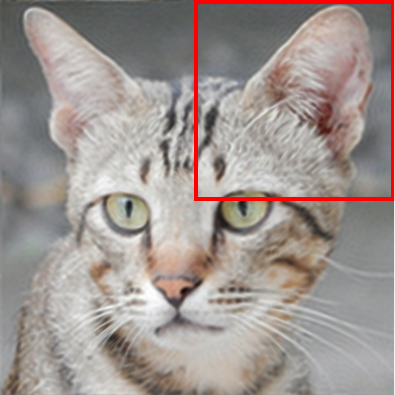}
    \end{minipage}
    \hfill
    \begin{minipage}[b]{0.20\textwidth}
        \centering
        \includegraphics[width=\textwidth]{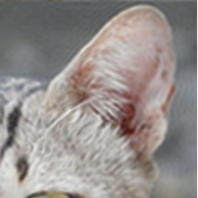}
    \end{minipage}
    \hfill
    \begin{minipage}[b]{0.02\textwidth}
        \centering
        \includegraphics[width=\textwidth]{figs/background.png}
    \end{minipage}

    \begin{minipage}[b]{0.02\textwidth}
    \captionsetup{skip=2pt}
        \centering
        \includegraphics[width=\textwidth]{figs/background.png}
        \caption*{\textcolor{white}{()}}
    \end{minipage}
    \begin{minipage}[b]{0.05\textwidth}
    \captionsetup{skip=2pt}
        \centering
        \includegraphics[width=\textwidth]{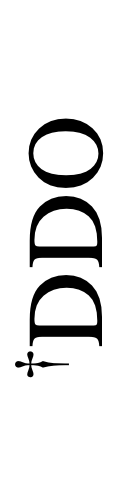}
        \caption*{\textcolor{white}{()}}
    \end{minipage}
    \begin{minipage}[b]{0.20\textwidth}
    \captionsetup{skip=2pt}
        \centering
        \includegraphics[width=\textwidth]{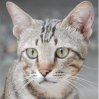}
        \caption*{(a)$\,256\times256$$\,$}
    \end{minipage}
    \hfill
    \begin{minipage}[b]{0.20\textwidth}
    \captionsetup{skip=2pt}
        \centering
        \includegraphics[width=\textwidth]{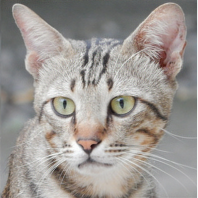}
        \caption*{(b)$\,512\times512$$\,\,$}
    \end{minipage}
    \hfill
    \begin{minipage}[b]{0.20\textwidth}
    \captionsetup{skip=2pt}
        \centering
        \includegraphics[width=\textwidth]{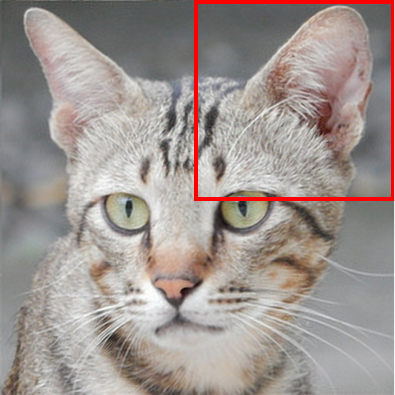}
        \caption*{(c)$\,1024\times1024$$\,$}
    \end{minipage}
    \hfill
    \begin{minipage}[b]{0.20\textwidth}
    \captionsetup{skip=2pt}
        \centering
        \includegraphics[width=\textwidth]{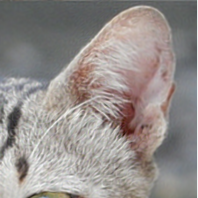}
        \caption*{(d) Zoomed\textcolor{white}{()}}
    \end{minipage}
    \hfill
    \begin{minipage}[b]{0.02\textwidth}
    \captionsetup{skip=2pt}
        \centering
        \includegraphics[width=\textwidth]{figs/background.png}
        \caption*{\textcolor{white}{()}}
    \end{minipage}

    \caption{\textbf{Comparison of neural operators} (\secref{app:alias-free}):
    Each network architecture is trained to predict a image for given coordinate values. All models are trained at 256$\times$256-resolution. Predicted images at various resolutions (a–c) are illustrated, and the red rectangular regions in (c) are enlarged and shown in (d). The number of parameters for all models is kept below 3 million (except for $^\dagger$DDO, which uses 10 times more parameters).
    }
    \label{fig:fitting-samples}
\end{figure}

We used an image from the AFHQ animal faces dataset \citep{choi2020stargan} for the experiment. All models were trained at a resolution of 256$\times$256 and tested at various resolutions. For this, we compared the UNO used in GANO \citep{rahman2022generative}, FNO used in Multileveldiff \citep{hagemann2023multilevel}, Sparse Neural Operator (SparseNO) used in \citet{bond2024infty}, as well as a UNO variant used in our DDO model.

We observed that SparseNO requires significantly more computational memory compared to other models. Specifically, it was necessary to limit the size of a SparseNO model to fewer than 3 million parameters to fit within the memory constraints of a single NVIDIA A100 GPU. To ensure a fair comparison, we limited the size of all models to a similar scale (with the number of parameters kept below 3 million) and conducted all experiments using one NVIDIA A100 GPU. Additionally, we included a DDO model, whose size is 10 times larger for reference. 

\paragraph{Results}
Figure \ref{fig:fitting-samples} shows the predicted images at various resolutions generated by our DDO model and the baseline models. To highlight the characteristics of the neural operators, we enlarge the region within the red rectangle of the predicted images at 1024$\times$1024 resolution in Figure \ref{fig:fitting-samples} (d).

The FNO model performs well at 256$\times$256 and 1024$\times$1024 resolutions. However, it exhibits noticeable visual artifacts at 512$\times$512, and ringing artifacts are observed at 1024$\times$1024 due to its use of spectral convolution. Similarly, the UNO model performs well at 256$\times$256 but displays consistent visual artifacts at higher resolutions, likely caused by its dependence on pointwise operations at every layer. The SparseNO model produces excellent results at the training resolution; however, artifacts appear at higher resolutions. These issues are likely due to the internal kernel resizing method, as SparseNO applies the bicubic interpolation to enlarge its fixed kernel to a higher resolution, which is not invariant to discretization and thus introduces errors.

In contrast, the DDO model produces consistent results across all resolutions. However, like FNO or UNO, it also shows some ringing artifacts at 1024$\times$1024 due to its reliance on spectral convolution. Finally, the larger-scale DDO model demonstrates minimal ringing artifacts, with fine details remaining consistently preserved across resolutions. This better quality suggests that if the neural operator used in DDO can be successfully integrated into $\infty$-Diff instead of SparseNO, it may potentially lead to performance improvements.

In this section, we briefly introduce the characteristics of neural operators employed in function-valued diffusion models. However, discussions surrounding function-valued models extend far beyond these models. Extensive studies on neural operators \citep{li2020neural, li2020fourier, kovachki2021neural, fanaskov2022spectral, bartolucci2024representation} have explored the discretization invariance, leading to advancements in resolution-adaptive network architectures \citep{demeule2024adaptive}. Additionally, alias-free network designs have been extensively discussed in works such as \citet{karras2021alias}. While these topics are highly relevant to the performance of function-valued generative models, a thorough investigation is beyond the scope of this paper and is left for future work.

\end{updaterequired}

\end{document}